\RequirePackage{fix-cm}
\documentclass[11pt,reqno]{article}   

\usepackage{graphicx}
\usepackage[export]{adjustbox}
\usepackage{amsmath,amssymb,amsthm,enumerate,natbib,color,ifthen,graphicx,overpic,hyperref,authblk,textcomp,cleveref}
\usepackage{xargs}
\usepackage[textwidth=4cm, textsize=footnotesize]{todonotes}
\usepackage[utf8]{inputenc}
\usepackage[linesnumbered,vlined,boxed,commentsnumbered]{algorithm2e}
\usepackage{enumitem}
\usepackage{dashundergaps}
\usepackage{aliascnt,bbm}

\usepackage{tabularx}

\def\nset{{\mathbb{N}}}
\def\rset{\mathbb R}

\def\rmd{\mathrm{d}}
\def\argmin{\operatorname{Argmin}}

\def\min{\mathrm{min}}
\def\max{\mathrm{max}}
\def\1{\mathbbm{1}}
\newcommand{\un}{\ensuremath{\mathbbm{1}}}
\def\PE{\mathbb{E}} % esperance
\newcommand{\pscal}[2]{\left\langle#1,#2\right\rangle}
\newcommand{\eqdef}{\ensuremath{\stackrel{\mathrm{def}}{=}}}
\newcommand{\kmax}{K_\mathrm{max}}
\newcommand{\km}{\mathrm{K}}
\newcommand{\fgm}{\mathrm{FGM}}

\def\barY{\overline{Y}}
\def\EM{\mathrm{EM}}

\def\FIEM{\mathrm{FIEM}}

\newcommand\true{\mathrm{true}}
\newcommand\curr{\mathrm{curr}}
\newcommand\init{\mathrm{init}}
\newcommand\batch{\mathcal{B}}
\newcommand\lbatch{\mathrm{b}}
\newcommand\A{\mathsf{A}}

\newcommand{\R}{\mathsf{R}}
\def\Zset{\mathsf{Z}}
\def\Rset{\mathbb{R}}
\newcommand\Sset{{\Rset^q}}
\def\Zsigma{\mathcal{Z}}
\newcommand{\loss}[1]{\ensuremath{\mathcal{L}_{#1}}}
\newcommand{\Q}[1]{\ensuremath{\mathsf{Q}_{#1}}}
\def\L{\mathsf{L}}
\newcommand{\bars}{\bar{s}}
\newcommand{\s}{s}
\newcommand{\hatS}{\widehat{S}}
\newcommand{\Sronde}{\widetilde{S}}

\newcommand{\Smem}{\mathsf{S}}

\newcommand{\param}{\theta}
\newcommand{\Param}{\Theta}
\newcommand{\map}{\mathsf{T}}
\def\F{\mathcal{F}} % filtration

\newcommand\sequence[3] {\ifthenelse{\equal{#3}{}}{\ensuremath{\{
#1_{#2}\}}}{\ensuremath{\{ #1^{#2}, \eqsp #2 \in #3 \}}}}
\newcommand\sequencedown[3] {\ifthenelse{\equal{#3}{}}{\ensuremath{\{
#1_{#2}\}}}{\ensuremath{\{ #1_{#2}, \eqsp #2 \in #3 \}}}}
\newcommand{\lyap}{V}

\def\Id{\mathrm{I}}

\def\eqsp{\;}

\newcommand{\ie}{i.e.}

\newcommand{\pas}{\gamma}
\def\param{\theta}

\def\pa{\mathsf{a}}
\def\pb{\mathsf{b}}
\def\pc{\mathsf{c}}
\def\pd{\mathsf{d}}

\newcommand{\coint}[1]{\left[#1\right)}
\newcommand{\ocint}[1]{\left(#1\right]}
\newcommand{\ooint}[1]{\left(#1\right)}
\newcommand{\ccint}[1]{\left[#1\right]}

\newtheorem{assumption}{A\hspace{-3pt}}

\newtheorem{theorem}{Theorem}
\newaliascnt{proposition}{theorem}
\newtheorem{proposition}[proposition]{Proposition}
\aliascntresetthe{proposition}
\crefname{proposition}{proposition}{propositions}
\Crefname{Proposition}{Proposition}{Propositions}

\newaliascnt{lemma}{theorem}
\newtheorem{lemma}[lemma]{Lemma}
\aliascntresetthe{lemma}
\crefname{lemma}{lemme}{lemmes}
\Crefname{Lemma}{Lemme}{Lemmes}

\newaliascnt{corollary}{theorem}

\aliascntresetthe{corollary}
\crefname{corollary}{corollaire}{corollaires}
\Crefname{Corollary}{Corollaire}{Corollaires}

\newaliascnt{definition}{theorem}

\aliascntresetthe{definition}
\newaliascnt{remark}{theorem}

\aliascntresetthe{remark}

\begin{document}

\title{Fast Incremental Expectation Maximization for
  finite-sum  optimization: nonasymptotic convergence. \thanks{This work is partially supported by the {\em
      Fondation Simone et Cino Del Duca} through the project OpSiMorE;
    by the French {\em Agence Nationale de la Recherche} (ANR), project under reference ANR-PRC-CE23 MASDOL and Chair ANR of research and teaching in artificial intelligence - SCAI Statistics and Computation for AI; and by the  Russian Academic Excellence Project '5-100'. }}

\author[1]{G. Fort}
\author[2]{P. Gach}
\author[3]{E. Moulines}

\affil[1]{Institut de Mathématiques de Toulouse \& CNRS, France; {\small gersende.fort@math.univ-toulouse.fr}}
\affil[2]{Institut de Mathématiques de Toulouse \& Université Toulouse 3, France; {\small pierre.gach@math.univ-toulouse.fr}}
\affil[3]{CMAP \& Ecole Polytechnique, France; {\small eric.moulines@polytechnique.edu}}

\maketitle

\begin{abstract}
  Fast Incremental Expectation Maximization (FIEM) is a version of the
  EM framework for large datasets.  In this paper, we first recast
  FIEM and other incremental EM type algorithms in the {\em Stochastic
    Approximation within EM} framework. Then, we provide nonasymptotic
  bounds for the convergence in expectation as a function of the
  number of examples $n$ and of the maximal number of iterations
  $\kmax$.  We propose two strategies for achieving an
  $\epsilon$-approximate stationary point, respectively with $\kmax =
  O(n^{2/3}/\epsilon)$ and $\kmax = O(\sqrt{n}/\epsilon^{3/2})$, both
  strategies relying on a random termination rule before $\kmax$ and
  on a constant step size in the Stochastic Approximation step. Our
  bounds provide some improvements on the literature. First, they
  allow $\kmax$ to scale as $\sqrt{n}$ which is better than $n^{2/3}$
  which was the best rate obtained so far; it is at the cost of a
  larger dependence upon the tolerance $\epsilon$, thus making this
  control relevant for small to medium accuracy with respect to the
  number of examples $n$. Second, for the $n^{2/3}$-rate, the
  numerical illustrations show that thanks to an optimized choice of
  the step size and of the bounds in terms of quantities
  characterizing the optimization problem at hand, our results design
  a less conservative choice of the step size and provide a better
  control of the convergence in expectation.
\end{abstract}

{\bf keywords:} Computational Statistical Learning \and Large Scale
  Learning \and Incremental Expectation Maximization algorithm \and
  Momentum Stochastic Approximation \and Finite-sum optimization.


\section{Introduction}
The Expectation Maximization (EM) algorithm was introduced by
\cite{Dempster:em:1977} to solve a non-convex optimization problem on
$\Param \subseteq \rset^d$ when the objective function $F$ is defined
through an integral:
  \begin{equation} \label{eq:intro:F}
  F(\param) \eqdef - \frac{1}{n}\log \int_{\Zset_n} G(z;\param)
  \ \mu_n(\rmd z) \eqsp,
  \end{equation}
  for $n \in \nset \setminus \{0\}$, a positive function $G$ and a
  $\sigma$-finite positive measure $\mu_n$ on a measurable space
  $(\Zset_n, \Zsigma_n)$. EM is a Majorize-Minimization (MM) algorithm
  which, based on the current value of the iterate $\param_\curr$,
  defines a majorizing function $\param \mapsto
  \Q{}(\param,\param_\curr)$ given, up to an additive constant, by
  \begin{multline*}
    \Q{}(\param, \param_\curr)  \eqdef - \frac{1}{n}\int_{\Zset_n}
    \log G(z; \param) \ G(z; \param_\curr) \exp(n \, F(\param_\curr))
    \mu_n(\rmd z) \eqsp.
  \end{multline*}
  The next iterate is chosen to be the/a minimum of $\Q{}(\cdot,
  \param_\curr)$.  Each iteration of EM is divided into two steps. In
  the E step (expectation step), a surrogate function is computed. In
  the M step (minimization step), the surrogate function is minimized.
  The computation of the $\Q{}$ function is straightforward when there
  exist functions $\Phi: \Param \to \rset^q$ and $S: \Zset_n \to
  \rset^q$ such that $n^{-1} \log G(z; \param) =
  \pscal{S(z)}{\Phi(\param)}$; this yields $\Q{}(\param, \param_\curr)
  = - \pscal{\bars(\param_\curr)}{\Phi(\param)}$ where
  $\bars(\theta_\curr)$ denotes the expectation of the function $S$
  with respect to (w.r.t.) the probability measure
  \[
\pi_{\param_\curr}(\rmd z) \eqdef G(z; \param_\curr) \exp(n \,
F(\param_\curr)) \mu_n(\rmd z)\eqsp.
\]
This paper is concerned with the case $\Zset_n$ is the $n$-fold
Cartesian product of the set $\Zset$ (denoted by $\Zset^n$), $z =(z_1,
\cdots, z_n) \in \Zset^n$, $S(z) = n^{-1} \sum_{i=1}^n \s_i(z_i)$,
$\mu_n$ is the tensor product of the $\sigma$-finite
positive measure $\mu$ on the measurable space $(\Zset, \Zsigma)$. This implies that
\begin{align*}
F(\param) &= \frac{1}{n}\sum_{i=1}^n f_i(\param) \eqsp, \\ f_i(\param)
& \eqdef - \log \int_\Zset \exp(\pscal{\s_i(z)}{\Phi(\param)})
\ \mu(\rmd z) \eqsp; \\ \bars(\param_\curr) & \propto \frac{1}{n}
\sum_{i=1}^n \int_{\Zset} \s_i(z)
\frac{\exp(\pscal{\s_i(z)}{\Phi(\param_\curr)}) \, \mu(\rmd
  z)}{\int_\Zset \exp(\pscal{\s_i(z')}{\Phi(\param_\curr)}) \mu(\rmd
  z')} \eqsp.
\end{align*} This finite-sum
framework is motivated by large scale learning problems. In such case,
$n$ is the number of observations, assumed to be independent; the
function $f_i$ stands for a possibly non-convex loss associated to the
observation $\# i$ and can also include a penalty (or a
regularization) term.  In the statistical context, $F$ is the negative
normalized log-likelihood of the $n$ observations in a latent variable
model, and $G$ is the complete likelihood; when $\log G(z;\param) \propto
\pscal{S(z)}{\Phi(\param)}$, it belongs to the curved exponential
family (see e.g. \cite{Brown:1986} and \cite{Sundberg:2019}).

When $n$ is large, the computation of $\bars(\param_\curr)$
is computationally costly and should be avoided. We consider incremental algorithms
which use, at each iteration,  a mini-batch
of examples. The computational complexity of these procedures typically displays a trade-off
between the loss of information incurred by the use of a subset of the observations,
and a faster progress toward the solutions since the parameters can be  updated more often.

  A pioneering work in this direction is the {\em incremental EM} by
  \cite{Neal:hinton:1998}: the data set is divided into $B$ blocks and
  a single block is visited between each parameter update. The $\Q{}$
  function of {\em incremental EM} is again a sum over $n$ terms, but
  each E step consists in updating only a block of terms in this sum
  (see~\cite{Ng:mclachlan:2003}).

  The {\em Online EM} algorithm by \cite{cappe:moulines:2009} was
  originally designed to process data streams. It replaces the
  computation of $\bars(\param_\curr)$ by an iteration of a Stochastic
  Approximation (SA) algorithm (see \cite{robbins:monro:1951}). {\em
    Online EM} in the finite sum setting is closely related to {\em
    Stochastic Gradient Descent}.  Improved versions were considered
  by \cite{chen:etal:2018} and by \cite{karimi:etal:2019} which
  introduced respectively {\em Stochastic EM with Variance Reduction
    (sEM-vr)} and {\em Fast Incremental EM (FIEM)} as variance
  reduction techniques within {\em Online EM} as an echo to {\em
    Stochastic Variance Reduced Gradient} (SVRG,
  \cite{Johnson:zhang:2013}) and {\em Stochastic Averaged Gradient}
  (SAGA, \cite{Defazio:bach:2014}) introduced as variance reduction
  techniques within {\em Stochastic Gradient Descent}.

  In this paper, we aim to study such incremental EM methods combined
  with a SA approach. The first goal of this paper is to cast {\em
    Online EM}, {\em incremental EM} and {\em FIEM} into a framework
  called hereafter {\em Stochastic Approximation within EM}
  approaches; see \autoref{sec:SAwithinEM}. We show that the E step of
  FIEM can be seen as the combination of an SA update and of a control
  variate; we propose to optimize the trade-off between update and
  variance reduction, which yields to the {\em opt-FIEM} algorithm
  (see also \autoref{sec:toymodel:exact} for a numerical exploration).

The second and main objective of this paper, is to derive
nonasymptotic upper bounds for the convergence in expectation of FIEM
(see \autoref{sec:FIEM:complexity}).

Following \cite{ghadimi:lan:2013} (see also
\cite{allenzhu:hazan:2016}, \cite{reddi:etal:2016},
\cite{fang:etal:2018}, \cite{zhou:etal:2018} and
\cite{karimi:etal:2019}), we propose to fix a maximal length $\kmax$
and terminate a path $\{\param^k, k \geq 0\}$ of the algorithm at some
random time $K$ uniformly sampled from $\{0, \ldots, \kmax -1\}$ prior
the run and independently of it; our bounds control the expectation
$\PE\left[ \| \nabla F(\param^K) \|^2 \right]$ and as a corollary, we
discuss how to fix $\kmax$ as a function of the sample size $n$ in
order to reach an $\epsilon$-approximate stationary point \ie\ to find
$\hat \param^{K,\epsilon}$ such that $\PE\left[ \| \nabla F(\hat
  \param^{K,\epsilon})\|^2 \right] \leq \epsilon$. Such a property is
sometimes called $\epsilon$-accuracy in expectation (see
e.g. \cite[Definition 1]{reddi:etal:2016}).

\cite{karimi:etal:2019} established that {\em incremental EM}, which
picks at random one example per iteration, reaches $\epsilon$-accuracy
by choosing $\kmax = O(n \epsilon^{-1})$: even if the algorithm is
terminated at a random time $K$, this random time is chosen as a
function of $\kmax$ which has to increase linearly with the size $n$
of the data set.  They also prove that for {\em FIEM},
$\epsilon$-approximate stationarity is reached with $\kmax = O(n^{2/3}
\epsilon^{-1})$ - here again, with one example picked at random per
iteration.  For these reasons, {\em FIEM} is preferable especially
when $n$ is large (see \autoref{sec:mixtureGaussian} for a numerical
illustration). Our major contribution in this paper is to show that
for {\em FIEM}, the rate depends on the choice of some design
parameters. By choosing a constant step size sequence in the SA step,
depending upon $n$ as $O(n^{-2/3})$, then $\epsilon$-accuracy requires
$\kmax = O(n^{2/3} \epsilon^{-1})$; we provide a choice of the step
size (with an explicit dependence on the constants of the problem) and
an explicit expression of the upper bound, which improve the results
reported in \cite{karimi:etal:2019} (see
\autoref{sec:FIEM:errorrate:case1}; see also
\autoref{sec:toymodel:exact} for illustration). We then prove in
\autoref{sec:FIEM:errorrate:case2} that $\epsilon$-accuracy can be
achieved with $\kmax = O(\sqrt{n} \epsilon^{-3/2})$ iterations using
another strategy for the definition of the step size. Finally, we go
beyond the uniform distribution for the random termination time $K$ by
considering a large class of distributions on the set $\{0, \ldots,
\kmax-1\}$ (see \autoref{sec:FIEM:errorrate:case3}).

\paragraph{Notations.}
$\pscal{a}{b}$ denotes the standard Euclidean scalar product on
$\rset^\ell$, for $\ell \geq 1$; and $\|a\|$ the associated norm. For
a matrix $A$, $A^T$ is its transpose. By convention, vectors are
column vectors. \\ For a smooth function $\phi$, $\dot \phi$ denotes
its gradient; for a smooth real-valued function of several variables
$\L$, $\partial_\tau^{k} \L$ stands for the partial derivative of
order $k$ with respect to the variable $\tau$. \\ For a non negative
integer $n$, $[n] \eqdef \{0, \cdots, n\}$ and $[n]^\star \eqdef \{1,
\cdots, n\}$. $a \wedge b$ is the minimum of two real numbers
$a,b$. The big $O$ notation is used to leave out constants. \\ For a
random variable $U$, $\sigma(U)$ denotes the sigma algebra generated
by $U$.

%%%%%%%%%%%%%%%%%%%%%%%%%%%%%%%%%%%%%%%%%%%%%%%%%%%%%%%
\section{Incremental EM algorithms for finite-sum optimization}
\label{sec:algo}
%%%%%%%%%%%%%%%%%%%%%%%%%%%%%%%%%%%%%%%%%%%%%%%%%%%%%%%
\subsection{EM in the expectation space}
\label{sec:motivation}
This paper deals with EM-based algorithms  to solve
\begin{equation}\label{eq:problem}
\argmin_{\param\in \Param}  F(\param), \qquad  F(\param) \eqdef \frac{1}{n} \sum_{i=1}^n \loss{i}(\param) +  \R(\param) \eqsp,
\end{equation}
where
\begin{equation}\label{eq:def:loss}
\loss{i}(\param) \eqdef - \log \int_\Zset \tilde{h}_i(z) \ \exp\left( \pscal{\s_i(z)}{\phi(\param)} \right) \ \mu(\rmd  z) \eqsp,
\end{equation}
under the following assumption:
\begin{assumption} \label{hyp:model} $\Param \subseteq \rset^d$ is a measurable convex subset.  $(\Zset, \Zsigma)$ is a measurable
    space and $\mu$ is a $\sigma$-finite positive measure on
    $(\Zset, \Zsigma)$. The functions $\R: \Param \to \rset$, $\phi : \Param
    \to \rset^q$ and $\tilde{h}_i: \Zset \to \rset_+$, $\s_i: \Zset
    \to \rset^q$ for $i \in [n]^\star$ are measurable
    functions. Finally, for any $\param \in \Param$ and $i \in
    [n]^\star$, $-\infty<\loss{i}(\param) < \infty$.
\end{assumption}
 Under A\autoref{hyp:model}, for any $\param \in \Param$ and $i \in [n]^\star$, the quantity $p_i(z; \param) \, \mu(\rmd z)$ where
\[
p_i(z; \param) \eqdef \tilde{h}_i(z) \ \exp\left(
\pscal{\s_i(z)}{\phi(\param)} + \loss{i}(\param)\right) \eqsp,
\]
defines a probability measure on $(\Zset, \Zsigma)$. We assume that
\begin{assumption} \label{hyp:bars}
  For all $\param \in \Param$ and $i \in [n]^\star$, the expectation
  \[
\bars_i(\param) \eqdef \int_\Zset \s_i(z) \ p_i(z;\param) \mu(\rmd z)
\]
exists and is computationally tractable.
\end{assumption}
For any $\param \in \Param$, define
\begin{equation}\label{eq:def:bars}
  \bars(\param) \eqdef \frac{1}{n} \sum_{i=1}^n \bars_i(\param) \eqsp.
\end{equation}
The framework defined by \eqref{eq:problem} and \eqref{eq:def:loss}
covers many computational learning problems such as empirical risk
minimization with non-convex losses: $\R$ may include a regularization
condition on the parameter $\param$, $\loss{i}$ is the loss function
associated to example $\# i$ and $n^{-1} \sum_{i=1}^n \loss{i}$ is the
empirical loss. This framework includes negative log-likelihood
inference in latent variable model (see
e.g.~\cite{little:rubin:2002}), when the complete data likelihood is
from a curved exponential family; in this framework, $z \mapsto
p_i(z;\param) \mu(\rmd z)$ is the a posteriori distribution of the
latent variable $\# i$.

Given $\param' \in \Param$, define the function $\overline{F}(\cdot,
\param'): \Param \to \rset$ by
\begin{align*}
\overline{F}(\param, \param') & \eqdef -
\pscal{\bars(\param')}{\phi(\param)} + \R(\param) + \frac{1}{n}
\sum_{i=1}^n  \mathcal{C}_i(\param')  \eqsp,\\
\mathcal{C}_i(\param') & \eqdef \loss{i}(\param')
+\pscal{\bars_i(\param')}{\phi(\param')} \eqsp.
\end{align*}
It is well known (see \cite{maclachlan:2008,lange2016mm}; see also
\autoref{supp:sec:details2} in the supplementary material) that
$\{\overline{F}(\cdot, \param'), \param' \in \Param \}$ is a family of
majorizing function of the objective function $F$ from which a
Majorize-Minimization approach for solving \eqref{eq:problem} can be
derived. Define
\begin{equation}
\label{eq:definition-bar-L}
\L(s,\cdot): \param \mapsto - \pscal{s}{\phi(\param)} + \R(\param)
\end{equation}
and consider the following assumption:
\begin{assumption} \label{hyp:Tmap}  For any $s \in \Sset$,  $\param \mapsto \L(s,\param)$ has a unique global minimum on $\Param$ denoted $\map(s)$.
\end{assumption}
In most successful applications of the EM algorithm, the function
$\param \mapsto \L(s,\param)$ is strongly convex. Strong convexity is
however not required here.  Starting from the current point
$\param^k$, the EM iterative scheme $\param^{k+1} = \map \circ
\bars(\param^k)$ first computes a point in $\bars(\Param)$ through the
expectation $\bars$, and then apply the map $\map$ to obtain the new
iterate $\param^{k+1}$. It can therefore be described in the
$\bars(\Param)$-space, a space sometimes called the {\em expectation
  space}: define the sequence $\sequence{\bars}{k}{\nset}$ by $\bars^0
\in \Sset$ and for any $k \geq 0$
\begin{align} \label{eq:exact:update:bars}
\bars^{k+1} &\eqdef \bars \circ \map(\bars^k) \eqsp.
\end{align}
Sufficient conditions for the characterization of the limit points of
any instance $\{\bars^k, k \geq 0\}$ as the critical points of $F
\circ \map$, for the convergence of the functional along the sequence
$\{F \circ \map(\bars^k), k \geq 0\}$, or for the convergence of the
iterates $\{\bars^k, k \geq 0\}$ exist in the literature (see
e.g. \cite{Wu:1983,Lange:1995,Delyon:etal:1999} in the EM context and
\cite{zangwill:1967,csiszar:tusnady:1984,gunawardana:2005,parisi:etal:2019}
for general iterative MM algorithms). Proposition~\ref{lem:nablaV}
characterizes the fixed points of $\map \circ \bars$ and of $\bars
\circ \map$ under a set of conditions which will be adopted for the
convergence analysis in Section~\ref{sec:FIEM:complexity}.
\begin{assumption} \label{hyp:regV}
  \begin{enumerate}[label=(\roman*)]
     \item \label{hyp:model:C1} The functions $\phi$ and $\R$ are
       continuously differentiable on $\Param^v$ where $\Param^v
       \eqdef \Param$ if $\Param$ is open, or $\Param^v$ is a
       neighborhood of $\Param$ otherwise. $\map$ is continuously
       differentiable on $\Sset$.
\item \label{hyp:model:F:C1} The function $F$ is continuously
  differentiable on $\Param^v$ and for any $\param \in \Param$, we
  have
         \[
\dot F(\param) = - \left( \dot{\phi}(\param)\right)^T \, \bars(\param)
+ \dot \R(\param) \eqsp.
\]
\item \label{hyp:regV:C1} For any $s \in \Sset$, $B(s) \eqdef
  \dot{\left(\phi \circ \map \right)}(s)$ is a symmetric $q \times q$
  matrix with positive minimal eigenvalue.
  \end{enumerate}
\end{assumption}
Under A\autoref{hyp:model} to
A\autoref{hyp:regV}-\ref{hyp:model:C1} and the assumption that
$\Param$ and $\phi(\Param)$ are open subsets of resp. $\Rset^d$ and
$\rset^q$, then \autoref{lem:expfam:reg}  shows that
A\autoref{hyp:regV}-\ref{hyp:model:F:C1} holds and the functions
$\loss{i}$ are continuously differentiable on $\Param$ for all $i \in
[n]^\star$.

Under A\autoref{hyp:model}, A\autoref{hyp:Tmap} and the assumptions
that \textit{(i)} $\map$ is continuously differentiable on $\Sset$ and
\textit{(ii)} for any $s \in \Sset$, $\tau \mapsto \L(s,\tau)$ (see
\eqref{eq:definition-bar-L}) is twice continuously differentiable on
$\Theta^v$ (defined in A\autoref{hyp:regV}-\ref{hyp:model:C1}), then
for any $s \in \Sset$, $\partial^2_\tau \L(s,\map(s))$ is
positive-definite and
\[
B(s) = \left( \dot \map(s)\right)^T \ \partial_\tau^2 \L(s,\map(s))
\ \left( \dot \map(s)\right) \eqsp;
\]
see \cite[Lemma 2]{Delyon:etal:1999}.  Therefore, $B(s)$ is a
symmetric matrix and if $\mathrm{rank}(\dot{\map}(s)) = q = q \wedge
d$, its minimal eigenvalue is positive.
\begin{proposition}
  \label{lem:nablaV}
Assume A\autoref{hyp:model}, A\autoref{hyp:bars} and
A\autoref{hyp:Tmap}.  Define the measurable functions $\lyap: \Sset
\to \rset$ and $h:\Sset \to \rset^q$ by
\[
\lyap(s) \eqdef F \circ \map(s) \eqsp, \qquad h(s) \eqdef \bars \circ
\map(s) -s \eqsp.
\]
\begin{enumerate}
  \item \label{lem:nablaV:item1} If $s^\star$ is a fixed point of
    $\bars \circ \map$, then $\map(s^\star)$ is a fixed point of $\map
    \circ \bars$. Conversely, if $\param^\star$ is a fixed point of
    $\map \circ \bars$ then $\bars (\param^\star)$ is a fixed point of
    $\bars \circ \map$.
  \item \label{lem:nablaV:item2} Assume also A\autoref{hyp:regV}. For
    all $s \in \Sset$, we have $\dot \lyap(s) = - B(s) \ h(s)$, and
    the zeros of $h$ are the critical points of $\lyap$.
\end{enumerate}
\end{proposition}
The proof is in \autoref{sec:proof:nablaV}.
As a conclusion, the EM algorithm summarized in \autoref{algo:EM}, is
designed to converge to the zeros of
\begin{equation}\label{eq:meanfield}
s \mapsto h(s) \eqdef \bars \circ \map (s) -s \eqsp,
\end{equation}
which, for some models, are the critical points of $F \circ
\map$.

\begin{algorithm}[htbp]
  \KwData{$\kmax \in \nset$, $\bars^0 \in \Sset$}
  \KwResult{The EM sequence: $\bars^k, k \in [\kmax]$}
    \For{$k=0, \ldots, \kmax-1$}{$\bars^{k+1}= \bars \circ \map(\bars^k)$  \label{line:EM:updateS}}
    \caption{EM in the expectation space} \label{algo:EM}
\end{algorithm}

\subsection{Stochastic Approximation within EM}
\label{sec:SAwithinEM}
In the finite-sum framework, the number of expectation evaluations
$\bars_i$ per iteration of EM is the number $n$ of examples (see
Line~\ref{line:EM:updateS} of \autoref{algo:EM} and
\eqref{eq:def:bars}). It is therefore very costly in the large scale
learning framework. We review in this section few alternatives of EM
which all substitute the EM update $\bars^{k+1} = \bars \circ
\map(\bars^k)$ (see Line~\ref{line:EM:updateS} in \autoref{algo:EM})
with an update $\hatS^k \to \hatS^{k+1}$ of the form
\begin{equation}
  \label{eq:SAEM}
\hatS^{k+1} =\hatS^k + \pas_{k+1} \s^{k+1} \eqsp,
\end{equation}
where $\{\pas_k ,k \geq 1\}$ is a deterministic positive sequence of
{\em step sizes} (also called {\em learning rates}) chosen by the user
and $\s^{k+1}$ is an approximation of $h(\hatS^k) = \bars \circ
\map(\hatS^k) - \hatS^k$. When it is a random approximation, the
iterative algorithm described by \eqref{eq:SAEM} is a SA algorithm
designed to target the zeros of the mean field $s \mapsto h(s)$ (see
\eqref{eq:meanfield}); see
e.g. \cite{benveniste:etal:1990,borkar:2008} for a general review on
SA. Many stochastic approximations of EM can be described by
\eqref{eq:SAEM}: let us cite for example the {\tt Stochastic EM}
by~\cite{celeuxd85}, the {\tt Monte Carlo EM} (MCEM, introduced by
\cite{Wei:tanner:1990} and studied by \cite{Fort:moulines:2003}) which
corresponds to $\pas_{k+1}=1$ and the {\tt Stochastic Approximation
  EM} (SAEM) introduced by \cite{Delyon:etal:1999}.

In the finite-sum framework, observe from \eqref{eq:meanfield} that
for any $s \in \rset^q$,
\begin{equation}
  \label{eq:h:to:meanfield}
h(s) = \PE\left[ \bars_I \circ \map(s) + W \right] - s \eqsp,
\end{equation}
where $I$ is a uniform random variable on $[n]^\star$ and $W$ is a
zero-mean random vector. Such an expression gives insights for the
definition of SA schemes, including the combination with a variance
reduction techniques through an adequate choice of $W$ (see e.g.
\cite[Section 4.1.]{Glasserman:2004} for an introduction to control
variates). We review below recent EM-based algorithms, designed for
the finite-sum setting.

\subsubsection{The Fast Incremental EM algorithm}
\label{sec:Fi-EM}
{\tt Fast Incremental EM (FIEM)} was introduced by
\cite{karimi:etal:2019}; it is given in \autoref{algo:FIEM}.
\begin{algorithm}[htbp]
  \KwData{$\kmax \in \nset$, $\hatS^0 \in \Sset$, $\pas_{k} \in
    \ooint{0,\infty}$ for $k \in [\kmax]^\star$} \KwResult{The FIEM
    sequence: $\hatS^k, k \in [\kmax]$} $\Smem_{0,i} = \bars_i \circ
  \map(\hatS^0)$ for all $i \in [n]^\star$\; $\Sronde^0 = n^{-1}
  \sum_{i=1}^n \Smem_{0,i}$\; \For{$k=0, \ldots, \kmax-1$}{Sample
    $I_{k+1}$ uniformly from $[n]^\star$ \label{line:FIEM:debutaux}\; $
    \Smem_{k+1,i} = \Smem_{k,i}$ for $i \neq I_{k+1}$ \;
    $\Smem_{k+1,I_{k+1}} = \bars_{I_{k+1}} \circ \map(\hatS^k)$\;
    $\Sronde^{k+1} = \Sronde^k + n^{-1} \left(\Smem_{k+1,I_{k+1}} -
    \Smem_{k,I_{k+1}}\right)$ \label{line:FIEM:finaux} \; Sample
    $J_{k+1}$ uniformly from $[n]^\star$ \label{line:FIEM:single} \;
    $\hatS^{k+1} = \hatS^k + \pas_{k+1} (\bars_{J_{k+1}} \circ
    \map(\hatS^k) - \hatS^k + \Sronde^{k+1} -
    \Smem_{k+1,J_{k+1}})$ \label{line:FIEM:update}}
    \caption{Fast Incremental EM \label{algo:FIEM}}
\end{algorithm}
Lines~\ref{line:FIEM:debutaux} to \ref{line:FIEM:finaux} are a
recursive computation of $n^{-1} \sum_{i=1}^n \Smem_{k+1,i}$, stored
in $\Sronde^{k+1}$, where for $k \geq 0$,
\begin{equation}
\label{eq:def:Smem}
\Smem_{k+1,i} \eqdef \left\{ \begin{array}{ll}
\bars_{I_{k+1}} \circ \map(\hatS^k) & \ \text{if $i = I_{k+1}$ \eqsp,} \\
\Smem_{k,i} & \ \text{otherwise \eqsp.}
\end{array} \right.
\end{equation}
This procedure avoids the computation of a sum with $n$ terms at each
iteration of FIEM, but at the price of a memory footprint since the
$\rset^q$-valued vectors $\Smem_{k,i}$ for $i \in [n \wedge
  \kmax]^\star$ have to be stored.  Line~\ref{line:FIEM:update} is of
the form \eqref{eq:SAEM} with $\s^{k+1}$ equal to the sum of two
terms: $\bars_{J_{k+1}} \circ \map(\hatS^k) - \hatS^k$ is an oracle
for $\PE\left[\bars_I \circ \map(s) -s \right]$ evaluated at $s =
\hatS^k$; and $W \eqdef \Sronde^{k+1} - \Smem_{k+1,J_{k+1}}$ acts as a
control variate, which conditionally to the past $\F_{k+1/2} \eqdef
\sigma(\hatS^0, I_1, J_1, \ldots, I_k, J_k, I_{k+1}\}$, is centered.
A natural extension, which is not addressed in this paper, is to
replace the draws $I_{k+1}, J_{k+1}$ by mini-batches of examples
sampled in $[n]^\star$ - uniformly, with or without replacement.

The introduction of such a variable $W$ is inherited from the {\tt
  Stochastic Averaged Gradient} (SAGA, by \cite{Defazio:bach:2014}).
The convergence analysis of {\tt FIEM} was given in
\cite{karimi:etal:2019}: they derive nonasymptotic convergence results
in expectation. The theoretical contribution of our paper, detailed in
\autoref{sec:FIEM:complexity}, is to complement and improve these
results.

On the computational side, each iteration of {\tt FIEM} requires two
draws from $[n]^\star$, two expectation evaluations of the form
$\bars_i(\param)$ and a maximization step; there is a space complexity
through the storage of the auxiliary quantity $\Smem_{k, \cdot}$ - its
size being proportional to $q(2 \kmax \wedge n)$ (in some specific
situations, the size can be reduced - see the comment in \cite[Section
  4.1]{Schmidt:2017}). The initialization step also requires a
maximization step and $n$ expectation evaluations.

\subsubsection{An optimized FIEM algorithm, opt-FIEM}
\label{sec:beyondFIEM}
From \eqref{eq:h:to:meanfield}, \autoref{line:FIEM:update} of
\autoref{algo:FIEM} and the control variate technique, we explore here
the idea to modify the original {\tt FIEM} as follows (compare to
\autoref{line:FIEM:update} in \autoref{algo:FIEM})
\begin{multline}
  \label{eq:update:FIEMopt}
\hatS^{k+1} = \hatS^{k} + \pas_{k+1} \left( \bars_{J_{k+1}} \circ
\map(\hatS^k) - \hatS^k   \right. \\
\left. + \lambda_{k+1} \left(\Sronde^{k+1} -
\Smem_{k+1,J_{k+1}}\right)\right)
\end{multline}
where $\lambda_{k+1} \in \rset$ is chosen in order to minimize the
conditional fluctuation
\[
\pas_{k+1}^{-2} \ \PE\left[ \| \hatS^{k+1} - \hatS^{k}\|^2 \vert
  \F_{k+1/2} \right] \eqsp.
\]
Upon noting that $\PE\left[ \hatS^{k+1} - \hatS^{k} \vert \F_{k+1/2}
  \right] = \pas_{k+1} h(\hatS^k)$, it is easily seen that
equivalently, $\lambda_{k+1}$ is chosen as the minimum of the
conditional variance
\[
 \PE\left[ \| \pas_{k+1}^{-1} \left( \hatS^{k+1} - \hatS^{k} \right) -
   h(\hatS^k) \|^2 \vert \F_{k+1/2} \right] \eqsp.
\]
We will refer to this technique as the optimized FIEM ({\tt opt-FIEM})
below; {\tt FIEM} corresponds to the choice $\lambda_{k+1} =1$ for any
$k \geq 0$ and {\tt Online EM} corresponds to the choice
$\lambda_{k+1} =0$ for any $k \geq 0$ (see \autoref{algo:SA}).

Upon noting that, given two random variables $U,V$ such that
$\PE[\|V\|^2] >0$, the function $\lambda \mapsto \PE\left[ \|U +
  \lambda V\|^2 \right]$ reaches its minimum at a unique point given
by $\lambda_\star \eqdef - \PE\left[ U^T V\right] / \PE\left[\|V\|^2
  \right]$, the optimal choice for $\lambda_{k+1}$ is given by
(remember that conditionally to $\F_{k+1/2}$, $\Sronde^{k+1} -
\Smem_{k+1,J_{k+1}}$ is centered),
\begin{equation}
  \label{eq:optimal:lambda}
  \lambda_{k+1}^{\star} \eqdef - \frac{\mathrm{Tr} \
    \mathrm{Cov}\left(\bars_J \circ \map(\hatS^k), \Sronde^{k+1} -
    \Smem_{k+1,J} \vert \F_{k+1/2}\right)}{\mathrm{Tr} \
    \mathrm{Var}\left( \Sronde^{k+1} - \Smem_{k+1,J} \vert
    \F_{k+1/2}\right)}
  \end{equation}
where $J$ is a uniform random variable on $[n]^\star$,
independent of $\F_{k+1/2}$, $\mathrm{Tr}$ denotes the trace of a
matrix, and $\mathrm{Cov}$, $\mathrm{Var}$ are resp. the covariance
and variance matrices. With this optimal value, we have from \eqref{eq:update:FIEMopt}
\begin{align}
& \pas_{k+1}^{-2} \ \PE\left[ \| \hatS^{k+1} - \hatS^{k}\|^2 \vert
  \F_{k+1/2} \right] \nonumber \\ & =\mathrm{Tr} \, \mathrm{Var}\left(\bars_J
\circ \map(\hatS^k) - \hatS^k \vert \F_{k+1/2} \right) \, \cdots \nonumber  \\
& \times \left( 1 -
\mathrm{Corr^2} \left(\bars_J \circ \map(\hatS^k), \Sronde^{k+1} -
  \Smem_{k+1,J} \vert \F_{k+1/2}\right)\right) \eqsp, \label{eq:optFIEM:redvar}
  \end{align}
where
\[
\mathrm{Corr}(U,V) \eqdef \mathrm{Tr} \mathrm{Cov}(U,V) / \{
\mathrm{Tr} \mathrm{Var}(U) \ \mathrm{Tr} \mathrm{Var}(V)
\}^{1/2} \eqsp.
\]

If the {\tt opt-FIEM} algorithm $\{(\hatS^k, \Smem_{k,\cdot}), k \geq
0 \}$ were converging to $(s^\star, \Smem_{\star,\cdot})$, we would
have $n^{-1} \sum_{i=1}^n \Smem_{\star,i} = s^\star = \bars \circ
\map(s^\star)$ and $\Smem_{\star,i} = \bars_i \circ \map(s^\star)$
thus giving intuition that asymptotically when $k \to \infty$,
$\lambda_{k}^\star \approx 1$ (which implies that the correlation is
$1$ in \eqref{eq:optFIEM:redvar}). The value $\lambda=1$ is the value
proposed in the original {\tt FIEM}: therefore, asymptotically {\tt
  opt-FIEM} and {\tt FIEM} should be equivalent and {\tt opt-FIEM}
should have a better behavior in the first iterations of the
algorithm. We will compare numerically {{\tt FIEM}, {\tt opt-FIEM} and
  \tt Online EM} in \autoref{sec:toymodel:exact}.

Upon noting that
\begin{align*}
  \lambda_{k+1}^\star & = - \frac{n^{-1} \sum_{j=1}^n \pscal{\bars_j
      \circ \map(\hatS^k)}{\Sronde^{k+1} - \Smem_{k+1,j}}}{n^{-1}
    \sum_{j=1}^n \|\Sronde^{k+1} - \Smem_{k+1,j} \|^2} \eqsp, \\ &= -
  \frac{n^{-1} \sum_{j=1}^n \pscal{\bars_j \circ
      \map(\hatS^k)}{\Sronde^{k+1} - \Smem_{k+1,j}}}{n^{-1}
    \sum_{j=1}^n \|\Smem_{k+1,j}\|^2 - \|\Sronde^{k+1} \|^2} \eqsp,
\end{align*}
the computational cost of $\lambda_{k+1}^\star$ is proportional to
$n$: it is therefore an intractable quantity in the large scale
learning setting considered in this paper. A numerical approximation
has to be designed: for example, a Monte Carlo approximation of the
numerator; and a recursive approximation (along the iterations $k$) of
the denominator, mimicking the same idea as the recursive computation
of the sum $\Sronde^{k} = n^{-1} \sum_{i=1}^n \Smem_{k,i}$ in {\tt
  FIEM}.

\subsubsection{Online EM}
\label{sec:onlineEM}
      {\tt Online EM} is given by \autoref{algo:SA}; this description
      is a natural extension of the algorithm by
      \cite{cappe:moulines:2009} which was designed to process a
      stream of data.
\begin{algorithm}[htbp]
  \KwData{$\kmax \in \nset$, $\hatS^0 \in \Sset$, $\gamma_{k} \in
    \ooint{0,\infty}$ for $k \in [\kmax]^\star$} \KwResult{The Online
    EM sequence: $\hatS^k, k \in [\kmax]$} \For{$k=0, \ldots,
    \kmax-1$}{Sample $I_{k+1}$ uniformly from
    $[n]^\star$ \label{line:SA:single}\; $\hatS^{k+1}= \hatS^k +
    \pas_{k+1} \left( \bars_{I_{k+1}} \circ \map(\hatS^k) - \hatS^k
    \right)$. \label{line:SA} }
    \caption{Online EM} \label{algo:SA}
\end{algorithm}
{\tt Online EM} is of the form \eqref{eq:SAEM} with $\s^{k+1} \eqdef
\bars_{I_{k+1}} \circ \map(\hatS^k) - \hatS^k $ which corresponds to a
natural oracle for \eqref{eq:h:to:meanfield} when $W=0$. Conditionally
to the past $\hatS^k$, $\s^{k+1}$ is an unbiased approximation of
$h(\hatS^k)$.

Each iteration requires one draw in $[n]^\star$, one expectation
evaluation and one maximization step. Instead of sampling one
observation per iteration, a mini-batch of examples can be used:
line~\ref{line:SA} would get into
\[
\hatS^{k+1} = \hatS^k + \pas_{k+1} \left( \lbatch^{-1} \sum_{i \in
  \batch_{k+1}} \bars_i \circ \map(\hatS^k) - \hatS^k \right)
\]
where $\batch_{k+1}$ is a set of integers of cardinality $\lbatch$,
sampled uniformly from $[n]^\star$, with or without replacement.

 Almost-sure convergence of the iterates in the long-time behavior
 ($\kmax \to \infty$) for {\tt Online EM} was addressed in
 \cite{cappe:moulines:2009}; similar convergence results in the
 mini-batch case for the ML estimation of exponential family mixture
 models were recently established by
 \cite{nguyen:etal:2020}. Nonasymptotic rates for the convergence in
 expectation are derived in \cite{Karimi:miasojedow:2019}.

\subsubsection{The incremental EM algorithm}
\label{sec:i-EM}
The {\tt Incremental EM (iEM)} algorithm is described by
\autoref{algo:iEM}. This description generalizes the original
incremental EM proposed by \cite{Neal:hinton:1998}, which corresponds
to the case $\pas_{k+1} = 1$ and to a deterministic visit to the
successive examples.
\begin{algorithm}[htbp]
  \KwData{$\kmax \in \nset$, $\hatS^0 \in \Sset$, $\gamma_{k} \in
    \ooint{0,\infty}$ for $k \in [\kmax]^\star$} \KwResult{The iEM
    sequence: $\hatS^k, k \in [\kmax]$} $\Smem_{0,i} = \bars_i \circ
  \map(\hatS^0)$ for all $i \in [n]^\star$\; $\Sronde^0 = n^{-1}
  \sum_{i=1}^n \Smem_{0,i}$\; \For{$k=0, \ldots, \kmax-1$}{ Sample
    $I_{k+1}$ uniformly from $[n]^\star$ \label{line:iEM:debutaux}\; $
    \Smem_{k+1,i} = \Smem_{k,i}$ for $i \neq
    I_{k+1}$ \label{line:iEM:debut}\; $\Smem_{k+1,I_{k+1}} =
    \bars_{I_{k+1}} \circ
    \map(\hatS^k)$ \label{line:iEM:updatecompo}\; $\Sronde^{k+1} =
    \Sronde^k + n^{-1} \left(\Smem_{k+1,I_{k+1}} -
    \Smem_{k,I_{k+1}}\right)$ \label{line:iEM:finaux} \; $\hatS^{k+1}
    = \hatS^k + \pas_{k+1} (\Sronde^{k+1} -
    \hatS^k)$ \label{line:iEM:fin}}
    \caption{incremental EM \label{algo:iEM}}
\end{algorithm}
As for FIEM, Lines~\ref{line:iEM:debutaux} to \ref{line:iEM:finaux}
are a recursive computation of $\Sronde^{k+1} = n^{-1} \sum_{i=1}^n
\Smem_{k+1,i}$; and the update mechanism in Line~\ref{line:iEM:fin} is
of the form \eqref{eq:SAEM} with $\s^{k+1} \eqdef \Sronde^{k+1} -
\hatS^k$. Conditionally to the past $\sigma(\hatS^0, I_1, \ldots,
I_{k})$, $\s^{k+1}$ is a {\em biased} approximation of $h(\hatS^k)$.

 \autoref{algo:iEM} can be adapted in order to use a mini-batch of
 examples per iteration: the data set is divided into $B$ blocks prior
 running {\tt iEM}.  \cite{Ng:mclachlan:2003} provided a numerical
 analysis of the role of $B$ when {\tt iEM} is applied to fitting a
 normal mixture model with fixed number of components;
 \cite{gunawardana:2005} provided sufficient conditions for the
 convergence in likelihood in the case the $B$ blocks are visited
 according to a deterministic cycling.

Per iteration, the computational cost of {\tt iEM} is one draw, one
expectation evaluation and one maximization step. As for {\tt FIEM},
there is a memory footprint for the storage of the $\rset^q$-valued
vectors $\Smem_{k,i}$ for $i \in [n \wedge \kmax]^\star$.  The
initialization requires $n$ expectation evaluations and one maximization
step.

%%%%%%%%%%%%%%%%%
%%%%%%%%%%%%%%%%%%%%%%%%%%%%%%%%%%%%%%%%%%%%%%%%%%%%%%%%%%
\section{Nonasymptotic  bounds for convergence in expectation}
\label{sec:FIEM:complexity}
%%%%%%%%%%%%%%%%%%%%%%%%%%%%%%%%%%%%%%%%%%%%%%%%%%%%%%%%%%
%%%%%%%%%%%%%%%%%
The bounds are obtained by strengthening A\autoref{hyp:regV} with the
following assumptions
\begin{assumption} \label{hyp:regV:bis}
  \begin{enumerate}[label=(\roman*)]
   \item \label{hyp:regV:C1:vmax} There exist $0 < v_\min \leq
     v_{max}< \infty $ such that for all $s\in \Sset$, the spectrum of
     $B(s)$ is in $\ccint{v_\min, v_\max}$; $B(s)$ is defined in
     A\autoref{hyp:regV}.
        \item \label{hyp:Tmap:smooth} For any $i \in [n]^\star$,
  $\bars_i \circ \map$ is globally Lipschitz on $\Sset$ with constant
  $L_i$.
      \item \label{hyp:regV:DerLip} The function $s \mapsto \dot
        \lyap(s) = - B(s) h(s)$ is globally Lipschitz on $\Sset$ with
        constant $L_{\dot \lyap}$.
        \end{enumerate}
\end{assumption}

\subsection{A general result}
Finding a point $\hat \param^\epsilon$ such that $F(\hat
\param^\epsilon) - \min F \leq \epsilon$ is NP-hard in the non-convex
setting (see \cite{murty:kabadi:1987}). Hence, in non-convex
deterministic optimization of a smooth function $F$, convergence is
often characterized by the quantity $\inf_{1 \leq k \leq \kmax} \|
\nabla F(\param^k)\|$ along a path of length $\kmax$; in non-convex
stochastic optimization, the quantity $\inf_{1 \leq k \leq \kmax}
\PE\left[ \| \nabla F(\param^k)\|^2 \right]$ is sometimes considered
when the expectation is w.r.t. the randomness introduced to replace
intractable quantities with oracles.  Nevertheless, in many frameworks
such as the finite-sum optimization one we are interested in, such a
criterion can not be used to define a termination rule for the
algorithm since $\nabla F$ is intractable.

For EM-based methods in the expectation space, \autoref{lem:nablaV}
and \eqref{eq:meanfield} imply that the convergence can be
characterized by a "distance" of the path $\{\hatS^k, k \geq 0\}$ to
the set of the roots of $h$. We therefore introduce the following
criteria: given a maximal number of iterations $\kmax$, and a random
variable $K$ taking values in $[\kmax-1]$, define
\begin{align*}
 \mathsf{E_0} &\eqdef \frac{1}{v_\max^2} \PE\left[ \| \dot
   \lyap(\hatS^K)\|^2 \right] \eqsp, \\ \mathsf{E_1} &\eqdef \PE\left[
   \| h(\hatS^K) \|^2 \right] \eqsp, \\ \mathsf{E_2} & \eqdef
 \PE\left[ \| \Sronde^{K+1} - \bars \circ \map(\hatS^K) \|^2 \right]
 \eqsp,
\end{align*}
where $K$ is chosen independently of the path. Upper bounds of these
quantities provide a control of convergence in expectation for FIEM
stopped at the random time $K$. Below $K$ is the uniform r.v. on
$[\kmax-1]$, except in \autoref{sec:FIEM:errorrate:case3}.

The quantities $\mathsf{E_0}$ and $\mathsf{E_1}$ are classical in the
literature: they stand for a measure of resp. a distance to a
stationary point of the objective function $\lyap = F \circ \map$, and
a distance to the fixed points of EM. $\mathsf{E_2}$ is specific to
FIEM: it quantifies how far the control variate $\Sronde^{k+1}$ is
from the intractable mean $\bars \circ \map(\hatS^k)$ (see
\autoref{sec:Fi-EM} for the definition of $\Sronde^{k+1}$). Under our
assumptions, $\mathsf{E_0}$ and $\mathsf{E_1}$ are related as stated
in \autoref{lem:fromVdot:to:h}, which is a straightforward consequence
of \autoref{lem:nablaV}.
\begin{proposition}
  \label{lem:fromVdot:to:h}
  Assume A\autoref{hyp:model}, A\autoref{hyp:bars},
  A\autoref{hyp:Tmap}, A\autoref{hyp:regV} and
  A\autoref{hyp:regV:bis}-\ref{hyp:regV:C1:vmax}. For any $s \in
  \Sset$, we have $\pscal{h(s)}{\dot \lyap(s)} \leq - v_\min
  \|h(s)\|^2$ and $\mathsf{E}_0 \leq \mathsf{E}_1$.
\end{proposition}
Theorem~\ref{theo:FIEM:NonUnifStop} is a general result for the
control of quantities of the form
\[
\sum_{k=0}^{\kmax-1} \left\{ \alpha_k \PE\left[ \|h(\hatS^k) \|^2 \right] +
 \delta_k \PE\left[ \| \Sronde^{k+1} - \bars \circ
  \map(\hatS^k) \|^2 \right] \right\}
\]
where $\alpha_k \in \rset$ and $\delta_k >0$. In
\autoref{sec:FIEM:errorrate:case1} and
\autoref{sec:FIEM:errorrate:case2}, we discuss how to choose the step
sizes $\{\pas_k, k \geq 1\}$ such that for any $k \in [\kmax-1]$,
$\alpha_k $ is non-negative and such that $A_{\kmax} \eqdef
\sum_{k=0}^{\kmax-1} \alpha_k$ is positive. We then deduce from
\autoref{theo:FIEM:NonUnifStop} an upper bound for
\begin{multline}\label{eq:FIEM:error}
\sum_{k=0}^{\kmax-1} \frac{\alpha_k}{A_{\kmax}}\PE\left[\| h(\hatS^k) \|^2 \right]  
+ \sum_{k=0}^{\kmax-1} \frac{\delta_k}{A_{\kmax}} \PE\left[ \|
  \Sronde^{k+1} - \bars \circ \map(\hatS^k) \|^2 \right]
\end{multline}
such that the larger $A_{\kmax}$ is, the better the bound
is. \eqref{eq:FIEM:error} is then used to obtain upper bounds on
$\mathsf{E}_1$ and $\mathsf{E}_2$; which provide in turn an upper
bound on $\mathsf{E}_0$ by \autoref{lem:fromVdot:to:h}.
\begin{theorem} \label{theo:FIEM:NonUnifStop}
Assume A\autoref{hyp:model}, A\autoref{hyp:bars}, A\autoref{hyp:Tmap},
A\autoref{hyp:regV} and A\autoref{hyp:regV:bis}.  Define $L^2 \eqdef
n^{-1} \sum_{i=1}^n L_i^2$.

Let $\kmax$ be a positive integer, $\sequencedown{\gamma}{k}{\nset}$
be a sequence of positive step sizes and $\hatS^0 \in \Sset$. Consider
the FIEM sequence $\sequence{\hatS}{k}{[\kmax]}$ given by
\autoref{algo:FIEM}. Set $\Delta \lyap \eqdef \PE\left[\lyap(\hatS^0)
  \right] - \PE\left[\lyap(\hatS^{\kmax}) \right]$.

We have
\begin{align*}
 \sum_{k=0}^{\kmax -1} \alpha_k \ \PE\left[ \| h(\hatS^k)\|^2 \right]
 + \sum_{k=0}^{\kmax -1} \delta_k \PE\left[ \| \Sronde^{k+1}- \bars
   \circ \map(\hatS^k) \|^2 \right] \leq \Delta \lyap \eqsp,
\end{align*}
with, for any $k \in [\kmax-1]$,
\begin{align*}
\alpha_k & \eqdef \gamma_{k+1} v_{min} - \gamma_{k+1}^2 \left( 1 +
\Lambda_{k} L^2\right) \frac{L_{\dot \lyap}}{2}  \eqsp, \\
\delta_k & \eqdef
\gamma_{k+1}^2 \left( 1 + \frac{ \Lambda_k \beta_{k+1} L^2}{ \left(1+
  \beta_{k+1} \right)} \right) \frac{L_{\dot\lyap}}{2} \eqsp,
\end{align*}
where $\beta_{k+1}$ is any positive number, and for $k \in [\kmax-2]$,
\begin{align*}
\Lambda_k & \eqdef \left(1 + \frac{1}{\beta_{k+1}} \right) \cdots
\qquad \times \sum_{j=k+1}^{\kmax-1} \gamma_{j+1}^2
\ \prod_{\ell=k+2}^j \left(1 - \frac{1}{n}+ \beta_\ell + \gamma_\ell^2
L^2 \right) \eqsp.
\end{align*}
By convention, $\Lambda_{\kmax-1} = 0$.
\end{theorem}
\begin{proof}
\label{pageref:sketch}
The detailed proof is in Section~\ref{sec:proofs}; let us give here a
sketch of proof.  Define $H_{k+1}$ such that $\hatS^{k+1} = \hatS^k +
\gamma_{k+1} H_{k+1}$.  $\lyap$ is regular enough so that
\begin{align*}
  \lyap(\hatS^{k+1}) - \lyap(\hatS^k) - \gamma_{k+1}
  \pscal{H_{k+1}}{\dot \lyap(\hatS^k)} \leq \gamma_{k+1}^2
  \frac{L_{\dot \lyap}}{2} \|H_{k+1}\|^2 \eqsp.
  \end{align*}
  Then, the next step is to prove that
\begin{align*}
& \PE\left[ \lyap(\hatS^{k+1}) \right] - \PE\left[\lyap(\hatS^k)
    \right]+ \pas_{k+1} \left(v_\min - \pas_{k+1}
  \frac{L_{\dot \lyap}}{2} \right) \PE\left[\| h(\hatS^k) \|^2 \right]
  \\ & \leq \pas_{k+1}^2 \frac{L_{\dot \lyap}}{2} \PE\left[\|H_{k+1} -
    \PE\left[H_{k+1} \vert \F_{k+1/2} \right]\|^2 \right] \eqsp,
\end{align*}
which, by summing from $k=0$ to $k=\kmax-1$, yields
\begin{multline}
\sum_{k=0}^{\kmax-1} \pas_{k+1} \left(v_\min - \pas_{k+1}
\frac{L_{\dot \lyap}}{2} \right)  \PE\left[\| h(\hatS^k) \|^2 \right] \\
\leq \PE\left[ \lyap(\hatS^{0}) \right] - \PE\left[\lyap(\hatS^{\kmax})
  \right] \nonumber \\  + \frac{L_{\dot \lyap}}{2}
\sum_{k=0}^{\kmax-1} \pas_{k+1}^2 \PE\left[\|H_{k+1} -
  \PE\left[H_{k+1} \vert \F_{k+1/2} \right]\|^2 \right]
\eqsp.\label{eq:condition:lambdak}
\end{multline}
The most technical part is to prove that the last term on the RHS is
upper bounded by
\begin{multline*}
  \frac{L_{\dot \lyap}}{2} \sum_{k=0}^{\kmax-1} \pas_{k+1}^2 L^2
 \ \left\{  \Lambda_{k} \PE\left[\| h(\hatS^k) \|^2 \right]
  \right. \\ \left. - \left( 1+ (1+\beta_{k+1}^{-1})^{-1}  \Lambda_{k}  \right) \ \PE\left[\|
    \Sronde^{k+1} - \bars \circ \map(\hatS^k) \|^2 \right] \right\}
  \eqsp.
\end{multline*}
This concludes the proof.
\end{proof}

In the Stochastic Gradient Descent literature, complexity is evaluated
in terms of {\em Incremental First-order Oracle} introduced by
\cite{agarwal:bottou:2015}, that is, roughly speaking, the number of
calls to an oracle which returns a pair $(f_i(x), \nabla f_i(x))$. In
our case, the equivalent cost is the number of expectation evaluations
$\bars_i(\param)$ and the number of optimization steps $s \mapsto
\map(s)$. $\kmax$ iterations of FIEM calls $2 \kmax$ evaluations of
such expectations and $\kmax$ optimization steps. As a consequence,
the complexity analyses consist in discussing how $\kmax$ has to be
chosen as a function of $n$ and $\epsilon$ in order to reach an
$\epsilon$-approximate stationary point defined by $\mathsf{E}_1
\leq\epsilon$.

\subsection{A uniform random stopping rule for a $n^{2/3}$-complexity}
\label{sec:FIEM:errorrate:case1}
The main result of this section establishes that by choosing a
constant step size and a termination rule $K$ sampled uniformly from
$[\kmax -1]$, an $\epsilon$-approximate stationary point can be
reached before
\[
\kmax = O(n^{2/3}
\epsilon^{-1} L_{\dot \lyap}^{1/3} L^{2/3})
\] iterations.

For $\lambda \in \ooint{0,1}$, $C>0$ and $n$ such that $n^{-1/3} <
\lambda/C$, define
\begin{equation}\label{eq:fn:statement}
f_n(C,\lambda) \eqdef \left( \frac{1}{n^{2/3}}+ \frac{C}{\lambda -
  C/n^{1/3}} \left(\frac{1}{n} + \frac{1}{1-\lambda} \right)\right)
\eqsp.
\end{equation}
\begin{proposition}[application of Theorem~\ref{theo:FIEM:NonUnifStop}]\label{coro:optimal:sampling}
  Let $\mu \in \ooint{0,1}$. Choose $\lambda \in \ooint{0,1}$ and $C
  \in \ooint{0, +\infty}$ such that
  \begin{equation} \label{eq:bounds:def:C}
\sqrt{C} f_n(C,\lambda) = 2 \mu v_\min \frac{L}{L_{\dot \lyap}}\eqsp.
\end{equation}
Let $\{\hatS^k, k \in \nset \}$ be the FIEM sequence given by
\autoref{algo:FIEM} run with the constant step size
  \begin{equation} \label{eq:C:uniform}
\pas_\ell = \pas_{\fgm} \eqdef \frac{\sqrt{C}}{n^{2/3} L} = \frac{2 \mu
  v_\min}{f_n(C,\lambda) \, n^{2/3} L_{\dot \lyap}} \eqsp.
\end{equation}
For any $n > (C/\lambda)^{3}$ and $\kmax \geq 1$, we
have \begin{align} \mathsf{E}_1 &+ \frac{\mu}{(1-\mu)f_n(C,\lambda) \,
    n^{2/3}} \mathsf{E_2}  \leq \frac{n^{2/3}}{\kmax}
  \frac{L_{\dot \lyap} \, f_n(C,\lambda)}{2 \mu (1-\mu) v_\min^2}
  \Delta \lyap \eqsp, \label{eq:borne:optimal:sampling}
\end{align}
where the errors $\mathsf{E}_i$ are defined with a random variable $K$
sampled uniformly from $[\kmax -1]$.
\end{proposition} The proof of \autoref{coro:optimal:sampling} is in \autoref{sec:proof:coro:optimal:sampling}.  The first
suggestion to solve the equation \eqref{eq:bounds:def:C} is to choose
$\lambda =C$ and $C \in \ooint{0,1}$ such that
\[
\sqrt{C} f_n(C,C) = 2 \mu v_\min L /L_{\dot \lyap} \eqsp.
\] This equation possesses an unique solution $C^\star$ in $\ooint{0,1}$  which
is upper bounded by $C^+$ given by
\[
C^+ \eqdef \frac{\sqrt{1+ 16 \mu^2 v_\min^2 L^2 L_{\dot \lyap}^{-2}} -1}{4 \mu v_\min L L_{\dot \lyap}^{-1}} \eqsp.
\]
The consequence is that, given $\varepsilon \in \ooint{0,1}$,  by setting
\begin{align*}
M \eqdef \frac{L_{\dot \lyap}}{2 \mu (1-\mu) v_\min^2} f_n(C^\star,
C^\star) \leq \frac{L_{\dot \lyap}}{2 \mu (1-\mu) v_\min^2} f_2(C^+,
C^+) \eqsp,
\end{align*}
we have
\[
  \kmax = M \ n^{2/3} \varepsilon^{-1} \Longrightarrow \mathsf{E}_1 +
  \frac{L_{\dot \lyap}}{2(1-\mu) L} \frac{\sqrt{C^\star} }{v_\min
    n^{2/3}} \mathsf{E_2} \leq \varepsilon \, \Delta \lyap \eqsp;
\]
see \autoref{secApp:proof:coro:optimal:sampling} in the supplementary
material for a detailed proof of this comment.

Another suggestion is to exploit how \eqref{eq:fn:statement} behaves
when $n \to +\infty$; we prove in the supplementary material
(\autoref{secApp:proof:coro:optimal:sampling}) that there exists
$N_\star$ depending only upon $L, L_{\dot \lyap}, v_\min$ such that
for any $n \geq N_\star$,
\[
\mathsf{E}_1 + \frac{1}{3 n^{2/3}} \left(\frac{L_{\dot \lyap}}{L
  v_\min} \right)^{2/3} \mathsf{E_2} \leq \frac{n^{2/3}}{\kmax}
\frac{8}{3} \frac{L}{v_\min} \left( \frac{L_{\dot \lyap}}{L
  v_\min}\right)^{1/3} \ \Delta \lyap \eqsp,
\]
by choosing $C \leftarrow 0.25 \, \left(v_\min L / L_{\dot
  \lyap}\right)^{2/3}$ in the definition of the step size $\pas_\fgm$.

The conclusions of \autoref{coro:optimal:sampling} confirm and improve
previous results in the literature: \cite[Theorem 2]{karimi:etal:2019}
proved that for FIEM applied with the constant step size
\begin{equation} \label{eq:stepsize:karimi}
\pas_{\km} \eqdef \frac{v_{\min} n^{-2/3}}{\max(6, 1+ 4 v_\min) \ \max(L_{\dot
    \lyap}, L_1, \ldots, L_n) } \eqsp,
\end{equation}
there holds
\begin{equation} \label{eq:control:karimi}
  \mathsf{E}_1 \leq \frac{n^{2/3}}{\kmax} \ \Delta \lyap
  \ \frac{\left( \max(6, 1+ 4 v_\min) \right)^2 \ \max(L_{\dot \lyap},
    L_1, \cdots, L_n)}{v_\min^2}\eqsp.
\end{equation}

We improve this result. Firstly, we show that the RHS in
\eqref{eq:borne:optimal:sampling} controls a larger quantity than
$\mathsf{E}_1$. Secondly, numerical explorations (see
e.g. \autoref{sec:toymodel:exact}) show that $\pas_{\fgm}$ is larger
than $\pas_\km$ thus providing a more aggressive step size which may
have a beneficial effect on the efficiency of the algorithm. Thirdly,
these numerical illustrations also show that
\autoref{coro:optimal:sampling} provides a tighter control of the
convergence in expectation. In both contributions however, the step
size depends upon $n$ as $O(n^{-2/3})$ and the bounds depend on $n$
and $\kmax$ resp. as the increasing function $n \mapsto n^{2/3}$ and
the decreasing function $\kmax \to 1/\kmax$. The dependence upon $n$
of the step size is the same as what was observed for Stochastic
Gradient Descent (see e.g. \cite{allenzhu:hazan:2016}).

\subsection{A uniform random stopping rule for a $\sqrt{n}$-complexity}
\label{sec:FIEM:errorrate:case2}
Here again, we consider an FIEM path run with a constant step size and
stopped at a random time $K$ sampled uniformly from $[\kmax -1]$: we
prove that an $\epsilon$-stationary point can be reached before
\[
\kmax = O( \sqrt{n} \epsilon^{-3/2})
\]
iterations. Define
\begin{equation}
  \label{eq:fn:statement:Ketn}
\tilde f_n(C,\lambda) \eqdef \frac{1}{(n \kmax)^{1/3}} + C
\left(\frac{1}{n} + \frac{1}{1-\lambda} \right) \eqsp.
\end{equation}

\begin{proposition}[application of Theorem~\ref{theo:FIEM:NonUnifStop}]
  \label{coro:optimal:sampling:Ketn}
  Let $\mu \in \ooint{0,1}$. Choose $\lambda \in \ooint{0,1}$ and $C
  >0$ such that
  \begin{equation} \label{eq:bounds:def:C:Ketn}
  \sqrt{C} \tilde f_n(C,\lambda) = 2 \mu v_\min \frac{L}{L_{\dot \lyap}}\eqsp.
  \end{equation}
Let $\{\hatS^k, k \in [\kmax] \}$ be the FIEM sequence given by
\autoref{algo:FIEM} run with the constant step size
  \begin{align} \label{eq:C:uniform:Ketn}
\pas_\ell & = \tilde \pas_{\fgm} \eqdef \frac{\sqrt{C}}{n^{1/3}
  \kmax^{1/3} L} \frac{2 \mu v_\min}{ L_{\dot \lyap} \tilde
  f_n(C,\lambda) \, n^{1/3} \kmax^{1/3} }  \eqsp.
\end{align}
 For any positive integers $n,\kmax$ such that $n^{1/3} \kmax^{-2/3}
 \leq \lambda/C$, we have
  \[
 \mathsf{E_1} + \frac{\mu}{(1-\mu) \tilde f_n(C,\lambda)}
 \frac{1}{(n\kmax)^{1/3}} \mathsf{E_2} \leq
 \frac{n^{1/3}}{\kmax^{2/3}} \frac{L_{\dot \lyap} \, \tilde
   f_n(C,\lambda)}{2 \mu (1-\mu) v_\min^2} \, \Delta \lyap \eqsp,
  \]
  where the errors $\mathsf{E}_i$ are defined with a random variable
  $K$ sampled uniformly from $[\kmax -1]$.
\end{proposition}
The proof of \autoref{coro:optimal:sampling:Ketn} is in
\autoref{sec:proof:coro:optimal:sampling:Ketn}.  From this upper
bound, it can be shown (see \autoref{secApp:errorrate:case2} in the
supplementary material) that for any $\tau >0$, there exists $M>0$
depending upon $L, L_{\dot \lyap}, v_\min, \mu$ and $\tau$ such that
for any $\varepsilon >0$,
\[
\kmax \geq \left( \sqrt{n} \tau^{3/2} \right) \vee \left( M \sqrt{n}
\varepsilon^{-3/2} \right)  \Longrightarrow
\frac{n^{1/3}}{\kmax^{2/3}} \frac{L_{\dot \lyap} \, \tilde f_n(\lambda
  \tau ,\lambda)}{ 2 \mu (1-\mu) v_\min^2} \leq \varepsilon \eqsp.
\]
To our best knowledge, this is the first result in the literature
which establishes a nonasymptotic control for FIEM at such a rate: the
upper bound depends on $n$ as the increasing function of $n \mapsto
n^{1/3}$ and depends on $\kmax$ as the decreasing function of $\kmax
\mapsto \kmax^{-2/3}$.

As a corollary of \autoref{coro:optimal:sampling} and
\autoref{coro:optimal:sampling:Ketn}, we have two upper bounds of the
errors $\mathsf{E}_1, \mathsf{E}_2$: the first one is $O(n^{2/3}
\kmax^{-1})$ and the second one is $O(n^{1/3} \kmax^{-2/3})$. The
first or second strategy will be chosen depending on the accuracy
level $\varepsilon$: if $\varepsilon = n^{-\mathsf{e}}$ for some
$\mathsf{e}>0$, then we have to choose $\kmax = O(n^{2/3}
\varepsilon^{-1}) = O(n^{2/3+\mathsf{e}})$ in the first strategy and
$\kmax = O(\sqrt{n} \varepsilon^{-3/2}) = O(n^{1/2+3 \mathsf{e}/2})$
in the second one; if $\mathsf{e} \in \ooint{0,1/3}$, the second
approach is preferable.

When $\kmax = A \sqrt{n} \epsilon^{-3/2}$, then the constant step size
is $\tilde \pas_{\fgm} = \sqrt{C\epsilon} (L A^{1/3} \sqrt{n})^{-1}$.
In the case
$\sqrt{n} \epsilon^{-3/2} < \tilde A n^{2/3} \epsilon^{-1}$, we have
$\tilde \pas_{\fgm} > \sqrt{C} / (L A^{1/3} \tilde A n^{2/3})$ thus
showing that the step size is lower bounded by $O(n^{-2/3})$ (see
$\pas_{\fgm}$ in \autoref{coro:optimal:sampling}).  We have
$\tilde \pas_\fgm \propto 1/\sqrt{n}$ when $\kmax \propto \sqrt{n}$:
the result of \autoref{coro:optimal:sampling:Ketn} is obtained with a
slower step size (seen as a function of $n$) than what was required in
\autoref{coro:optimal:sampling}.

We now discuss a choice for the pair $(\lambda,C)$ which exploits how
\eqref{eq:fn:statement:Ketn} behaves when $n \to +\infty$; we prove in
\autoref{secApp:errorrate:case2} in the supplementary material that
for any $\tau >0$, there exists $N_\star$ depending only upon $L,
L_{\dot \lyap}, v_\min, \tau$ such that for any $N_\star \leq n \leq
\tau^3 \kmax^2$,
\begin{multline*}
\mathsf{E_1} + \frac{2^{10/3} (1-\lambda_\star)^{-1/3} \mu^2 }{\tilde
  f_n^2(\lambda_\star \tau,\lambda_\star)}\left( \frac{L
  v_\min}{L_{\dot \lyap}} \right)^{2/3} \frac{1}{(n\kmax)^{1/3}}
\mathsf{E_2} \\ \leq \frac{n^{1/3}}{\kmax^{2/3}} \frac{4}{3} \left(
\frac{ 2 L^2 L_{\dot \lyap}}{v_\min^4 }\right)^{1/3}
(1-\lambda_\star)^{-1/3} \ \Delta \lyap \eqsp,
\end{multline*}
where $\lambda_\star$ is the unique solution of $\left(v_\min L
\right)^2 \tau^3 (1-\lambda_\star)^2 = (2 L_{\dot \lyap})^2
\lambda^3_\star$.

\subsection{A non-uniform random termination rule}
\label{sec:FIEM:errorrate:case3}
Given a distribution $p_0, \ldots, p_{\kmax-1}$ for the r.v. $K$, we
show how to fix the step sizes $\pas_1, \ldots, \pas_{\kmax}$ in order
to deduce from \autoref{theo:FIEM:NonUnifStop} a control of the errors
$\mathsf{E}_1$ and $\mathsf{E}_2$.  For $\lambda \in \ooint{0,1}$,
$C>0$ and $n > (C/\lambda)^3$, define the function $F_{n,C,\lambda}$
  \begin{align*}
  F_{n,C,\lambda}: x & \mapsto \frac{L_{\dot \lyap}}{2L^2 n^{2/3}} x \left( v_\min \frac{2L }{L_{\dot \lyap}}- x f_n(C,\lambda)
  \right) \eqsp,
  \end{align*}
  where $f_n$ is defined by \eqref{eq:fn:statement}. $F_{n,C,\lambda}$
  is positive, increasing and continuous on $\ocint{0, v_\min L/(L_{\dot \lyap}
    f_n(C,\lambda))}$.

\begin{proposition}[application of \autoref{theo:FIEM:NonUnifStop}]
  \label{coro:given:sampling} Let $K$ be a $[\kmax-1]$-valued random variable with positive weights  $p_0, \ldots, p_{\kmax-1}$.
  Choose $\lambda \in \ooint{0,1}$ and $C>0$ such that
   \begin{equation}
     \label{eq:FIEM:NonUnifStep:C}
\sqrt{C} \, f_n(C,\lambda) =  v_\min \frac{L}{ L_{\dot \lyap}}\eqsp.
   \end{equation}
For any $n > (C/\lambda)^3$ and $\kmax \geq 1$, we have
 \begin{multline*}
 \mathsf{E}_1 + \frac{L_{\dot \lyap}^2}{ v_\min^2} n^{2/3} \, \max_k
 p_k \, f_n(C,\lambda) \ \sum_{k=0}^{\kmax-1} \pas_{k+1}^2 \PE\left[
   \| \Sronde^{k+1} - \bars \circ \map(\hatS^k) \|^2 \right] \\ \leq
 n^{2/3} \ \max_k p_k \, \frac{2 L_{\dot \lyap} \, f_n(C,\lambda)
 }{v_\min^2} \ \Delta \lyap \eqsp,
  \end{multline*}
  where the FIEM sequence $\{\hatS^k, k \in [\kmax] \}$ is obtained
  with
  \[
  \gamma_{k+1} = \frac{1}{n^{2/3} L} \ F^{-1}_{n,C,\lambda}
  \left(\frac{p_k}{\max_\ell p_\ell} \frac{v_\min^2}{2 L_{\dot \lyap}
    f_n(C,\lambda)} \frac{1}{n^{2/3}}\right) \eqsp.
  \]
\end{proposition}
The proof of \autoref{coro:given:sampling} is in
\autoref{sec:proof:coro:givensample}.
 As already commented in \autoref{sec:FIEM:errorrate:case1}, if we
 choose $\lambda=C$, then \eqref{eq:FIEM:NonUnifStep:C} gets into
 \[
 \sqrt{C} \ \left( \frac{1}{n^{2/3}}+ \frac{1}{1 - n^{-1/3}}
 \left(\frac{1}{n} + \frac{1}{1-C} \right)\right) = \frac{v_\min
   L}{L_{\dot \lyap}} \eqsp.
 \]
 There exists an unique solution $C^\star$, which is upper bounded by
 a quantity which only depends upon the quantities $L, L_{\dot \lyap},
 v_\min$; hence, so $f_n(C^\star, C^\star)$ is and the control of
 $\mathsf{E}_i$ given in \autoref{coro:given:sampling} depends on $n$
 at most as $n \mapsto n^{2/3}$ and on $\kmax$ as $\kmax \mapsto
 \max_k p_k$.

 If we choose $\lambda=1/2$, the constant $C$ satisfies $C \leq
 \left(v_\min L / (4L_{\dot \lyap}) \right)^{2/3}$ (see
 \autoref{secApp:errorrate:case3} in the supplementary material), and
 the nonasymptotic control given by \autoref{coro:given:sampling} is
 available for $8 n > (v_\min L/ L_{\dot \lyap})^2$. 

Since $\sum_k p_k=1$, we have
$\max_k p_k \geq 1/\kmax$ thus showing that among the distributions
$\{p_j, j \in [\kmax -1]\}$, the quantity $\max_k p_k$ is minimal with
the uniform distribution. In that case, the results of
\autoref{coro:given:sampling} can be compared to the results of
\autoref{coro:optimal:sampling}: both RHS are increasing functions of
$n$ at the rate $n^{2/3}$; both are decreasing functions of $\kmax$ at
the rate $1/\kmax$; the constants $C,\lambda$ solving the equality in
\eqref{eq:bounds:def:C} in the case $\mu=1/2$ are the same as the
constants $C,\lambda$ solving \eqref{eq:FIEM:NonUnifStep:C}: as a
consequence,
\[
\frac{2 L_{\dot \lyap} f_n(C,\lambda) }{v_\min^2} = \frac{L_{\dot \lyap} f_n(C,\lambda)}{2 \mu(1-\mu) v_\min^2} \eqsp, \qquad \mu = 1/2.
\]
Finally, when $k \mapsto p_k$ is constant, the step sizes given by
\autoref{coro:given:sampling} are constant as in
\autoref{coro:optimal:sampling}; and they are equal since
\[
F_{n,C,\lambda}^{-1}\left( \frac{v^2_\min n^{-2/3}}{2L_{\dot \lyap} f_n(C,\lambda)} \right) = \sqrt{C}
= \frac{v_\min L}{L_{\dot \lyap} f_n(C,\lambda)} \eqsp.
\]
 Hence \autoref{coro:given:sampling} and
 \autoref{coro:optimal:sampling} are the same when $p_k = 1/\kmax$ for
 any $k$.

%%%%%%%%%%%%%%%%%%%%%%%%%%%%%%%%%%%%%%%%%%%%%
%%%%%%%%%%%%%%%%%%
%%%%%%%%%%%%%%%%%%
\section{A toy example}\label{sec:toymodel:exact}
In this section, we consider a very simple optimization problem which
could be solved without requiring the incremental EM
machinery~\footnote{The numerical applications are developed in MATLAB
  by the first author of the paper. The code files are publicly
  available from
  https://github.com/gfort-lab/OpSiMorE/tree/master/FIEM \label{footnote:github}}

$\mathcal{N}_p(\mu,\Gamma)$ denotes a $\rset^p$-valued Gaussian
distribution, with expectation $\mu$ and covariance matrix $\Gamma$.

\subsection{Description}
$n$ $\rset^y$-valued observations are modeled as the realization of
$n$ vectors $Y_i \in \rset^y$ whose distribution is described as
follows: conditionally to $(Z_1, \ldots, Z_n)$, the r.v. are
independent with distribution $Y_i \sim \mathcal{N}_y(A Z_i,\Id_y)$
where $A \in \mathbb{R}^{y \times p}$ is a deterministic matrix and
$\Id_y$ denotes the $y \times y$ identity matrix; $(Z_1,\ldots,Z_n)$
are i.i.d. under the distribution $\mathcal{N}_p(X \theta,\Id_p)$,
where $\theta \in \Theta \eqdef \rset^q$ and $X \in \rset^{p \times
  q}$ is a deterministic matrix. Here, $X$ and $A$ are known, and
$\theta$ is unknown; we want to estimate $\theta$, as a solution of a
(possibly) penalized maximum likelihood estimator, with penalty term
$\rho(\theta) \eqdef \upsilon \|\theta\|^2 /2$ for some $\upsilon \geq
0$. If $\upsilon =0$, it is assumed that the rank of $X$ and $AX$ are
resp. $q = q \wedge y$ and $p = p \wedge y$.  In this model, the
r.v. $(Y_1,\ldots,Y_n)$ are i.i.d. with distribution $\mathcal{N}_y(A
X \theta; \Id_y + A A^T)$.  The minimum of the function $\theta
\mapsto F(\theta) \eqdef - n^{-1} \log g(Y_{1:n};\theta) +
\rho(\theta)$, where $ g(Y_{1:n};\cdot)$ denotes the likelihood of the
vector $(Y_1, \ldots, Y_n)$, is unique and is given by
 \begin{align*}
   \theta_\star & \eqdef \left(\upsilon \Id_q + X^T A^T \left(\Id_y +
   A A^T\right)^{-1} A X \right)^{-1} \ X^T A^T \left(\Id_y + A
   A^T\right)^{-1} \barY_n \eqsp, \\ \barY_n & \eqdef \frac{1}{n}
   \sum_{i=1}^n Y_i \eqsp.
 \end{align*}
 Nevertheless, using the above description of the distribution of
 $Y_i$, this optimization problem can be cast into the general
 framework described in Section~\ref{sec:motivation}. The loss
 function (see \eqref{eq:def:loss}) is the normalized negative
 log-likelihood of the distribution of $Y_i$ and is of the form
 \eqref{eq:def:loss} with
 $$\phi(\theta) \eqdef \theta, \quad \R(\theta) \eqdef \frac{1}{2}
 \theta^T (X^T X + \upsilon \Id_q) \theta , \quad s_i(z) \eqdef X^T
 z.$$ Under the stated assumptions on $X$, the function $\param
 \mapsto - \pscal{s}{\phi(\param)} + R(\param)$ is defined on $
 \rset^q$ and for any $s \in \rset^q$, it possesses an unique minimum
 given by
 \[
 \map(s) \eqdef (\upsilon \Id_q + X^T X)^{-1} s \eqsp.
 \]
Define
 \begin{align*}
   \Pi_1 & \eqdef X^T (\Id_p + A^T A)^{-1} A^T \in \rset^{q \times y}
   \eqsp, \\ \Pi_2 & \eqdef X^T (\Id_p + A^T A)^{-1} X (\upsilon \Id_q
   + X^T X)^{-1} \in \rset^{q \times q} \eqsp.
   \end{align*}
   The a posteriori distribution $p_i(\cdot, \theta) \rmd \mu$ of the
   latent variable $Z_i$ given the observation $Y_i$ is a Gaussian
   distribution
 \[
 \mathcal{N}_p\left( (\Id_p + A^T A)^{-1} (A^T Y_i + X \theta), (\Id_p
 + A^T A)^{-1} \right),
 \]
 so that for all $i\in \{1, \ldots,n\}$,
 \begin{align*}
   \bars_i(\theta) & \eqdef X^T (\Id_p + A^T A)^{-1} (A^T Y_i+X\theta) \\
                   & =
                     \Pi_1 Y_i + X^T (\Id_p + A^T A)^{-1} X \theta \in \rset^q \eqsp, \\
   \bars_i \circ \map(s) &= \Pi_1 Y_i + \Pi_2 s \eqsp.
   \end{align*}
   Therefore, A\autoref{hyp:model}, A\autoref{hyp:bars},
   A\autoref{hyp:Tmap} and A\autoref{hyp:regV}-\ref{hyp:model:C1},
   \ref{hyp:model:F:C1} are satisfied. Since $\phi \circ \map(s) =
   \map(s)$ then $B(s) = (\upsilon \Id_q + X^T X)^{-1}$ for any $s \in
   \Sset$, and A\autoref{hyp:regV}-\ref{hyp:regV:C1} and
   A\autoref{hyp:regV:bis}-\ref{hyp:regV:C1:vmax} hold with
 \begin{align*}
   v_\min & \eqdef \frac{1}{\upsilon+\mathrm{max\_eig}(X^T X)} \eqsp,
   \\ v_\max
          & \eqdef \frac{1}{\upsilon+\mathrm{min\_eig}(X^T X)} \eqsp;
 \end{align*}
here, $\mathrm{max\_eig}$ and $\mathrm{min\_eig}$ denote resp. the
maximum and the minimum of the eigenvalues. $\bars_i \circ \map(s) =
\Pi_1 Y_i + \Pi_2 s$ thus showing that
A\autoref{hyp:regV:bis}-\ref{hyp:Tmap:smooth} holds with the same
constant $L_i =L$ for all $i$. Finally, $s \mapsto B^T(s) \left(\bars
\circ \map(s) -s \right)$ is globally Lipschitz with constant
\[
L_{\dot \lyap} \eqdef  \max \left| \mathrm{eig}\left( (\upsilon \Id_q + X^T X)^{-1} (\Pi_2 - \Id_q) \right) \right|;
\]
here $\mathrm{eig}$ denotes the eigenvalues.  This concludes the proof
of A\autoref{hyp:regV:bis}-\ref{hyp:regV:DerLip}.

 \subsection{The algorithms}
Given the current value $\hatS^k$, one iteration of {\tt EM}, {\tt
  Online EM}, {\tt FIEM} and {\tt opt-FIEM} are given by
\autoref{algo:toy:EM} and \autoref{algo:toy:FIEM}.

{\tt Online EM} requires $\kmax$ random draws from $[n]^\star$ per run
of length $\kmax$ iterations; {\tt FIEM} and {\tt opt-FIEM} require $2
\times \kmax$ draws. For a fair comparison of the algorithms along one
run, the same seed is used for all the algorithms when sampling the
examples from $[n]^\star$. Such a protocol allows to compare the
strategies by "freezing" the randomness due to the random choice of
the examples, and to really explain the different behaviors only by
the values of the design parameters (the step size, for example) or by
the updating scheme which is specific to each algorithm.

All the paths, whatever the algorithms, are started at the same value
$\hatS^0$.

\begin{algorithm}[htbp]
  \KwData{ $\hatS^k \in \Sset$, $\Pi_1$, $\Pi_2$ and $\barY_n$}
  \KwResult{$\hatS^{k+1}_{\EM}$}
    $\hatS^{k+1}_{\EM} = \Pi_1 \barY_n + \Pi_2 \hatS^k$
    \caption{Toy example: one iteration of {\tt EM}. \label{algo:toy:EM}}
\end{algorithm}

\begin{algorithm}[htbp]
  \KwData{ $\hatS^k \in \Sset$, $\Smem \in \Sset^n$, $\Sronde \in
    \Sset$; a step size $\pas_{k+1}\in\ocint{0,1}$ and a coefficient
    $\lambda_{k+1}$; the matrices $\Pi_1$, $\Pi_2$; the examples $Y_1,
    \cdots, Y_n$} \KwResult{$\hatS^{k+1}_{\FIEM}$} Sample
  independently $I_{k+1}$ and $J_{k+1}$ uniformly from $[n]^\star$ \; Store
  $s = \Smem_{I_{k+1}}$ \; Update $\Smem_{I_{k+1}} = \Pi_1
  Y_{I_{k+1}}+ \Pi_2 \hatS^k$ \; Update $\Sronde =
  \Sronde+n^{-1}(\Smem_{I_{k+1}} -s)$ \; Update $\hatS^{k+1}_{\FIEM} =
  \hatS^k + \gamma_{k+1} \left( \Pi_1 Y_{J_{k+1}} + \Pi_2 \hatS^k -
  \hatS^k + \lambda_{k+1} \left\{ \Sronde - \Smem_{J_{k+1}}
  \right\}\right)$
    \caption{Toy example: one iteration of {\tt Online EM} ($\lambda_{k+1}
      =0$), {\tt FIEM} ($\lambda_{k+1} =1$) and {\tt
        opt-FIEM}. \label{algo:toy:FIEM}}
\end{algorithm}

 \subsection{Numerical analysis}
 We choose $Y_i \in \rset^{15}$, $Z_i \in \rset^{10}$ and
 $\theta_\true \in \rset^{20}$. The entries of the matrix $A$
 (resp. $X$) are obtained as a stationary Gaussian auto-regressive
 process: the first column is sampled from $\sqrt{1-\rho^2} \,
 \mathcal{N}_{15}(0; \Id)$ (resp. from $\sqrt{1-\tilde \rho^2} \,
 \mathcal{N}_{10}(0; \Id)$) with $\rho =0.8$ (resp. $\tilde \rho =
 0.9$). $\theta_\true$ is sparse with $40 \%$ of the components set to
 zero; and the other ones are sampled uniformly from $\ccint{-5,5}$.

 The regularization parameter $\upsilon$ is set to $0.1$.

 \paragraph{{\tt FIEM}: the step sizes and the nonasymptotic controls.}
The first analysis is to compare the nonasymptotic bounds and the
constant step sizes provided by \autoref{coro:optimal:sampling},
\autoref{coro:optimal:sampling:Ketn} and \cite[Theorem
  2]{karimi:etal:2019} (see also \eqref{eq:stepsize:karimi} and
\eqref{eq:control:karimi}): the bounds are of the form
\[
\frac{n^\pa}{\kmax^\pb} \ \mathcal{B}  \ \Delta \lyap \eqsp;
\]
the numerical results below correspond to $\Delta \lyap =1$ and are
obtained with a data set of size $n=1e6$.  \autoref{fig:constant:C1}
shows the value of the constant $C$ solving \eqref{eq:bounds:def:C}
when $\lambda$ is successively set to $\{0.25, 0.5, 0.75 \}$ and as a
function of $\mu \in \ooint{0.01, 0.9}$. \autoref{fig:constant:C2}
shows the same analysis for the constant $C$ solving
\eqref{eq:bounds:def:C:Ketn}.  \autoref{fig:constant:B1} and
\autoref{fig:constant:B2} display the quantity $\mathcal{B}$ as a
function of $\mu$ and when the pair $(\lambda, C)$ is fixed to
$\lambda \in \{0.25, 0.5, 0.75 \}$ and $C$ solves resp.
\eqref{eq:bounds:def:C} and \eqref{eq:bounds:def:C:Ketn}. The role of
$\lambda$ looks quite negligible; the bound $\mathcal{B}$ seems to be
optimal with $\mu \approx 0.25$. Note that the constants $C$ and
$\mathcal{B}$ given by \autoref{coro:optimal:sampling:Ketn} depend on
$\kmax$: the results displayed here correspond to $\kmax=n$ but we
observed that the plots are the same with $\kmax = 1e2 \, n$ and
$\kmax = 1e3 \, n$ (remember that $n=1e6$).

\autoref{fig:stepsize} displays the step sizes as a function of $\mu
\in \ooint{0.01, 0.9}$, when $\lambda=1/2$ and for different
strategies of $\kmax$: $\kmax \in \{n, 1e2 \, n, 1e3 \, n\}$.
\autoref{fig:control} displays the quantity $n^\pa \kmax^{-\pb}
\mathcal{B}$. {\tt Case 1} (resp. {\tt Case 2}) corresponds to the
definition given in \autoref{coro:optimal:sampling}
(resp. \autoref{coro:optimal:sampling:Ketn}). For {\tt Case 1} and
{\tt Karimi et al}, $(\pa,\pb) =(2/3, 1)$ and for {\tt Case 2}, $(\pa,
\pb) = (1/3, 2/3)$. The first conclusion is that our results improve \cite{karimi:etal:2019}: we provide a larger step size
(improved by a factor up to $55$, with the strategy {\tt Case 1}, $\mu
=0.25$, $\lambda=0.5$) and a tighter bound (reduced by a factor up to
$235$, with the strategy {\tt Case 1}, $\mu =0.25$,
$\lambda=0.5$). The second conclusion is about the comparison of
\autoref{coro:optimal:sampling} and
\autoref{coro:optimal:sampling:Ketn}: as already commented (see
\autoref{sec:FIEM:errorrate:case2}), the first strategy is preferable
when the tolerance level $\epsilon$ is small (w.r.t. $n^{-1/3}$).
\begin{figure}[h]
  \includegraphics[width=\columnwidth]{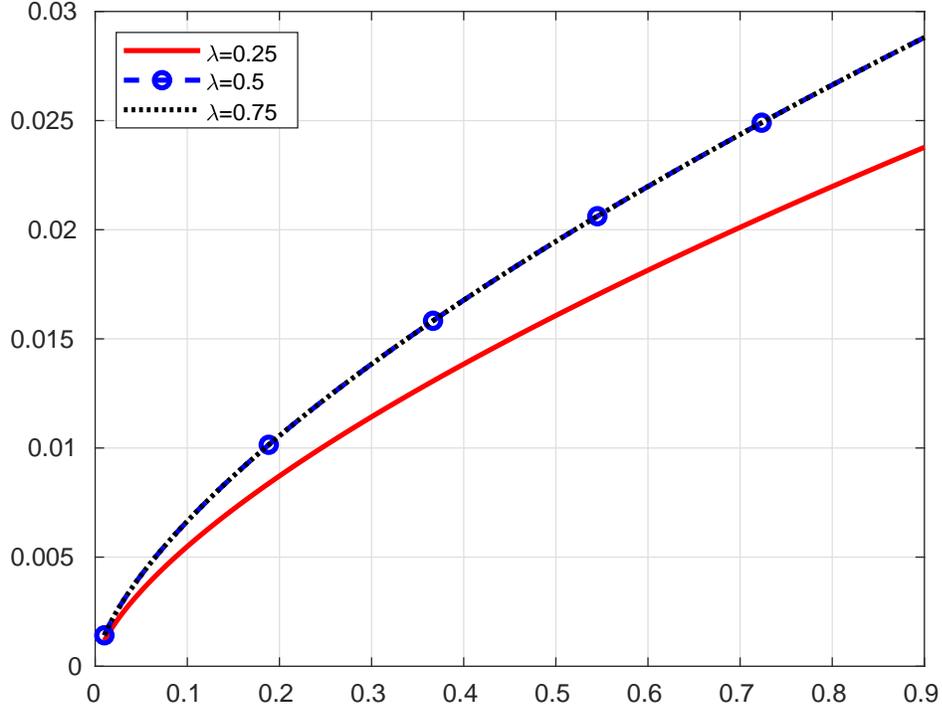}
\caption{For $\lambda \in \{0.25, 0.5, 0.75\}$ and
  $\mu \in \ooint{0.01, 0.9}$, evolution of the constant $C$ solving
  \eqref{eq:bounds:def:C}}
\label{fig:constant:C1}
\end{figure}
\begin{figure}[h]
  \includegraphics[width=\columnwidth]{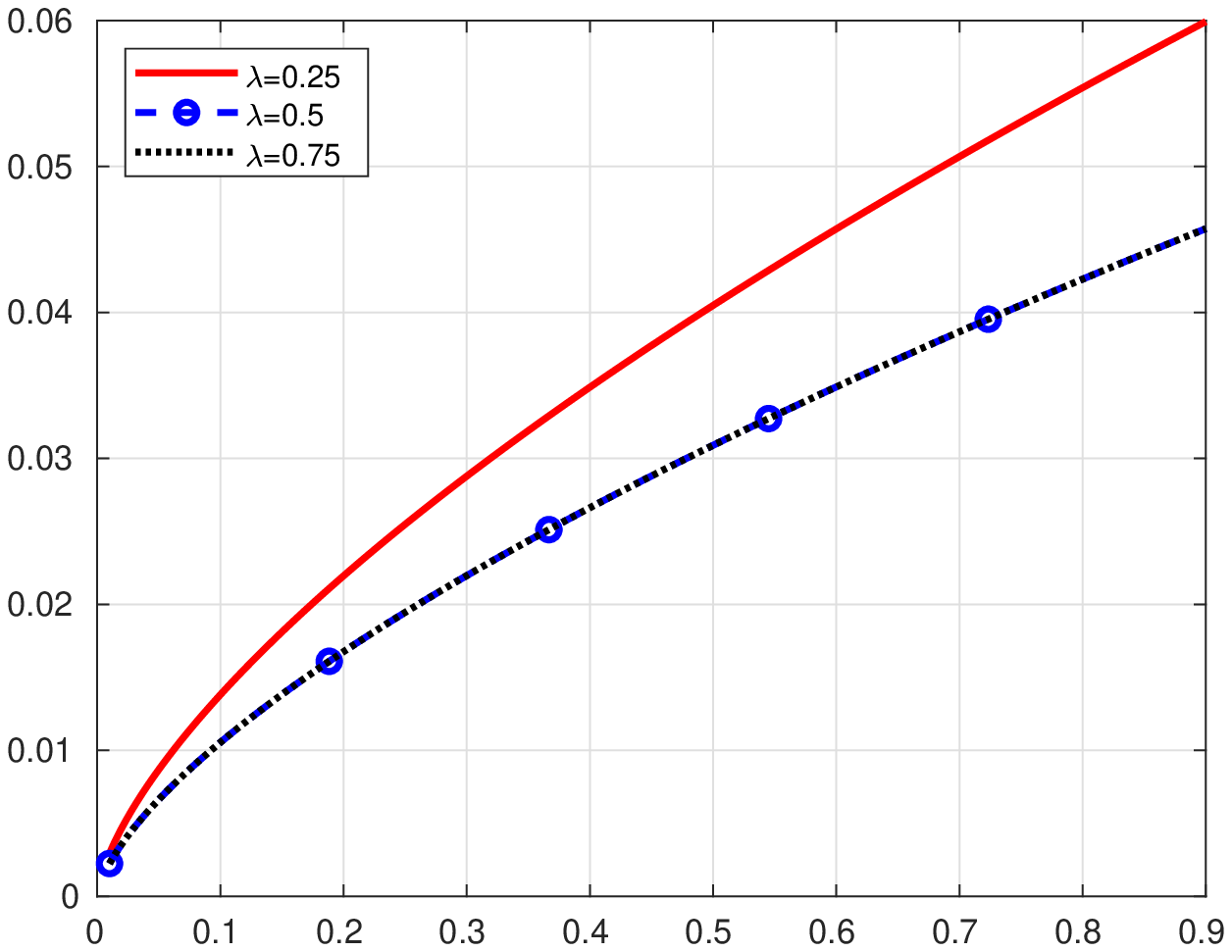}
\caption{For $\lambda \in \{0.25, 0.5, 0.75\}$ and $\mu \in
  \ooint{0.01, 0.9}$, evolution of the constant $C$ solving
  \eqref{eq:bounds:def:C:Ketn}}
\label{fig:constant:C2}
\end{figure}
\begin{figure}[h]
  \includegraphics[width=\columnwidth]{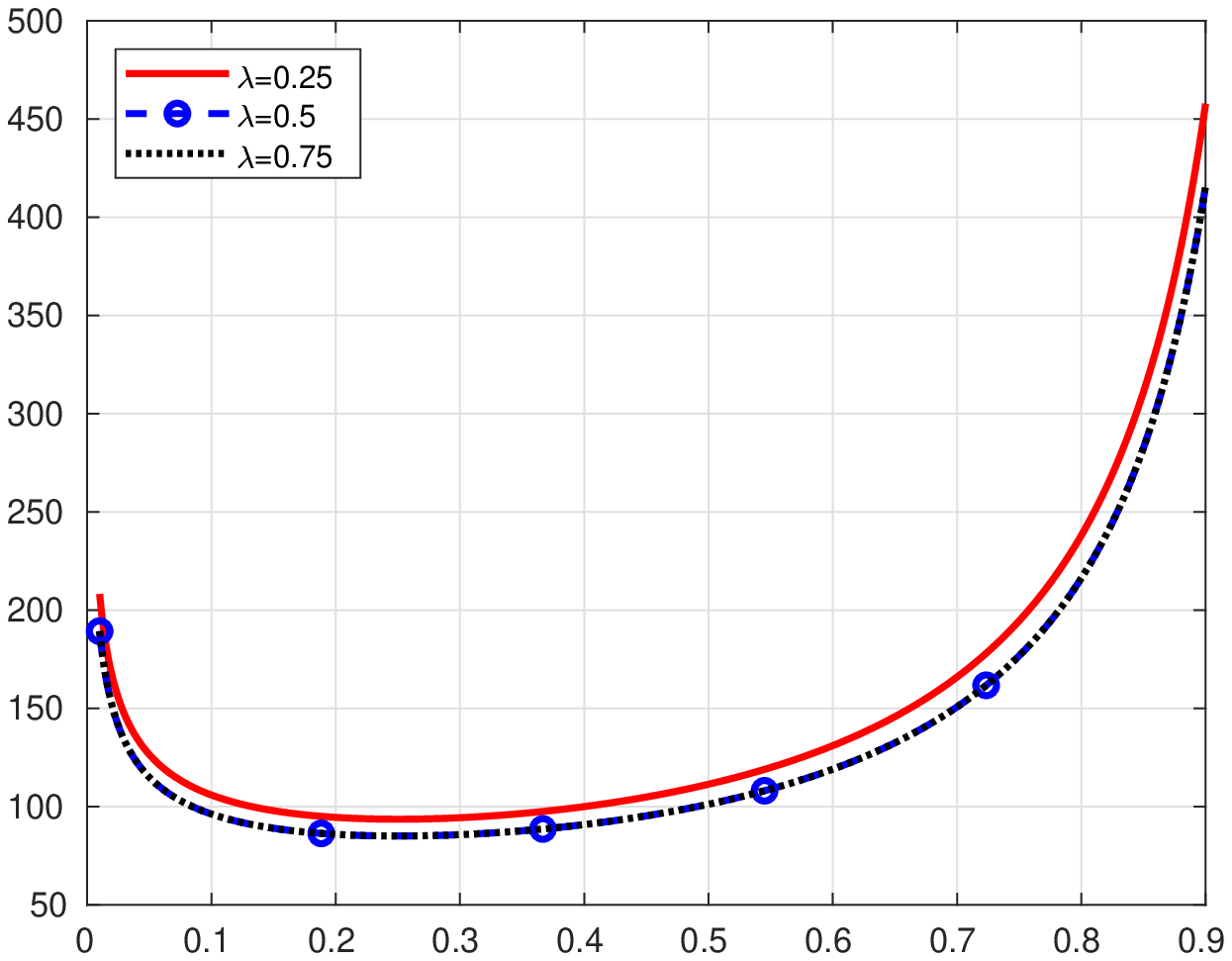}
\caption{For $\lambda \in \{0.25, 0.5, 0.75\}$ and $\mu \in
  \ooint{0.01, 0.9}$, evolution of the quantity $\mathcal{B}$ given by \autoref{coro:optimal:sampling}}
\label{fig:constant:B1}
\end{figure}
\begin{figure}[h]
  \includegraphics[width=\columnwidth]{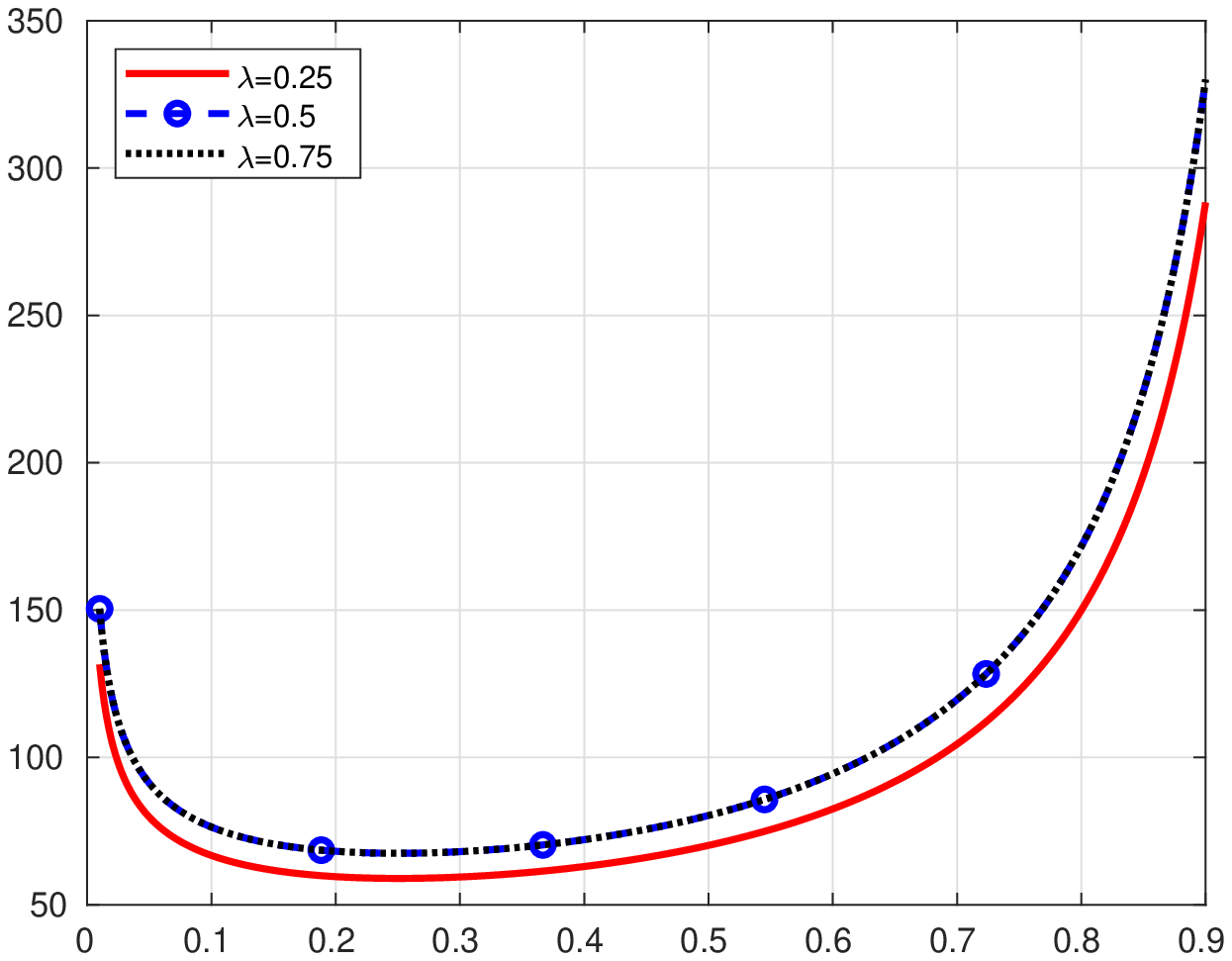}
\caption{For $\lambda \in \{0.25, 0.5, 0.75\}$ and $\mu \in
  \ooint{0.01, 0.9}$, evolution of the quantity $\mathcal{B}$ given by \autoref{coro:optimal:sampling:Ketn}}
\label{fig:constant:B2}
\end{figure}
\begin{figure}[h]
  \includegraphics[width=\columnwidth]{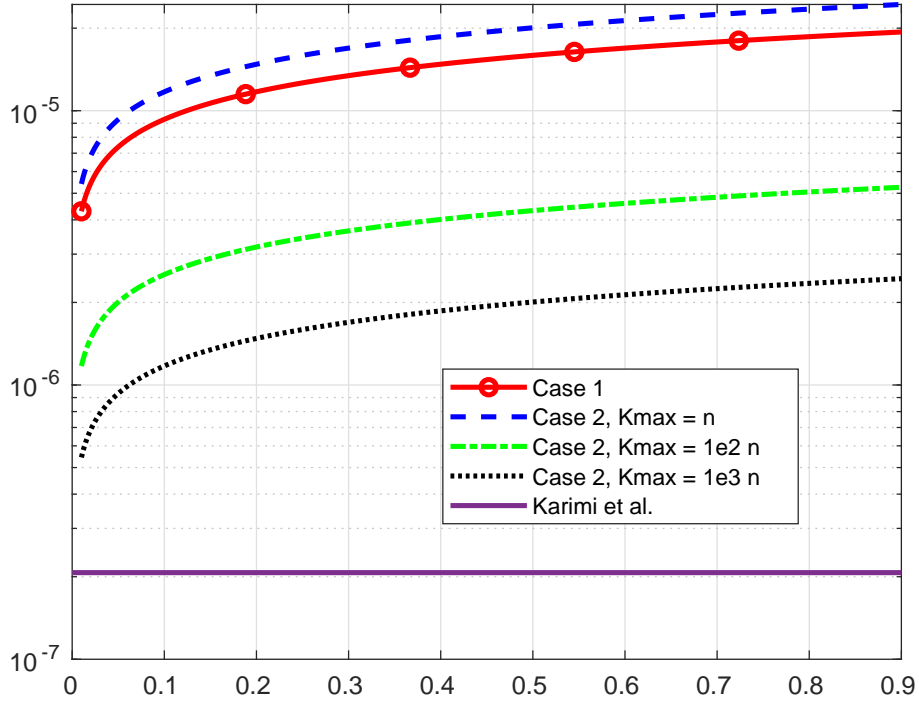}
  \caption{Value of the constant step size given by Karimi et al.,
    \autoref{coro:optimal:sampling} ({\tt Case 1}) and
    \autoref{coro:optimal:sampling:Ketn} ({\tt Case 2}). The step size
    is shown as a function of $\mu \in \ooint{0.01, 0.9}$. In
      {\tt Case 2}, different strategies for $\kmax$ are considered.}
\label{fig:stepsize}
\end{figure}
\begin{figure}[h]
  \includegraphics[width=\columnwidth]{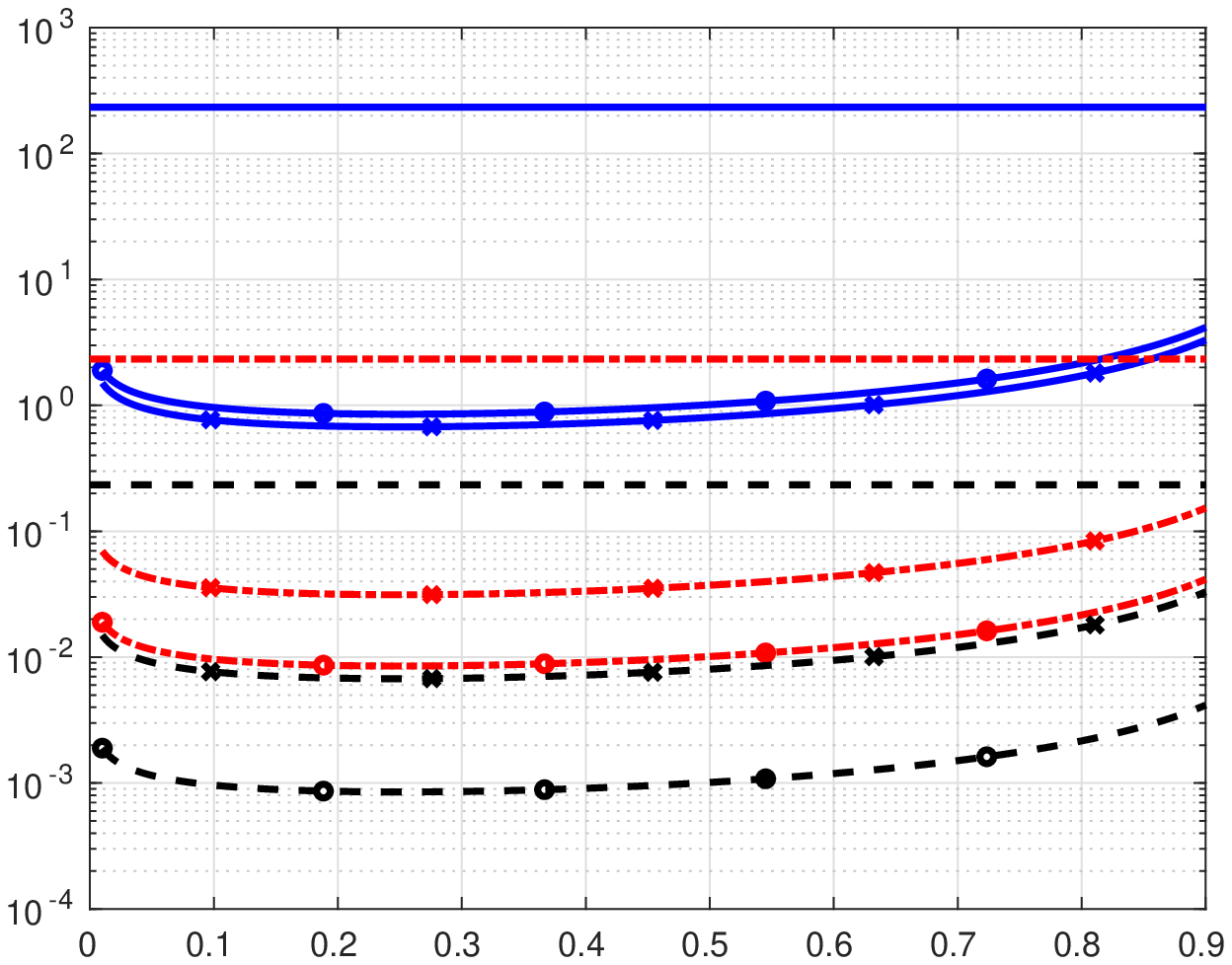}
\caption{Value of the control $n^\pa \kmax^{-\pb} \mathcal{B}$ given
  by \autoref{coro:optimal:sampling} ({\tt Case 1}, with a circle),
  \autoref{coro:optimal:sampling:Ketn} ({\tt Case 2}, with a cross)
  and Karimi et al. (no markers). The control is displayed as a
  function of $\mu \in \ooint{0.01, 0.9}$ and for different values of
  $\kmax$: $\kmax =n$ (solid line), $\kmax = 1e2 \, n$ (dash-dot line)
  and $\kmax = 1e3 \, n$ (dashed line).}
\label{fig:control}
\end{figure}

\paragraph{Comparison of {\tt Online EM}, {\tt FIEM} and {\tt opt-FIEM}.}
The algorithms are run with the same constant step size given by
\eqref{eq:C:uniform} when $C$ solves \eqref{eq:bounds:def:C} with $\mu
= 0.25$ and $\lambda = 0.5$. The size of the data set is $n=1e3$ and
the maximal number of iterations is $\kmax = 20 \, n$. Since the non
asymptotic bounds are essentially based on the control of
$\pas_{k+1}^{-2} \, \PE\left[ \| \hatS^{k+1} - \hatS^k\|^2 \right]$
(see the sketch of proof of \autoref{theo:FIEM:NonUnifStop} in
\autoref{sec:FIEM:complexity}), we first compare the algorithms
through this criterion: the expectation is approximated by a Monte
Carlo sum over $1e3$ independent runs.  The second criterion for
comparison is a distance of the iterates to the unique solution
$\param_\star$ via the expectation $\PE\left[ \|\param^k -
  \param_\star\| \right]$ and the standard deviation
$\mathrm{std}\left( \|\param^k - \param_\star\| \right)$ again
approximated by a Monte Carlo sum over the same $1e3$ independent
runs.

\autoref{fig:lambdaopt} displays the evolution of $k \mapsto
\lambda_{k+1}^\star$, the optimal coefficient given by
\eqref{eq:optimal:lambda}; in this toy example, it is computed
explicitly. As intuited in \autoref{sec:beyondFIEM}, we obtain
$\lambda_{k+1}^\star \approx 1$ for large iteration indexes $k$; {\tt
  FIEM} and {\tt opt-FIEM} have the same (or almost the same) update
scheme $\hatS^k \to \hatS^{k+1}$.
\begin{figure}[h]
  \includegraphics[width=\columnwidth]{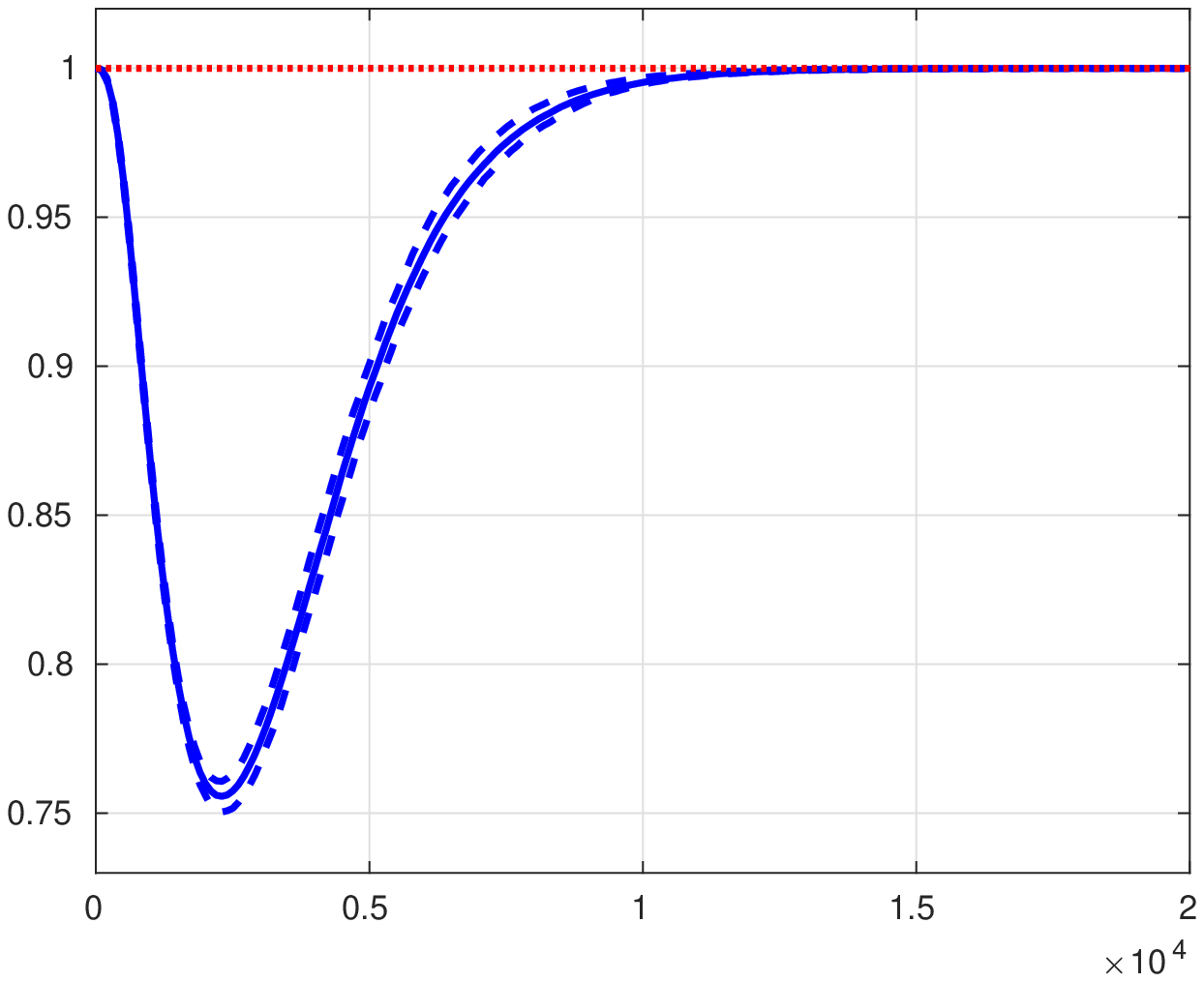}
\caption{The coefficient $\lambda_k^\star$ (see
  \eqref{eq:optimal:lambda}) as a function of the number of iterations
  $k$; it is a random variable, and the solid line is the mean value
  (the dashed lines are resp. the quantiles $0.25$ and $0.75$) over
  $1e3$ independent paths.}
\label{fig:lambdaopt}
\end{figure}
The ratio of the expectations $\PE\left[
  \|\param^k_{\mathrm{opt\_FIEM}}- \param_\star\| \right] / \PE\left[
  \|\param^k_{\mathrm{Alg}} - \param_\star\| \right]$ and of the
standard deviations $\mathrm{std}(\|\param^k_{\mathrm{opt\_FIEM}}-
\param_\star\|) / \mathrm{std}(\|\param^k_{\mathrm{Alg}} -
\param_\star\|)$ are displayed on \autoref{fig:ratio:errortheta} when
$\mathrm{Alg}$ is {\tt FIEM} and {\tt Online EM}.  They are shown as a
function of $k$ for $k \in \{1e2, 5e2, 1e3, 1.5e3, \ldots, 6e3, 7e3,
\ldots, 20e3\}$.  When $\mathrm{Alg}$ is {\tt FIEM} and the number of
iterations $k$ is large, we observe that both the ratio of the mean
values and the ratio of the standard deviations tend to one: this is
an echo to the previous comment $\lambda_{k}^\star \approx 1$. Note
also that when $k$ is large, {\tt Online EM} has a really poor
behavior when compared to {\tt opt-FIEM} (and therefore also to {\tt
  FIEM}). For the first iterations of the algorithm, we observe first
that {\tt opt-FIEM} and {\tt Online-EM} escape more rapidly from the
(possibly bad) initial value than {\tt FIEM}; {\tt opt-FIEM} surpasses
{\tt FIEM} by reducing the variance up to $22 \%$. Second, the plot
also shows that {\tt Online EM} may reduce the variability of {\tt
  opt-FIEM} up to $18 \%$, but {\tt opt-FIEM} provides a drastic
variability reduction in the first iterations. Since we advocate to
stop {\tt FIEM} at a random time $K$ sampled in the range $\{0,
\ldots, \kmax-1\}$, {\tt opt-FIEM} gives insights on how to improve
the behavior of incremental EM algorithms in the first iterations.
\begin{figure}[h]
  \includegraphics[width=\columnwidth]{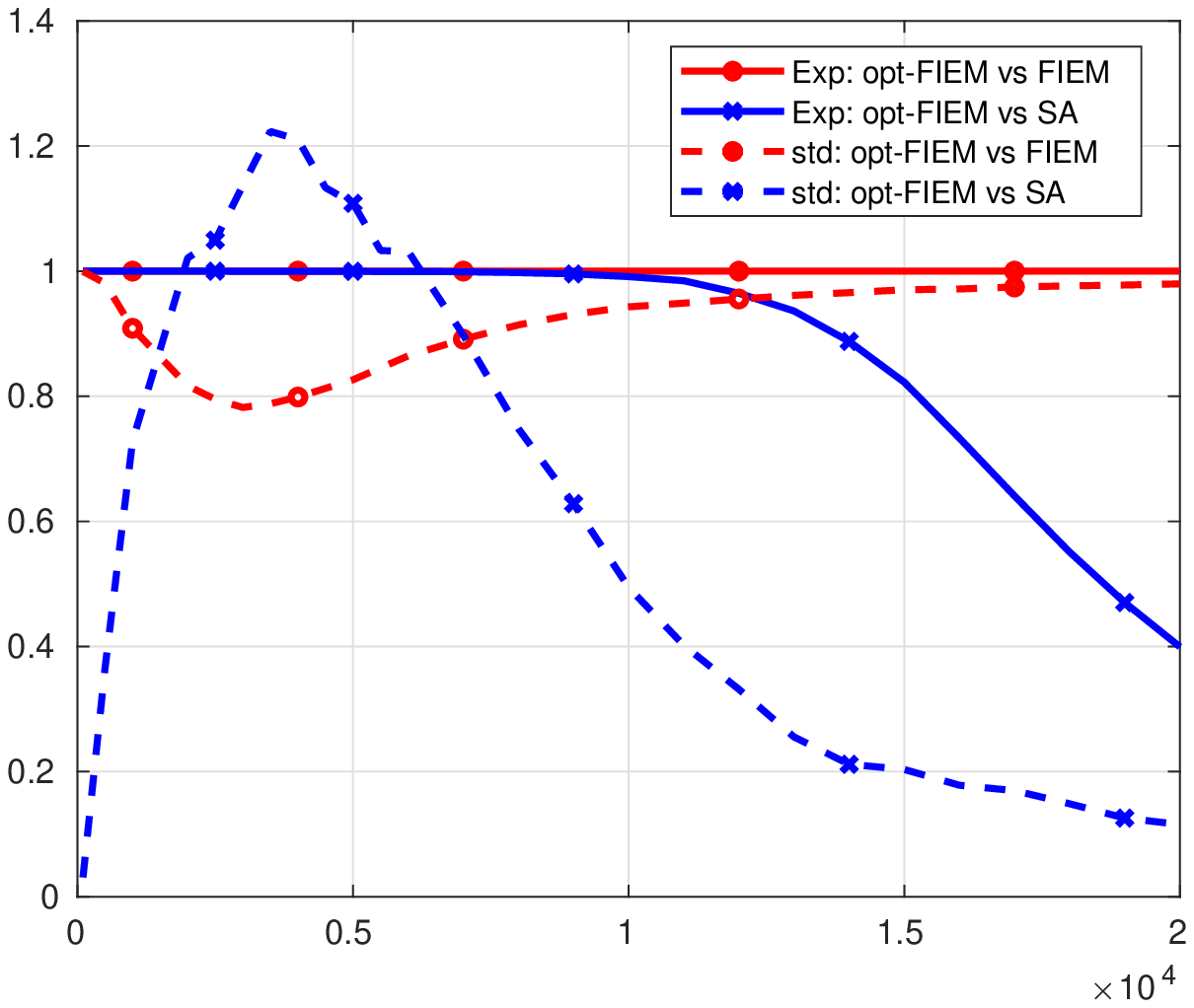}
\caption{For $k \in \{1e2, 5e2, 1e3, 1.5e3, \ldots, 6e3, 7e3, \ldots,
  20e3\}$, ratio of the expectations (Exp) $\PE\left[
    \|\param^k_{\mathrm{opt-FIEM}}- \param_\star\| \right] / \PE\left[
    \|\param^k_{\mathrm{Alg}} - \param_\star\| \right]$ when
  $\mathrm{Alg}$ is FIEM (solid line with circle) and then {\tt Online
    EM} (solid line with cross); and the standard deviations (std)
  $\mathrm{std}(\|\param^k_{\mathrm{opt_FIEM}}- \param_\star\|) /
  \mathrm{std}(\|\param^k_{\mathrm{Alg}} - \param_\star\|)$ when
  $\mathrm{Alg}$ is {\tt FIEM} (dashed line with circle) and then {\tt
    Online EM} (dashed line with cross). The expectations and standard
  deviations are approximated by a Monte Carlo sum over $1e3$
  independent runs.}
\label{fig:ratio:errortheta}
\end{figure}
\autoref{fig:fieldH} shows $k \mapsto \pas_{k+1}^{-2}
\PE\left[\|\hatS^{k+1} - \hatS^k\|^2 \right]$ for the three algorithms
when $k \in \ccint{1.5e3, 5e3}$. The plot illustrates again that {\tt
  opt-FIEM} improves {\tt FIEM} during these first iterations; and
improves drastically {\tt Online EM}.
\begin{figure}[h]
  \includegraphics[width=\columnwidth]{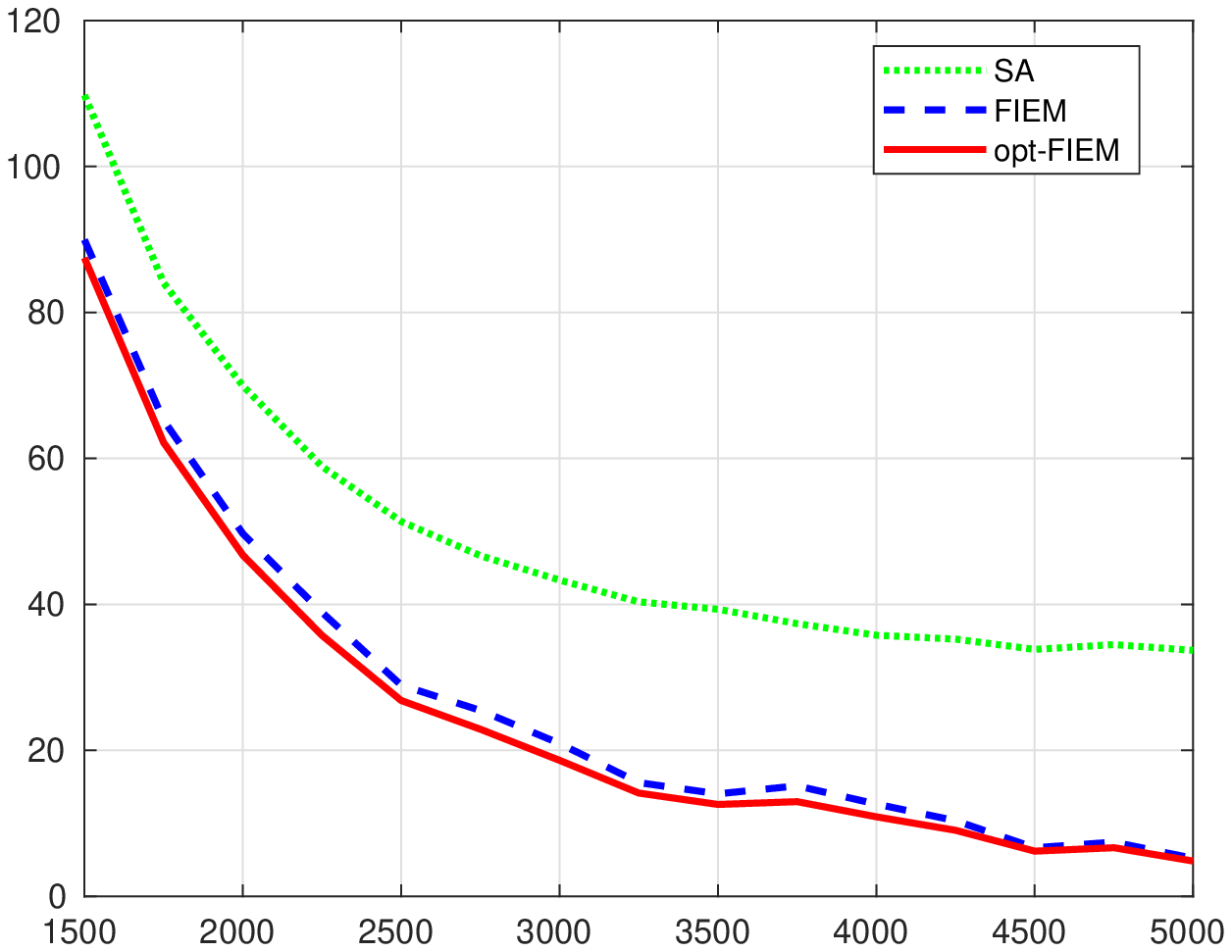}
\caption{Monte Carlo approximation (over $1e3$ independent runs) of $k
  \mapsto \pas_{k+1}^{-2} \PE\left[\|\hatS^{k+1} - \hatS^k\|^2
    \right]$ for {\tt Online EM}, {\tt FIEM} and {\tt opt-FIEM}.}
\label{fig:fieldH}
\end{figure}

%%%%%%%%%%%%%%%%%%%%%%%%%%%%%%%%%%%%%%%%%%%%
%%%%%%%%%%%%%%%%%%%%%%%%%%%%%%%%%%%%%%%%%%%%%
\section{Mixture of Gaussian distributions}
\label{sec:mixtureGaussian}
%%%%%%%%%%%%%%%%%%%%%%%%%%%%%%%%%%%%%%%%%%%%%
%%%%%%%%%%%%%%%%%%%%%%%%%%%%%%%%%%%%%%%%%%%%%
{\bf Notations.} For two $p \times p$ matrices $A,B$, $\pscal{A}{B}$
is the trace of $B^T A$: $\pscal{A}{B} \eqdef \mathrm{Tr}(B^T
A)$. $\Id_p$ stands for the $p \times p$ identity matrix. $\otimes$
stands for the Kronecker product. $\mathcal{M}_p^+$ denotes the set of
the invertible $p \times p$ covariance matrices. $\mathrm{det}(A)$ is
the determinant of the matrix $A$.

\medskip

In this section~\footnote{see footnote~\ref{footnote:github}}, {\tt
  FIEM} is applied to solve Maximum Likelihood inference in a mixture
of $g$ Gaussian distributions centered at $\mu_\ell$ and sharing the
same covariance matrix $\Sigma$ (see \cite{book:EM:mixture:2019} for a
recent review on mixture models): given $n$ $\rset^p$-valued
observations $y_1, \ldots, y_n$, find a point
$\hat{\param}_n^{\mathrm{ML}} \in \Param$ satisfying
$F(\hat{\param}_n^{\mathrm{ML}}) \eqdef
\R(\hat{\param}_n^{\mathrm{ML}})+ n^{-1} \sum_{i=1}^n
\loss{i}(\hat{\param}_n^{\mathrm{ML}}) \leq F(\param)$ for any $\param
\in \Param$ where $\theta \eqdef(\alpha_1, \ldots, \alpha_g, \mu_1,
\ldots, \mu_g, \Sigma)$,
\[
\Param \eqdef \left\{ \alpha_\ell \geq 0, \sum_{\ell=1}^g \alpha_\ell
=1 \right \} \times \rset^{pg} \times (\mathcal{M}^{+}_p) \subseteq
\rset^{g+pg+(p \times p)} \eqsp.
\]
 In addition,
\[
\R(\param) + \frac{1}{n} \sum_{i=1}^n \loss{i}(\param) = - \frac{1}{n} \sum_{i=1}^n \log
\sum_{\ell=1}^g \alpha_\ell \ \mathcal{N}_p(\mu_\ell, \Sigma)[y_i]
\eqsp,
\]
where we set (the term $p \log(2 \pi)/2$ is omitted)
\begin{align*}
R(\param) & \eqdef \frac{1}{2} \log \mathrm{det}(\Sigma) +
\frac{1}{2n} \sum_{i=1}^n y_i^T \Sigma^{-1} y_i = \frac{1}{2}
\left(\pscal{\Sigma^{-1}}{\frac{1}{n} \sum_{i=1}^n y_i y_i^T} - \log
\mathrm{det}(\Sigma^{-1}) \right) \eqsp.
\end{align*}
In this example,
$\loss{i}(\param) = - \log \sum_{z=1}^g
\exp(\pscal{\s_i(z)}{\phi(\param)})$ with
 \[
\s_i(z) = \A_{y_i} \ \left[ \begin{matrix} \1_{z=1} \\ \cdots
    \\ \1_{z=g} \end{matrix} \right] \in \rset^{g + pg} \eqsp, \qquad
\A_y \eqdef \left[ \begin{matrix} \Id_g \\ \Id_g \otimes
    y \end{matrix} \right] \eqsp.
 \]
  We use the MNIST dataset~\footnote{available at
    http://yann.lecun.com/exdb/mnist/}. The data are pre-processed as
  in \cite{nguyen:etal:2020}: the training set contains $n=6e4$ images
  of size $28 \times 28$; among these $784$ pixels, $67$ are non
  informative since they are constant over all the pictures so they
  are removed yielding to $n$ observations of length $717$; each
  feature is centered and standardized (among the $n$ observations)
  and a PCA of the associated $717 \times 717$ covariance matrix is
  applied in order to summarize the features by the first $p = 20$
  principal components. In the numerical applications, we fix $g=12$
  components in the mixture.

The maximization step $ s \mapsto \map(s)$ is given by
\begin{align*}
\widehat{\alpha_\ell} &\eqdef \frac{s_\ell}{\sum_{u=1}^g s_u} \eqsp,
\ell \in [g]^\star \eqsp, \\ \widehat{\mu_\ell} & \eqdef \frac{s_{g
    +(\ell-1)p+1:g+\ell p}}{s_\ell}\eqsp, \ell \in [g]^\star\eqsp,
\\ \widehat{\Sigma} & \eqdef \frac{1}{n} \sum_{i=1}^n y_i y_i^T -
\sum_{\ell=1}^g s_\ell \, \widehat{\mu_\ell} \, \widehat{\mu_\ell}^T
\eqsp,
\end{align*}
where $s \eqdef (s_1, \ldots, s_{g + pg})$; see
\autoref{secAPP:GMM:mapT} in the supplementary material. Since we want
$\map(s) \in \Param$, $\map$ is defined at least on $\mathcal{S}
\subset \rset^{g+ pg}$:
\begin{multline*}
  \mathcal{S} \eqdef \left\{n^{-1} \sum_{i=1}^n \A_{y_i} \rho_{i}:
  \right. \\ \left.  \rho_i=(\rho_{i,1}, \ldots, \rho_{i,g}), \rho_{i,
    \ell} \geq 0, \sum_{\ell=1}^g \rho_{i,\ell}=1 \right\} \eqsp;
\end{multline*}
see \autoref{secApp:GMM:domainS} in the supplementary material.

This model is used to go beyond the theoretical framework adopted in
this paper.  The first extension concerns the domain of $\map$:
A\autoref{hyp:Tmap} assumes that $\map$ is defined on $\rset^q$ (here,
$q = g + pg$) while the above description shows that it is not always
true.  This gap between theory and application is classical for
mixture of Gaussian distributions (while $\rho_{\ell,i}$ may be a
signed quantity or while we may have $\sum_{\ell=1}^g \rho_{i,\ell}
\neq 1$ for the considered algorithms (see \autoref{sec:supp:GMM:EM}
to \autoref{sec:supp:GMM:FIEM} in the supplementary material for a
detailed derivation), numerically we always obtained quantities
$\hatS^k$ which were in $\mathcal{S}$.

The second extension concerns the use of mini-batches at each
iteration of incremental EM algorithms: instead of sampling one
example per iteration (see e.g. \autoref{line:FIEM:debutaux},
\autoref{line:FIEM:single} in \autoref{algo:FIEM},
\autoref{line:SA:single} in \autoref{algo:SA} and
\autoref{line:iEM:debutaux} in \autoref{algo:iEM}), a mini-batch of
size $\lbatch$ is used - sampled at random from the $n$ available
examples, possibly with replacement. In the supplementary material, we
provide in \autoref{secAPP:GMM:algos} a description of {\tt iEM}, {\tt
  Online EM} and {\tt FIEM} in the case $\lbatch>1$.

{\tt EM}, {\tt iEM}, {\tt Online EM} and {\tt FIEM} are compared when
used to solve the above Maximum Likelihood inference problem. All the
paths of these algorithms are started from the same point
$\param^0 \in \Param$ defined by the randomization scheme described in
\cite[section 4]{kwedlo:2015}; we then set
$\hatS^0 \eqdef n^{-1} \sum_{i=1}^n \bars_i(\param^0)$; the normalized
log-likelihood $-F(\param^0)$ is equal to $-58.3097$ (equivalently,
the unnormalized log-likelihood is $-3.4986 \, \mathrm{e}{+6}$). Note
that, as mentioned below, the evaluation of the log-likelihood does
not include the constant $+ p \log(2 \pi)/2$.

Each iteration of {\tt iEM}, {\tt Online EM} (resp. {\tt FIEM}) calls
a mini-batch of $\lbatch = 100$ examples (resp. $2$ mini-batches of
size $\lbatch =100$ examples each) sampled uniformly from $[n]^\star$
with replacement; for a fair comparison of the paths produced by these
algorithms, the same seed is used.

The paths are seen as cycles of {\em epochs}, an epoch being defined
as the processing of $n$ examples: for {\tt EM}, an epoch is one
iteration; for {\tt iEM} and {\tt Online EM}, an epoch is $n/\lbatch$
iterations; for {\tt FIEM}, an epoch is $n / (2\lbatch)$
iterations. Below, the paths are run until $100 \, n$ examples are
processed, which means $100$ iterations or epochs for {\tt EM}, and
$100 \, n/\lbatch$ iterations (or $100$ epochs) for both {\tt iEM} and
{\tt Online EM}. Instead of a pure {\tt FIEM} algorithm, we implement
{\tt h-FIEM}, an hybrid algorithm obtained by first running
$\mathrm{kswitch}$ epochs of {\tt Online EM} and then switching to
epochs of {\tt FIEM}: we choose $\mathrm{kswitch} =6$ so that {\tt
  h-FIEM} processes $100 \, n$ examples after $6 \, n/\lbatch$
iterations (or $6$ epochs) of {\tt Online EM} and $94 \, n/(2\lbatch)$
iterations (or $94$ epochs) of {\tt FIEM}. The use of {\tt h-FIEM} is
to explicitly illustrate the variance reduction of the {\tt FIEM}
iterations when compared to the {\tt Online EM} ones.

{\tt iEM} is run with the constant step size $\pas_{k+1} =1$; {\tt
  Online EM} and {\tt FIEM} are run with $\pas_{k+1} = 5 \,
\mathrm{e}{-3}$.

\autoref{fig:gauss:loglike:init} and \autoref{fig:gauss:loglike:end}
display the normalized log-likelihood along a path of {\tt EM}, {\tt
  iEM}, {\tt Online EM} and {h-FIEM}, resp. for the first epochs (from
$1$ to $25$) and by discarding the first ones (from $15$ to
$100$). The first conclusion is that the incremental methods forget
the initial value far more rapidly than {\tt EM}, which is the
consequence of the incremental processing of the observations which
allow many updates of the parameter $\param^k$ (or equivalently, of
the statistic $\hatS^k$) before the use of $n$ examples (which is
equivalent to the learning cost of one iteration of {\tt EM}). The
second conclusion is that the incremental EM-based methods perform a
better maximization of the normalized log-likelihood $-F$. Finally,
{\tt Online EM} and {\tt h-FIEM} are better than {\tt iEM}: the
log-likelihood converges resp. to $-1.9094 \, \mathrm{e}{+6}$,
$-1.9080 \, \mathrm{e}{+6}$ and $-1.9100 \, \mathrm{e}{+6}$ (the plot
displays the normalized log-likelihood); and it is clear that {\tt
  h-FIEM} reduces the variability of the {\tt Online EM} path. The
same conclusions are drawn from different runs; the supplementary
material provides a similar plot when the curves are the average over
$10$ independent paths; \autoref{tab:mixtgauss} reports the mean value
and the standard deviation of the log-likelihood over these $10$ runs.

\begin{figure}[h]
  \includegraphics[width=\columnwidth]{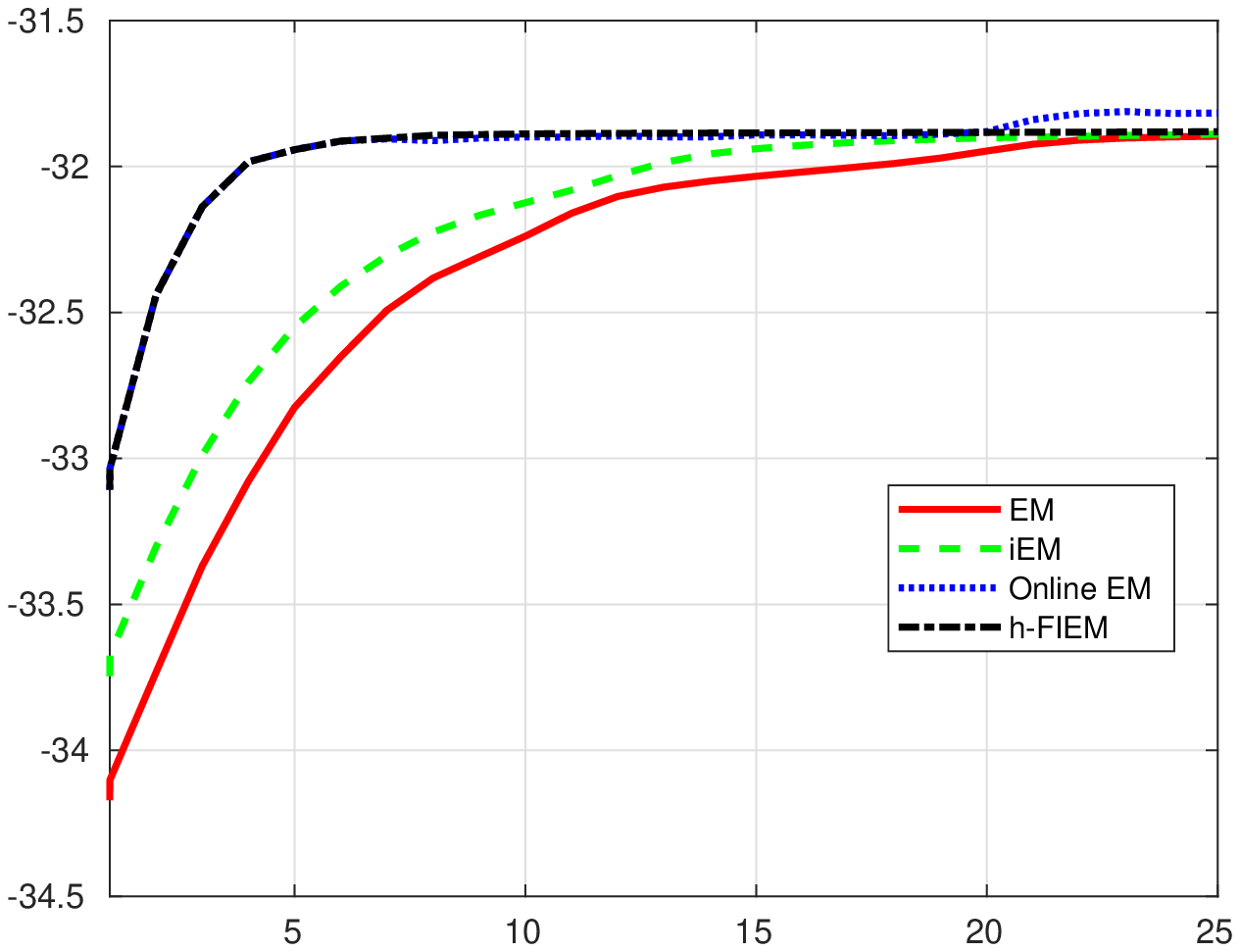}
  \caption{Evolution of the normalized log-likelihood along one path
    of length $100$ epochs: only the epochs $1$ to $25$ are
    displayed. All the paths start from the same value at time $t=0$,
    with a normalized log-likelihood equal to $-58.31$. }
\label{fig:gauss:loglike:init}
\end{figure}

\begin{figure}[h]
  \includegraphics[width=\columnwidth]{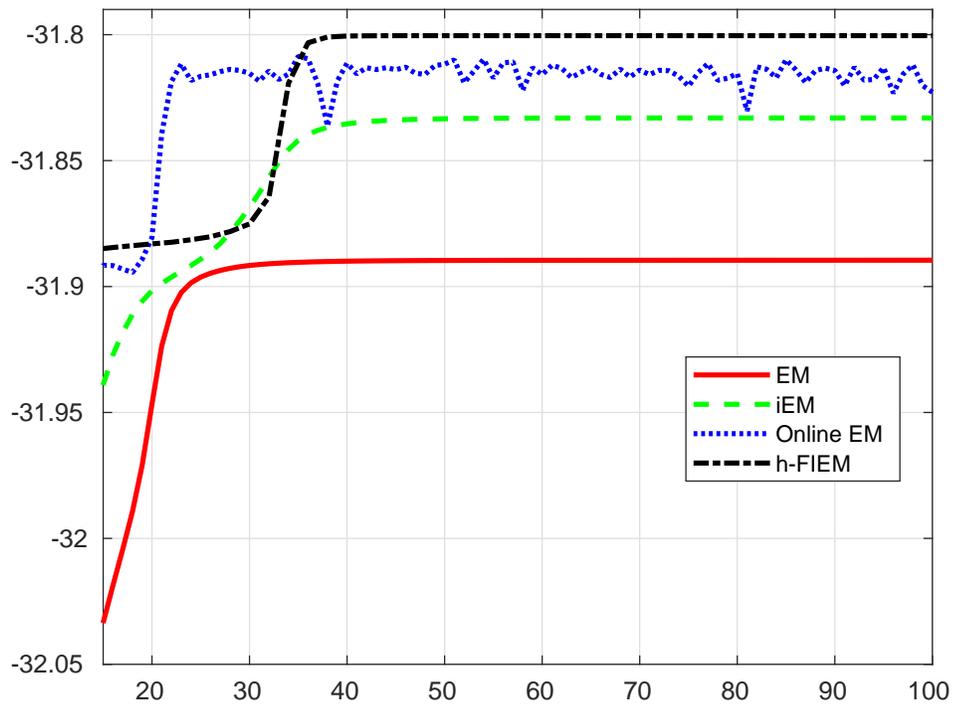}
\caption{Evolution of the normalized log-likelihood along one path of length
  $100$ epochs: the first $14$ epochs are discarded. All the paths
  start from the same value at time $t=0$, with a normalized log-likelihood equal
  to $-58.31$. }
\label{fig:gauss:loglike:end}
\end{figure}

\begin{table*}[h]
  \caption{Normalized log-likelihood along a {\tt EM}, {\tt iEM}, {\tt
      Online EM} and {\tt h-FIEM} path, at epoch $\# 1, 15, 25, 50,
    100$. The value is the average over $10$ independent runs (the
    standard deviation is in parenthesis). The log-likelihood is
    obtained by multiplying by $n=6\mathrm{e}{+4}$.\label{tab:mixtgauss}}
  \begin{tabular}{|c||c|c|c|c|c|}
    \hline
    & $\# 1$ & $\# 15$ & $\# 25$ & $\#50$ & $\# 100$ \\ \hline
    EM & $-3.4102 \mathrm{e}{+1}$ & $-3.2033 \mathrm{e}{+1}$  & $-3.1896 \mathrm{e}{+1}$  & $-3.1890 \mathrm{e}{+1}$  & $-3.1889 \mathrm{e}{+1}$  \\
   & - & - & - & - & - \\ \hline
    iEM & $-3.3672 \mathrm{e}{+1}$ & $-3.1982 \mathrm{e}{+1}$  & $-3.1869 \mathrm{e}{+1}$  & $-3.1843 \mathrm{e}{+1}$  & $-3.1827 \mathrm{e}{+1}$  \\
    & $(4.90\mathrm{e}{-3})$ & $(5.33\mathrm{e}{-2})$ & $(1.87\mathrm{e}{-2})$ & $(1.87\mathrm{e}{-2})$  & $(1.38\mathrm{e}{-2})$ \\ \hline
    Online EM & \textbf{-3.2999$\mathrm{e}{+1}$} & \textbf{-3.1872$\mathrm{e}{+1}$}  & \textbf{-3.1828$\mathrm{e}{+1}$}  & $-3.1823 \mathrm{e}{+1}$  & $-3.1823 \mathrm{e}{+1}$  \\
    &  $(2.67\mathrm{e}{-2})$  &  $(5.14\mathrm{e}{-2})$  &  $(4.67\mathrm{e}{-2})$  &  $(4.68\mathrm{e}{-2})$ &  $(4.50\mathrm{e}{-2})$ \\ \hline
    h-FIEM &  $-3.2999 \mathrm{e}{+1}$ & $-3.1900 \mathrm{e}{+1}$  & $-3.1853 \mathrm{e}{+1}$  & \textbf{-3.1806$\mathrm{e}{+1}$}  & \textbf{-3.1804$\mathrm{e}{+1}$}  \\
   & $(2.67\mathrm{e}{-2})$ & $(5.31\mathrm{e}{-2})$ & $(6.94\mathrm{e}{-2})$ & $(5.18\mathrm{e}{-2})$ & $(5.25\mathrm{e}{-2})$ \\ \hline
  \end{tabular}
\end{table*}

A fluctuation of $1\%$ (resp. $1$ \textperthousand) around the optimal
normalized log-likelihood corresponds to a lower bound of $-32.1184$
(resp. $-31.8322$): for {\tt EM} such an accuracy is reached after
$12$ iterations (resp. is never reached); for {\tt iEM}, it is reached
after $11$ epochs (resp. is never reached); for {\tt Online EM}, after
$4$ epochs (resp. $23$ epochs); for {\tt h-FIEM}, after $4$ epochs
(resp. $34$ epochs). An accuracy of $1$ \textpertenthousand\ is never
reached by {\tt Online EM} and is reached after $36$ epochs for {\tt
  h-FIEM}.

\autoref{fig:gauss:weight:path} shows the estimation of the $g=12$
weights $\alpha_\ell$ along a path of length $100$ epochs.  The
comparison of {\tt Online EM} (bottom left) and {\tt h-FIEM} (bottom
right) shows that {\tt h-FIEM} acts as a variability reduction
technique along the path, without slowing down the convergence
rate. \autoref{fig:gauss:weight:singlerun} displays the limiting value
of these paths i.e. the estimate of the weights $\alpha_1, \ldots,
\alpha_g$ defined as the value of the parameter at the end of $100$
epochs; the weights are sorted in descending order. {\tt Online EM}
and {\tt h-FIEM} provide similar estimates.

\begin{figure}[h]
  \includegraphics[width=\columnwidth]{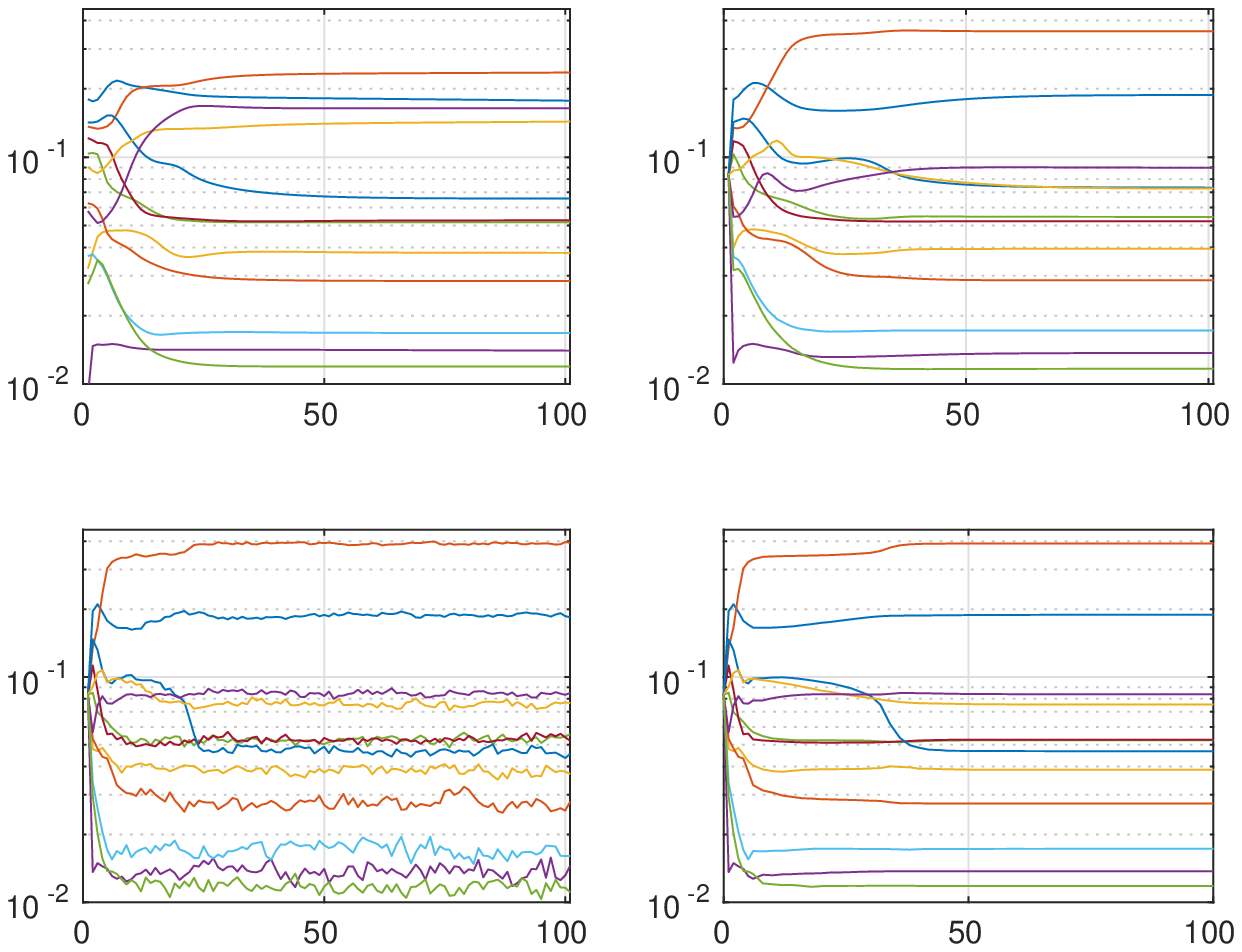}
  \caption{Evolution of the $g=12$ weights along one path of length
    $100$ epochs. All the paths start from the same value at time
    $t=0$. {\tt EM} (top left), {\tt iEM} (top right), {\tt Online EM}
    (bottom left) and {\tt h-FIEM} (bottom right).}
\label{fig:gauss:weight:path}
\end{figure}

\begin{figure}[h]
  \includegraphics[width=\columnwidth]{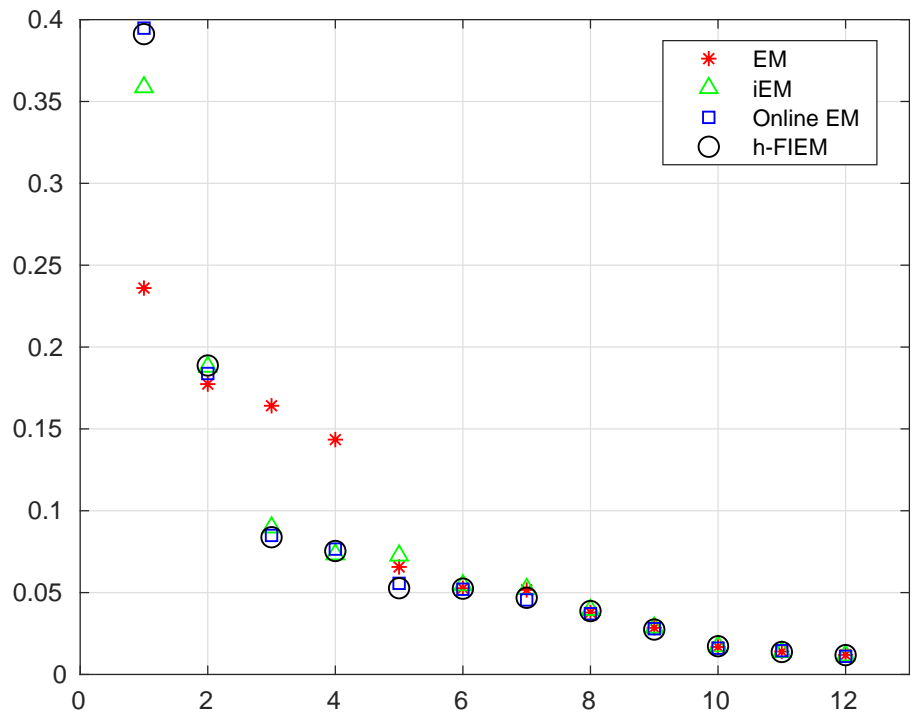}
\caption{Estimation of the $g=12$ weights of the mixture model.  The
  estimator is the value of the parameter obtained at the end of a
  single path of length $100$ epochs.}
\label{fig:gauss:weight:singlerun}
\end{figure}

%%%%%%%%%%%%%%%%
\section{Proof}
%%%%%%%%%%%%%%%
%%%%%%%%%%%%%%%
%%%%%%%%%%%%%%%%%%%%%%%%%%%%%%%%%%%%%%%
\subsection{Proof of \autoref{sec:algo}}
%%%%%%%%%%%%%%%%%%%%%%%%%%%%%%%%%%%%%%%
\subsubsection{Proof of \autoref{lem:nablaV}}
\label{sec:proof:nablaV}
\paragraph{(Proof of \ref{lem:nablaV:item1}).} The statements are trivial and we only prove the first claim: if $s^\star = \bars \circ \map (s^\star)$ then by applying $\map$ (under the uniqueness assumption  A\autoref{hyp:Tmap}), we have $\map(s^\star) = (\map \circ \bars) \circ \map(s^\star)$ and the proof follows.

\paragraph{(Proof of \ref{lem:nablaV:item2}).} For $\param \in
  \Param^v$, set $D\phi(\param) \eqdef \left( \dot\phi(\param)
  \right)^T$. By A\autoref{hyp:regV}-\ref{hyp:model:F:C1} and a chain
  rule,
  \[
  \dot \lyap(s) =\left(\dot \map(s) \right)^T \ \left\{ \dot
  \R(\map(s)) - D\phi(\map(s)) \ \bars \circ \map(s) \right\} \eqsp.
  \]
  Moreover, using A\autoref{hyp:Tmap} and
  A\autoref{hyp:regV}-\ref{hyp:model:C1}, the minimum $\map(s)$ is a
  critical point of $\param \mapsto \L(s,\param)$: we have for any $s
  \in \Sset$, $\dot\R(\map(s)) - D\phi(\map(s)) \, s = 0$.  Hence,
\[
  \dot \lyap(s) = - \left(\dot \map(s) \right)^T \ D\phi(\map(s))
  \ h(s) = - \left(B(s) \right)^T \ h(s) \eqsp.
\]
A\autoref{hyp:regV}-\ref{hyp:regV:C1} implies that $B^T =B$ and the zeros of
$h$ are the zeros of $\dot \lyap$.

%%%%%%%%%%%%%%%%%%%%%%%%%%%%%
\subsubsection{Auxiliary result}
\label{sec:auxiliary:results}
  \begin{lemma}\label{lem:expfam:reg}
    Assume that $\Param$ and $\phi(\Param)$ are open; and $\phi$ is
    continuously differentiable on $\Param$. Then for all $i \in \{1,
    \ldots,n\}$, $\loss{i}$ is continuously differentiable on
    $\Param$.

    If in addition A\autoref{hyp:model}, A\autoref{hyp:bars}, A\autoref{hyp:Tmap} and
    A\autoref{hyp:regV}-\ref{hyp:model:C1} hold, then $F$
    (resp. $\lyap \eqdef F \circ \map$) is continuously differentiable on $\Param$
    (resp. on $\Sset$) and for any $\param \in \Param$,
  \[
  \dot F(\param) = - \left( \dot{\phi}(\param)\right)^T \, \bars(\param)
+ \dot \R(\param) \eqsp.
\]
\end{lemma}
\begin{proof}
  A\autoref{hyp:model} and \cite[Proposition 3.8]{Sundberg:2019} (see
  also \cite[Theorem 2.2.]{Brown:1986}) imply that $L_i:\tau \mapsto
  \int_Z h_i(z) \, \exp\left({\pscal{s_i(z)}{\tau}} \right) \mu(\rmd
  z) $ is continuously differentiable on the interior of the set
 \[
 \{\tau \in \rset^q, \int_Z h_i(z) \, \exp\left({\pscal{s_i(z)}{\tau}}
 \right) \mu(\rmd z) < \infty \}
\]
and its derivative is
  \[
\int_\Zset s_i(z) \ h_i(z) \, \exp\left({\pscal{s_i(z)}{\tau}} \right)
\mu(\rmd z) \eqsp.
  \]
  This set contains $\phi(\Param)$ under A\autoref{hyp:model}.  The
  equality $\loss{i} = -\log( L_i \circ \phi)$ and the
  differentiability of composition of functions conclude the proof of
  the first item. The second one easily follows.
\end{proof}

%%%%%%%%%%%%%%%%%%%%%%%%%%%%%%%%%%%%%%%%%%%%%%%%%%%%
\subsection{Proofs of \autoref{sec:FIEM:complexity}}
\label{sec:proofs}
For any $k \geq 0$ and $i \in [n]^\star$, we define
$\hatS^{<k,i}$ such that
\[
\Sronde^{k} = \frac{1}{n}\sum_{i=1}^n \bars_i \circ \map(\hatS^{< k,
  i}) \eqsp;
\]
it means $\hatS^{<0,i} \eqdef \hatS^0$ for all $i \in [n]^\star$ and for
$k \geq 0$,
 \begin{equation}\label{eq:memory:lastupdate} \hatS^{< k+1,i}
  = \hatS^{\ell},
 \end{equation}
with
\begin{equation}
\left\{ \begin{array}{ll} \ell = k & \text{if $I_{k+1} = i$,} \\ \ell
  \in [k-1]^\star & \text{if $I_{k+1} \neq i, I_{k} \neq i, \ldots,
    I_{\ell+1} =i$,} \\ \ell =0 & \text{otherwise.}
    \end{array}
 \right.
 \end{equation}
Define the filtrations, for $k \geq 0$,
  \begin{align*}
\F_k & \eqdef \sigma(\hatS^0, I_1, J_1, \ldots, I_{k}, J_k),  \\
\F_{k+1/2} & \eqdef \sigma(\hatS^0, I_1, J_1, \ldots, I_{k}, J_k,
I_{k+1}) \eqsp;
\end{align*}
note that $\hatS^k \in \F_k$ and $\Smem_{k+1,\cdot} \in \F_{k+1/2}$.
Set
   \[
H_{k+1} \eqdef \bars_{J_{k+1}} \circ \map(\hatS^k) -\hatS^k +
\frac{1}{n} \sum_{i=1}^n \Smem_{k+1,i} - \Smem_{k+1, J_{k+1}} \eqsp.
\]
%%%%%%%%%%%%%%%%%%%%%%%
\subsubsection{Proof of \autoref{theo:FIEM:NonUnifStop}}
\label{sec:proof:theo}
By \autoref{lem:nablaV} and
A\autoref{hyp:regV:bis}-\ref{hyp:regV:DerLip}, $\dot \lyap$ is
$L_{\dot \lyap}$-Lipschitz on $\rset^q$, and we have
\begin{align*}
  \lyap(\hatS^{k+1})  - \lyap(\hatS^k) & \leq \pscal{\hatS^{k+1} -
    \hatS^k}{\dot \lyap(\hatS^k)} + \frac{L_{\dot \lyap}}{2} \|
  \hatS^{k+1}- \hatS^{k} \|^2 \\ & \leq \pas_{k+1}
  \pscal{H_{k+1}}{\dot \lyap(\hatS^k)} + \pas_{k+1}^2 \frac{ L_{\dot
      \lyap}}{2} \| H_{k+1} \|^2.
\end{align*}
 Taking the expectation yields, upon noting that $\hatS^k \in \F_k$,
\begin{align*}
  & \PE\left[\lyap(\hatS^{k+1}) \right] - \PE\left[\lyap(\hatS^k)
    \right] \\ & \leq \pas_{k+1} \PE\left[\pscal{\PE\left[ H_{k+1} \vert \F_k \right]}{\dot \lyap(\hatS^k)} \right] + \pas_{k+1}^2
  \frac{L_{\dot \lyap}}{2} \PE\left[\| H_{k+1} \|^2
    \right] \\ & \leq \pas_{k+1} \PE\left[\pscal{h(\hatS^k)}{\dot \lyap(\hatS^k)} \right] +
  \pas_{k+1}^2 \frac{L_{\dot \lyap}}{2} \PE\left[\| H_{k+1} \|^2 \right] \\ & \leq - \pas_{k+1} v_{min}
  \PE\left[\| h(\hatS^k) \|^2 \right] +
  \pas_{k+1}^2 \frac{L_{\dot \lyap}}{2} \PE\left[\| H_{k+1} \|^2 \right] \\ & \leq - \pas_{k+1} \left( v_{min} -
  \pas_{k+1} \frac{L_{\dot \lyap}}{2} \right) \PE\left[\| h(\hatS^k) \|^2 \right]  \\
& \qquad + \pas_{k+1}^2
  \frac{L_{\dot \lyap}}{2} \PE\left[\| H_{k+1} - h(\hatS^k) \|^2 \right]
\end{align*}
where we used that $\PE\left[H_{k+1} \vert \F_k \right] = h(\hatS^k)$
and \autoref{lem:fromVdot:to:h}. Set
\begin{align*}
A_k & \eqdef \PE\left[\| h(\hatS^k) \|^2 \right] \eqsp, \\
B_{k+1}
& \eqdef \PE\left[\| \Sronde^{k+1} - \bars \circ \map(\hatS^k) \|^2
  \right] \eqsp.
\end{align*}
By \autoref{lem:control:field} and \autoref{prop:variance:field}, we
have for any $k \geq 0$:
\begin{align*}
  \PE\left[\lyap(\hatS^{k+1}) \right] - \PE\left[\lyap(\hatS^k)
    \right] &\leq - \pas_{k+1} \left( v_{min} - \pas_{k+1}
  \frac{L_{\dot \lyap}}{2} \right) A_k - \pas_{k+1}^2 \frac{L_{\dot
      \lyap}}{2} B_{k+1} \\ &+ \pas_{k+1}^2 \frac{L_{\dot \lyap}}{2}
  \PE\left[ \| \bars_{J_{k+1}} \circ \map(\hatS^k) -
    \Smem_{k+1,J_{k+1}} \|^2 \right] \\ & \leq T_{1,k} + T_{2,k+1}
\end{align*}
by setting
\begin{align*}
T_{1,k} &\eqdef - \pas_{k+1} \left( v_{min} - \pas_{k+1} \frac{L_{\dot
    \lyap}}{2} \right) A_k + \pas_{k+1}^2 \frac{L_{\dot \lyap}}{2}
\sum_{j=0}^{k-1} \tilde \Lambda_{j+1,k} A_j \\ T_{2,k+1} & \eqdef -
\pas_{k+1}^2 \frac{L_{\dot \lyap}}{2} \left\{ B_{k+1}
+\sum_{j=0}^{k-1} \tilde \Lambda_{j+1,k} (1+\beta_{j+1}^{-1})^{-1}
B_{j+1} \right\} \eqsp;
\end{align*}
by convention, $\sum_{j=0}^1 a_j=0$.  By summing from $k=0$ to $k =
\kmax-1$, we have
\begin{multline*}
  \sum_{k=0}^{\kmax-1} \pas_{k+1} \left( v_{min} - \pas_{k+1}
  \frac{L_{\dot \lyap}}{2} \right) A_k - \frac{L_{\dot \lyap} \,
    L^2}{2} \sum_{k=0}^{\kmax-2} \pas_{k+1}^2 \Lambda_{k} \ A_k
  \\ \leq \Delta \lyap - \frac{L_{\dot \lyap}}{2} \sum_{k=0}^{\kmax
    -1} \pas_{k+1}^2 \left(L^2 \Xi_k + 1 \right) B_{k+1},
\end{multline*}
where for $0 \leq k \leq \kmax-2$ and with the convention
$\Lambda_{\kmax-1} = \Xi_{\kmax-1} =0$,
  \begin{align*}
\Lambda_k & \eqdef \left(1 + \frac{1}{\beta_{k+1}}
\right) \  \sum_{j=k+1}^{\kmax-1} \pas_{j+1}^2 \ \left(\frac{n-1}{n}
\right)^{j-k} \ \prod_{\ell=k+2}^j \left(1+ \beta_\ell + \pas_\ell^2
L^2 \right) \\ & \leq  \left(1 + \frac{1}{\beta_{k+1}}
\right) \ \sum_{j=k+1}^{\kmax-1} \pas_{j+1}^2 \ \prod_{\ell=k+2}^j
\left(1 - \frac{1}{n}+ \beta_\ell + \pas_\ell^2 L^2 \right),  \\
\Xi_k & \eqdef   \left(1 + \frac{1}{\beta_{k+1}}
\right)^{-1} \Lambda_k = \frac{\Lambda_k \beta_{k+1}}{1+\beta_{k+1}} \eqsp.
  \end{align*}
  Hence,
  \begin{multline*}
\sum_{k=0}^{\kmax-1} \{\pas_{k+1} \left( v_{min} - \pas_{k+1}
\frac{L_{\dot \lyap}}{2} \right) - \pas_{k+1}^2 \Lambda_{k}
\frac{L_{\dot \lyap} L^2}{2} \} \ A_k \\ + \sum_{k=0}^{\kmax -1}
\pas_{k+1}^2 \{ 1 + \Xi_{k} L^2 \}\frac{ L_{\dot\lyap}}{2} B_{k+1}
\leq \Delta \lyap \eqsp.
  \end{multline*}

\subsubsection{Proof of \autoref{coro:optimal:sampling}}
\label{sec:proof:coro:optimal:sampling}
It is a follow-up of \autoref{theo:FIEM:NonUnifStop}; the quantities
$\alpha_k, \Lambda_k, \delta_k$ introduced in the statement of
\autoref{theo:FIEM:NonUnifStop} are used below without being defined
again.  We consider the case when for $\ell \in [\kmax]^\star$,
\[
\beta_\ell \eqdef \frac{1-\lambda}{n^\pb} \eqsp, \qquad \pas_\ell^2 \eqdef
\frac{C}{L^2 n^{2\pc} \kmax^{2\pd}} \eqsp,
\]
for some $\lambda \in \ooint{0,1}$, $C >0$ and $\pb, \pc, \pd$ to be
defined in the proof in such a way that \textit{(i)} $\alpha_k \geq
0$, \textit{(ii)} $\sum_{k=0}^{\kmax-1} \alpha_k$ is positive and as
large as possible. Since there will be a discussion on $(n,C,\lambda)$, we make more explicit the dependence of some constants upon these
quantities: $\alpha_k$ will be denoted by $\alpha_k(n,C,\lambda)$.

With these definitions, we have
\begin{align*}
1-\frac{\rho_n}{n} & \eqdef 1 - \frac{1}{n} + \beta_\ell + \pas_\ell^2
L^2 = 1 - \frac{1}{n} \left(1 - \frac{1-\lambda}{n^{\pb-1}} -
\frac{C}{n^{2\pc-1} \kmax^{2 \pd}} \right) \eqsp,
\end{align*}
and choose $(\pb,\pc,\pd, \lambda, C)$ such that
\begin{equation}
  \label{eq:proof:coro:optsample:a}
  \frac{1-\lambda}{n^{\pb-1}} + \frac{C}{n^{2\pc -1}
  \kmax^{2 \pd}} < 1 \eqsp,
\end{equation}
which ensures that $\rho_n\in \ooint{0,1}$. Hence, for any $k \in [\kmax-2]$,
\begin{align*}
\Lambda_k & \leq n^\pb \left(\frac{1}{n^\pb} + \frac{1}{1-\lambda}
\right) \frac{C}{L^2 n^{2\pc} \kmax^{2 \pd}}\sum_{j=k+1}^{\kmax-1}
\left( 1- \frac{\rho_n}{n}\right)^{j-k-1}  \\
&\leq \left(\frac{1}{n^\pb} + \frac{1}{1-\lambda} \right)
\frac{C}{L^2 \rho_n} \frac{1}{n^{2\pc-\pb-1} \kmax^{2 \pd}}\eqsp.
\end{align*}
 From this upper bound, we deduce for any $k \in [\kmax-1]$:
 $\alpha_k(n,C,\lambda) \geq \underline{\alpha}_n(C,\lambda)$ where
\begin{align} \label{eq:proof:coro:optsample:d}
 \underline{\alpha}_n(C,\lambda) \eqdef \frac{\sqrt{C}}{L n^{\pc}
   \kmax^{\pd}} \left( v_\min - \frac{L_{\dot \lyap}}{2 L}
 \frac{\sqrt{C}}{n^\pc \kmax^\pd} - \frac{L_{\dot \lyap}}{2 L}
 \frac{C^{3/2}}{\rho_n \, n^{3 \pc -\pb-1} \kmax^{3 \pd}}
 \left(\frac{1}{n^\pb} + \frac{1}{1-\lambda} \right)\right) \eqsp.
  \end{align}
From \eqref{eq:proof:coro:optsample:a} and
\eqref{eq:proof:coro:optsample:d}, we choose $\pb =1$, $\pc = 2/3$,
$\pd = 0$; which yields for $n \geq 1$, since $\rho_n = \lambda - C
n^{-1/3}$
\[
n^{2/3} \underline{\alpha}_n(C,\lambda) \geq \mathcal{L}_n(C,\lambda) \eqsp,
\]
with
\begin{align*}
\mathcal{L}_n(C,\lambda) &\eqdef \frac{L_{\dot \lyap} \sqrt{C}}{2 L^2}
\left( v_\min \frac{2L}{L_{\dot \lyap}}- \sqrt{C} f_n(C,\lambda)
\right) \eqsp, \\ f_n(C,\lambda) & \eqdef \frac{1}{n^{2/3}}+
\frac{C}{\lambda - Cn^{-1/3}} \left(\frac{1}{n} + \frac{1}{1-\lambda}
\right)\eqsp.
\end{align*}
Let $\mu \in \ooint{0,1}$. Fix $\lambda \in \ooint{0,1}$ and $C>0$
such that (see \eqref{eq:proof:coro:optsample:a} for the second condition)
\begin{equation}
  \label{eq:proof:coro:optsample:b}
 \sqrt{C} f_n(C,\lambda) = 2 \mu v_\min \frac{L}{L_{\dot \lyap}} \eqsp, \qquad
 \frac{1}{n^{1/3}} < \frac{\lambda}{C} \eqsp.
\end{equation}
This implies that $n^{2/3} \alpha_k(n,C,\lambda) \geq n^{2/3}
\underline{\alpha}_n(C,\lambda) \geq n^{2/3} \alpha_\star(C) \eqdef
\sqrt{C} (1-\mu) v_\min / L$.  We obtain an upper bound on
$\mathsf{E}_1$ by
\begin{equation*}
\mathsf{E}_1 \leq \frac{1}{\kmax \ \alpha_\star(C)}
\sum_{k=0}^{\kmax-1} \alpha_k(n,C,\lambda) \, \PE\left[ \| h(\hatS^k)
  \|^2 \right] \eqsp.
\end{equation*}
 For $\mathsf{E}_2$, since
$\delta_k \geq L_{\dot \lyap} \pas_{k+1}^2 /2$,
\begin{multline*}
\frac{L_{\dot \lyap} \, \sqrt{C}}{2L(1-\mu) n^{2/3}} \frac{1}{v_\min}
\mathsf{E}_2  \\
\leq \frac{L_{\dot \lyap} \, C}{2L^2 n^{4/3}}
\frac{1}{\kmax \ \alpha_\star(C)} \sum_{k=0}^{\kmax-1} \PE\left[ \|
  \Sronde^{k+1} - \bars \circ \map(\hatS^k)
  \|^2 \right] \\ \leq \frac{1}{\kmax \ \alpha_\star(C)}
\sum_{k=0}^{\kmax-1} \delta_k \PE\left[ \| \Sronde^{k+1} - \bars \circ \map(\hatS^k) \|^2 \right] \eqsp.
\end{multline*}
We then conclude by
\begin{equation}\label{eq:proof:coro:optsample:c}
 \frac{1}{\kmax \, \alpha_\star(C)} = \frac{n^{2/3}}{\kmax}
 \frac{L}{\sqrt{C} (1-\mu) v_\min}\eqsp,
\end{equation}
and use $\sqrt{C} f_n(C,\lambda) = 2 \mu v_\min L / L_{\dot \lyap}$.

\subsubsection{Proof of \autoref{coro:optimal:sampling:Ketn}}
\label{sec:proof:coro:optimal:sampling:Ketn}
It is a follow-up of \autoref{theo:FIEM:NonUnifStop}; the quantities
$\alpha_k, \Lambda_k, \delta_k$ introduced in the statement of
\autoref{theo:FIEM:NonUnifStop} are used below without being defined
again.

We consider the case when, for $\ell\in[\kmax]^\star$,
\[
\beta_\ell \eqdef \frac{1-\lambda}{n^\pb}, \qquad \pas_\ell^2 \eqdef \frac{C}{L^2 n^{2\pc} \kmax^{2\pd}}
\]
for some $\lambda \in \ooint{0,1}$, $C >0$ and $\pb, \pc, \pd$ to be
defined in the proof in such a way that \textit{(i)} $\alpha_k \geq
0$, \textit{(ii)} $\sum_{k=0}^{\kmax-1} \alpha_k$ is positive and as
large as possible. Since there will be a discussion on $(n,C,\lambda)$, we make more explicit the dependence of some constants upon these
quantities: $\alpha_k$ will be denoted by $\alpha_k(n,C,\lambda)$.

With these definitions, we have
\[
\rho \eqdef 1 - \frac{1}{n} + \beta_\ell + L^2 \pas_\ell^2 = 1
- \frac{1}{n} \left(1 - \frac{1-\lambda}{n^{\pb-1}} - \frac{C}{n^{2\pc-1}
  \kmax^{2 \pd}} \right) \eqsp,
\]
and choose $(\pb,\pc,\pd, \lambda, C)$ such that
\begin{equation}
  \label{eq:proof:coro:optsample:Ketn:a}
  \frac{1-\lambda}{n^{\pb-1}} + \frac{C}{n^{2\pc -1}
  \kmax^{2 \pd}} \leq 1 \eqsp,
\end{equation}
which ensures that $\rho\in \ocint{0,1}$. Hence, for any $ k \in
[\kmax-2]$,
\begin{align*}
\Lambda_k & \leq n^\pb \left(\frac{1}{n^\pb} + \frac{1}{1-\lambda}
\right) \frac{C}{L^2 n^{2\pc} \kmax^{2 \pd}}\sum_{j=k+1}^{\kmax-1}
\rho^{j-k-1}  \\
& \leq \left(\frac{1}{n^\pb} + \frac{1}{1-\lambda} \right)
\frac{C}{L^2 n^{2\pc-\pb} \kmax^{2 \pd-1}}\eqsp.
\end{align*}
From this upper bound, we obtain the following lower bound for any $k \in [\kmax-1]$: $\alpha_k(n,C,\lambda) \geq
\underline{\alpha}_n(C,\lambda)$ where
\begin{multline*}
  (n^\pc \kmax^\pd) \ \underline{\alpha}_n(C,\lambda)
  \eqdef \frac{\sqrt{C}}{L} \left( v_\min - \sqrt{C} \frac{L_{\dot \lyap}}{2
    L} \left\{ \frac{1}{n^\pc \kmax^\pd}  \right. \right. \\
  \left. \left.  + \frac{C}{n^{3 \pc -\pb}
    \kmax^{3 \pd-1}} \left(\frac{1}{n^\pb} + \frac{1}{1-\lambda}
  \right) \right\} \right) \eqsp.
  \end{multline*}
Based on this inequality and on
\eqref{eq:proof:coro:optsample:Ketn:a}, we choose $\pb =1$ and $\pc =
\pd = 1/3$; which yields for $n \geq 1$,
\begin{align*}
& (n \kmax)^{1/3} \ \underline{\alpha}_n(C,\lambda) =
  \mathcal{L}_n(C,\lambda) \eqdef \frac{\sqrt{C}L_{\dot \lyap} }{2L^2}
  \left( v_\min \frac{2L}{L_{\dot \lyap}}- \sqrt{C} \tilde
  f_n(C,\lambda) \right) \eqsp, \\ & \tilde f_n(C,\lambda) \eqdef
  \frac{1}{(n \kmax)^{1/3}} + C \left(\frac{1}{n} +
  \frac{1}{1-\lambda} \right) \eqsp.
  \end{align*}
Let $\mu \in \ooint{0,1}$. Fix $\lambda \in \ooint{0,1}$ and $C>0$
such that (see \eqref{eq:proof:coro:optsample:Ketn:a} for the second condition)
\begin{equation}
  \label{eq:proof:coro:optsample:Ketn:b}
 \sqrt{C} \tilde f_n(C,\lambda) = 2 \mu v_\min \frac{L}{L_{\dot \lyap}} \eqsp, \qquad
 \frac{n^{1/3}}{\kmax^{2/3}} \leq \frac{\lambda}{C} \eqsp.
\end{equation}
This implies that
\begin{multline*}
(n \kmax)^{1/3} \alpha_k(n,C,\lambda) \geq (n \kmax)^{1/3}
  \underline{\alpha}_n(C,\lambda) \\ \geq (n \kmax)^{1/3}
  \alpha_\star(C) \eqdef \sqrt{C}(1-\mu) v_\min /L \eqsp.
\end{multline*}
We obtain the upper bound on $\mathsf{E}_1$ by
\begin{equation*}
\mathsf{E}_1 \leq \frac{1}{\kmax \ \alpha_\star(C)}
\sum_{k=0}^{\kmax-1} \alpha_k(n,C,\lambda) \ \PE\left[ \| h(\hatS^k)
  \|^2 \right] \eqsp.
\end{equation*}
For $\mathsf{E}_2$ and since $\delta_k \geq L_{\dot \lyap}
\pas_{k+1}^2 /2$
\begin{multline*}
\frac{L_{\dot \lyap} \, \sqrt{C}}{2(1-\mu)L n^{1/3}}
\frac{1}{\kmax^{1/3} v_\min} \mathsf{E}_2 \leq \frac{L_{\dot \lyap} \,
  C}{2L^2 n^{2/3} \kmax^{2/3}} \frac{1}{\kmax \ \alpha_\star(C)} \ \sum_{k=0}^{\kmax-1} \PE\left[ \|
  \Sronde^{k+1} - \bars \circ \map(\hatS^k) \|^2 \right] \\ \leq
\frac{1}{\kmax \ \alpha_\star(C)} \sum_{k=0}^{\kmax-1} \delta_k
\PE\left[ \| \Sronde^{k+1}- \bars \circ \map(\hatS^k) \|^2 \right]
\eqsp.
\end{multline*}
We then conclude by
\begin{equation}\label{eq:proof:coro:optsample:Ketn:c}
  \frac{1}{\kmax \alpha_\star(C)} = \frac{n^{1/3}}{\kmax^{2/3}}
  \frac{L}{\sqrt{C} (1-\mu) v_\min}\eqsp,
\end{equation}
and use $\sqrt{C} \tilde f_n(C,\lambda) = 2 \mu v_\min L / L_{\dot
  \lyap}$.

\subsubsection{Proof of \autoref{coro:given:sampling}}\label{sec:proof:coro:givensample}
It is a follow-up of \autoref{theo:FIEM:NonUnifStop}; the quantities
$\alpha_k, \Lambda_k, \delta_k$ introduced in the statement of
\autoref{theo:FIEM:NonUnifStop} are used below without being defined
again.

Let $p_0, \ldots, p_{\kmax-1}$ be positive real numbers such that
$\sum_{k=0}^{\kmax-1} p_k=1$.  We consider the case when
\[
\beta_\ell \eqdef \frac{1-\lambda}{n^\pb} \eqsp, \qquad \pas_\ell^2
\eqdef \frac{C_\ell}{L^2 n^{2 \pc} \kmax^{2 \pd}} \eqsp,
\]
for $\lambda \in \ooint{0,1}$, $C_\ell >0$, and $\pb, \pc, \pd$ to be
defined in the proof.

The first step consists in the definition of a function $\mathcal{A}$
and of a family $\mathcal{C}$ of vectors $\underline{C} = (C_1,
\ldots, C_{\kmax}) \in (\rset^+)^{\kmax}$ such that
\[
\alpha_k \geq \mathcal{A}(C_{k+1}) \geq 0 \eqsp, \qquad
\sum_{\ell=0}^{\kmax-1}\mathcal{A}(C_{\ell+1}) >0 \eqsp.
\]
The second step proves that we can find $\underline{C} \in
\mathcal{C}$ such that $p_k = \mathcal{A}(C_{k+1}) /
\sum_{\ell=0}^{\kmax-1} \mathcal{A}(C_{\ell+1})$ for any $k \in [\kmax-1]$.

Such a pair $(\mathcal{A}, \underline{C})$ is not unique, and among the possible
ones, we indicate two strategies, all motivated by making the sum
$\sum_{\ell=0}^{\kmax-1} \mathcal{A}(C_{\ell+1})$ as large as possible.

{\bf Step 1- Definition of the function $\mathcal{A}$.}  With the definition of
the sequences $\pas_\ell$ and $\beta_\ell$, we have
\[
  1-\frac{\rho_{n,\ell}}{n} \eqdef 1 - \frac{1}{n} + \beta_\ell +
  \pas_\ell^2 L^2 = - \frac{1}{n} \left(1 -
  \frac{1-\lambda}{n^{\pb-1}} - \frac{C_\ell}{n^{2\pc-1} \kmax^{2
      \pd}} \right)
\]
and choose $(\pb,\pc,\pd, \lambda, C_\ell)$ such that
\begin{equation}
\label{eq:proof:coro:givensample:a}
\frac{1-\lambda}{n^{\pb-1}} + \frac{C_\max}{n^{2\pc -1} \kmax^{2 \pd}}
< 1 \eqsp, \ \text{where} \ C_\max \eqdef \max_\ell C_\ell \eqsp,
\end{equation}
which ensures that $\rho_{n,\ell} \in \ooint{0,1}$.  Define
\[
\rho_n \eqdef \min_\ell \rho_{n,\ell} = 1 -
\frac{1-\lambda}{n^{\pb-1}} - \frac{C_\max }{n^{2\pc-1} \kmax^{2 \pd}}
\eqsp.
\]Hence, for any $k \in [\kmax-2]$,
\begin{align*}
\Lambda_k & \leq n^\pb \left(\frac{1}{n^\pb} + \frac{1}{1-\lambda}
\right) \frac{1}{L^2 n^{2\pc} \kmax^{2 \pd}} \sum_{j=k+1}^{\kmax-1}
C_{j+1} \left(1-\frac{\rho_n}{n}\right)^{j-k-1} \\ &\leq
\left(\frac{1}{n^\pb} + \frac{1}{1-\lambda} \right) \frac{C_\max}{L^2
  \rho_n} \frac{1}{n^{2\pc-\pb-1} \kmax^{2 \pd}}\eqsp.
\end{align*}
 From this upper bound, we obtain the following lower bound on
 $\alpha_k$, for any $k \in [\kmax-1]$,
\[
\alpha_k \geq \frac{\sqrt{C_{k+1}} }{L n^{\pc} \kmax^{\pd}} \left(
v_\min - \frac{L_{\dot \lyap}}{2 L} \frac{\sqrt{C_{k+1}}}{n^\pc
  \kmax^\pd} - \frac{L_{\dot \lyap}}{2 L} \frac{C_\max
  \sqrt{C_{k+1}}}{\rho_n \, n^{3 \pc -\pb-1} \kmax^{3 \pd}}
\left(\frac{1}{n^\pb} + \frac{1}{1-\lambda} \right) \right) \eqsp.
\]
 Based on this inequality and on \eqref{eq:proof:coro:givensample:a},
 we choose $\pb =1$, $\pc = 2/3$, $\pd = 0$: this yields $\rho_n =
 \lambda -C_\max n^{-1/3}$ and $\alpha_k \geq \underline{\alpha}_k$
 with (see \eqref{eq:fn:statement} for the definition of $f_n$)
\begin{equation}\label{eq:proof:coro:givensample:alpha}
\underline{\alpha}_k \eqdef \frac{\sqrt{C_{k+1}} L_{\dot \lyap}}{2 L^2
  n^{2/3}} \left( v_\min \frac{2L}{L_{\dot \lyap}}- \sqrt{C_{k+1}}
f_n(C_\max,\lambda)\right) \eqsp;
\end{equation}
the condition \eqref{eq:proof:coro:givensample:a} gets into $n^{-1/3}<
\lambda / C_\max$.

Define the quadratic function $x \mapsto \mathcal{A}(x) \eqdef A x ( v_\min -
Bx)$ where
\begin{equation}\label{eq:proof:coro:givensample:def:A:B}
A \eqdef \frac{1}{L n^{2/3}} \eqsp, \quad B \eqdef f_n(C,\lambda)  \frac{L_{\dot \lyap}}{2L}\eqsp;
\end{equation}
we have $\underline{\alpha}_k = \mathcal{A}(\sqrt{C_{k+1}})$.  By
\autoref{lem:polynomial:function} in the supplementary material, $\mathcal{A}$ is increasing on $\ocint{0,
  v_\min /(2B)}$, reaches its maximum at $x_\star \eqdef v_\min /(2B)$
and its maximal value is $\mathcal{A}_\star \eqdef A v^2_\min / (4B)$. In
addition, its inverse $\mathcal{A}^{-1}$ exists on $\ocint{0, \mathcal{A}_\star}$.

{\bf Step 2- Choice of $C_1,\ldots,C_{\kmax}$.}  We are now looking
for $C_1, \ldots, C_{\kmax}$ such that
\[
p_k = \mathcal{A}(\sqrt{C_{k+1}}) / \sum_{\ell =0}^{\kmax-1}
\mathcal{A}(\sqrt{C_{\ell+1}})
\] or equivalently
\begin{equation}
  \label{eq:fromptoC}
\frac{p_k}{p_I} = \frac{\mathcal{A}(\sqrt{C_{k+1}})}{\mathcal{A}(\sqrt{C_I})}, \qquad I
\in \mathrm{argmax}_k p_k \eqsp.
\end{equation}
It remains to fix $\mathcal{A}(\sqrt{C_I})$ in such a way that $\mathcal{A}$ is invertible
on $\ocint{0, \sqrt{C_I}}$. Since we also want $\sum_\ell
\mathcal{A}(\sqrt{C_{\ell+1}}) = \mathcal{A}(\sqrt{C_I}) / p_I$ as large as possible, and
$\mathcal{A}$ is increasing on $\ocint{0,x_\star}$, we choose
\begin{equation}
  \label{eq:cond:Cmax:F}
  \sqrt{C_I} = \sqrt{C_\max} = x_\star = \frac{v_\min}{2B} \eqsp.
\end{equation}
Therefore, $C_\max$ solves the equation $\sqrt{C_\max} = v_\min /
(2B)$ or equivalently
\begin{equation} \label{eq:Cmax:lambda}
\frac{v_\min L }{L_{\dot \lyap}} = \sqrt{C_\max} f_n(C_\max, \lambda) \eqsp,
\end{equation}
under the constraint that $\lambda \in \ooint{0,1}$ and $n^{-1/3} <
\lambda /C_\max$.  When $C_\max$ is fixed, we set
\[
\sqrt{ C_{k+1}} \eqdef \mathcal{A}^{-1}\left(\frac{p_k}{\max_\ell p_l}
\mathcal{A}(\sqrt{C_\max}) \right).
\]
With these definitions, we have (see \eqref{eq:fromptoC})
\[
 \frac{1}{\sum_{k=0}^{\kmax-1} \mathcal{A}(\sqrt{C_{k+1}})} = \frac{\max_\ell
   p_\ell}{\mathcal{A}( \sqrt{C_\max})}.
 \]
 Remember that
\[
\mathcal{A}(\sqrt{C_\max}) = \mathcal{A}(x_\star) = v_\min \sqrt{C_\max} /(2 L n^{2/3}) \eqsp.
\]

 {\bf Step 3. Lower bound on $\delta_k$}
 We write
 \[
\delta_k \geq \frac{L_{\dot \lyap}}{2} \pas_{k+1}^2 \eqsp,
 \]
 so that
 \[
 \frac{\delta_k}{\sum_{k=0}^{\kmax-1} \mathcal{A}(\sqrt{C_{k+1}})} \geq
 \frac{L_{\dot \lyap} L }{ v_\min} n^{2/3} \frac{\max_\ell
   p_\ell}{\sqrt{C_\max}} \pas_{k+1}^2 \eqsp.
 \]

  \subsubsection{Auxiliary results}
  \label{subsec:auxiliary}
  \begin{lemma} \label{lem:control:field}
   Assume A\autoref{hyp:model}, A\autoref{hyp:bars} and
   A\autoref{hyp:Tmap}.  For any $k \geq 0$,
   \[
\PE\left[ \| H_{k+1} \|^2 \right]  = \PE \left[ \| H_{k+1} - h(\hatS^k) \|^2 \right] + \PE \left[ \|  h(\hatS^k) \|^2 \right],
  \]
and \begin{multline*} \PE \left[ \| H_{k+1} - h(\hatS^k) \|^2 \right]
  + \PE\left[\| \Sronde^{k+1}- \bars \circ \map(\hatS^k) \|^2 \right]
  \\ = \PE\left[ \| \bars_{J_{k+1}} \circ \map(\hatS^k) -
    \Smem_{k+1,J_{k+1}} \|^2 \right] \eqsp.
\end{multline*}
  \end{lemma}
\begin{proof}
 Since $ \PE\left[ H_{k+1} \big \vert \F_{k+1/2}\right] = h(\hatS^k)$,
 we have
    \begin{align*}
     \PE\left[ \|H_{k+1} \|^2 \right] & = \PE\left[ \|H_{k+1} - h(
       \hatS^k) \|^2 \right] + \PE\left[ \| h(\hatS^k) \|^2 \right] \eqsp.
  \end{align*}
   In addition, upon noting that $\Smem_{k+1,i} \in \F_{k+1/2}$ for
   any $i \in [n]^\star$,
   \begin{align*}
     H_{k+1} - h( \hatS^k) & = \bars_{J_{k+1}} \circ \map(\hatS^k) -
     \Smem_{k+1,J_{k+1}} - \bars \circ \map(\hatS^k) + \Sronde^{k+1}
     \\ & =\bars_{J_{k+1}} \circ \map(\hatS^k) - \Smem_{k+1,J_{k+1}} - \PE\left[ \bars_{J_{k+1}} \circ \map(\hatS^k) -
       \Smem_{k+1,J_{k+1}} \Big \vert \F_{k+1/2}\right] \eqsp,
   \end{align*}
   we have
    \begin{multline*}
 \PE \left[ \| H_{k+1} - h(\hatS^k) \|^2 \right] + \PE\left[\|
   \Sronde^{k+1} - \bars \circ \map(\hatS^k) \|^2 \right] \\ =
 \PE\left[ \| \bars_{J_{k+1}} \circ \map(\hatS^k) -
   \Smem_{k+1,J_{k+1}} \|^2 \right] \eqsp.
    \end{multline*}
\end{proof}

\begin{proposition} \label{prop:variance:field}
  Assume A\autoref{hyp:model}, A\autoref{hyp:bars},
  A\autoref{hyp:Tmap} and
  A\autoref{hyp:regV:bis}-\ref{hyp:Tmap:smooth}.  Set $L^2 \eqdef
  n^{-1} \sum_{i=1}^n L_i^2$. Then
  \[
\PE\left[ \| \bars_{J_{1}} \circ \map(\hatS^0) -
    \Smem_{1,J_{1}} \|^2 \right] =0 \eqsp,
  \]
and for any $k \geq 1$ and $\beta_1, \ldots ,\beta_k>0$,
  \begin{multline*}
  \PE\left[ \| \bars_{J_{k+1}} \circ \map(\hatS^k) -
    \Smem_{k+1,J_{k+1}} \|^2 \right] \\ \leq \sum_{j=1}^k \widetilde
  \Lambda_{j,k} \left\{ \ \PE\left[\| h(\hatS^{j-1}) \|^2 \right] -
  \left(1 + \frac{1}{\beta_j} \right)^{-1} \PE\left[\| \Sronde^j -
    \bars \circ \map(\hatS^{j-1}) \|^2 \right] \right\} \eqsp,
  \end{multline*}
  where
\begin{align*}
 \widetilde \Lambda_{j,k} & \eqdef L^2 \, \left(\frac{n-1}{n}
 \right)^{k-j+1} \pas_j^2 \left(1 + \frac{1}{\beta_j} \right) \  \prod_{\ell=j+1}^k \left(1+ \beta_\ell + \pas_\ell^2 L^2 \right).
\end{align*}
By convention, $\prod_{\ell=k+1}^k a_\ell =1$.
  \end{proposition}
\begin{proof}
   \begin{align*}
    \PE\left[ \| \bars_{J_{1}} \circ \map(\hatS^0) - \Smem_{1,J_{1}}
      \|^2 \right] = \frac{1}{n}\sum_{i=1}^n \PE\left[ \| \bars_{i}
      \circ \map(\hatS^0) - \Smem_{1,i} \|^2 \right] =0 \eqsp.
  \end{align*}
  Let $k \geq 1$. We write (see \eqref{eq:def:Smem})
\begin{multline*}
\Smem_{k+1,i} = \Smem_{k,i} \un_{I_{k+1} \neq i} + \bars_i \circ \map(\hatS^k) \un_{I_{k+1}=i}  \\
= \bars_i \circ \map(\hatS^{< k,i}) \un_{I_{k+1} \neq i} + \bars_i \circ \map(\hatS^k) \un_{I_{k+1}=i},
\end{multline*}
where $\hatS^{<\ell,i}$ is defined by
\eqref{eq:memory:lastupdate}. This yields, by
A\autoref{hyp:regV:bis}-\ref{hyp:Tmap:smooth}
\begin{align} \label{eq:definition:Delta}
\frac{1}{n} \sum_{i=1}^n \PE\left[\| \bars_{i} \circ \map(\hatS^k) -
  \Smem_{k+1,i} \|^2 \right] & = \frac{1}{n} \sum_{i=1}^n \PE\left[\|
  \bars_{i} \circ \map(\hatS^k) - \bars_i \circ \map (\hatS^{< k,i})
  \|^2 \un_{I_{k+1} \neq i} \right] \nonumber \\ & \leq \Delta_k
\eqdef \frac{n-1}{n^2} \sum_{i=1}^n L_i^2 \, \PE\left[\| \hatS^k -
  \hatS^{< k,i} \|^2 \right].
\end{align}
 We have
\[
\Delta_k = \frac{n-1}{n^2} \sum_{i=1}^n L_i^2 \, \PE\left[\| \hatS^k -
  \hatS^{k-1} + \left(\hatS^{k-1} - \hatS^{<k-1,i} \right) \un_{I_{k}
    \neq i} \|^2 \right]
\]
where we used in the last inequality that
\[
\hatS^{< k,i} = \hatS^{k-1} \un_{I_{k} =i} + \hatS^{< k-1,i}
\un_{I_{k} \neq i} \eqsp.
\] Upon noting that
$2 \pscal{\tilde U}{V} \leq \beta^{-1} \|\tilde U\|^2 +\beta \|V\|^2$
for any $\beta >0$, we have for any $\mathcal{G}$-measurable r.v. $V$
\[
  \PE\left[\|U + V \|^2 \right] \leq \PE\left[\|U\|^2 \right] +
  \beta^{-1} \PE\left[\| \PE\left[U \vert \mathcal{G} \right] \|^2
  \right]  + (1+ \beta) \PE\left[\|V\|^2\right] \eqsp.
\]
Applying this inequality with $\beta \leftarrow \beta_k$, $U
\leftarrow \hatS^k - \hatS^{k-1} = \pas_k H_k$ and $\mathcal{G}
\leftarrow \F_{k-1/2}$ yields
\begin{align*}
   \Delta_k & \leq \pas_k^2 \frac{n-1}{n} L^2 \, \PE\left[\| H_k \|^2
     \right] + \frac{\pas_k^2}{\beta_k} \frac{n-1}{n} L^2 \,
   \PE\left[\| \PE\left[ H_k \vert \F_{k-1/2} \right] \|^2 \right]
   \\ &+ (1+\beta_k) \frac{n-1}{n^2} \sum_{i=1}^n L_i^2 \,
   \PE\left[\|\hatS^{k-1} - \hatS^{< k-1,i} \|^2 \un_{I_{k} \neq i}
     \right].
\end{align*}
By Lemma~\ref{lem:control:field}  and \eqref{eq:definition:Delta}, we have
\begin{align*}
\PE\left[\| H_k\|^2 \right] & \leq \PE\left[ \| h( \hatS^{k-1}) \|^2
  \right]  +\Delta_{k-1} -\PE\left[\| \Sronde^k - \bars \circ
  \map(\hatS^{k-1}) \|^2 \right];
\end{align*}
for the second term, we use again $\PE\left[ H_k \vert \F_{k-1/2}
  \right]= h(\hatS^{k-1})$; for the third term, since $I_{k} \in
\F_{k-1/2}$, $\hatS^{k-1} \in \F_{k-1}$, $\hatS^{< k-1,i} \in
\F_{k-1}$, then
\[
\sum_{i=1}^n L_i^2 \, \PE\left[\|\hatS^{k-1} - \hatS^{< k-1,i} \|^2 \un_{I_{k} \neq i}
  \right] = n \Delta_{k-1}.
\]
Therefore, we established
\begin{multline*}
  \Delta_k \leq \left( 1+ \beta_k + \pas_k^2 L^2 \right) \frac{n-1}{n}
  \Delta_{k-1}  \\
+ \pas_k^2 (1 + \frac{1}{\beta_k}) L^2 \ \frac{n-1}{n}
  \,\PE\left[\| h(\hatS^{k-1}) \|^2 \right] \\ - \pas_k^2 L^2
  \frac{n-1}{n}\PE\left[\| \Sronde^k - \bars \circ \map(\hatS^{k-1})
    \|^2 \right].
\end{multline*}
The proof is then concluded by standard algebra upon noting that
$\Delta_0 = 0$.
\end{proof}

%%%%%%%%%%%%%%%%%%%%%%%%%%%%%%%%
%%%%%%%%%%%%%%%%%%%%%%%%%%%%%%%%

%\bibliographystyle{spbasic}
% \bibliography{FIEM-biblio}

\clearpage
\newpage

\begin{center}
\large{\bf Supplementary material to "Fast Incremental Expectation Maximization for finite-sum  optimization: nonasymptotic convergence" }
\end{center}

This {\em supplementary material} provides
\begin{enumerate}
\item  proofs of some comments.
\item details and additional analyses for the numerical illustration
  on Gaussian Mixture Models (\autoref{sec:mixtureGaussian}).
\end{enumerate}

{\bf Notations.} Vectors are column vectors. For $a,b \in \rset^d$,
$\pscal{a}{b} = a^T b$ is the Euclidean scalar product; $\pscal{a}{b}$
denotes the standard Euclidean scalar product on $\rset^\ell$, for
$\ell \geq 1$; and $\|a\|$ the associated norm. For a matrix $A$,
$A^T$ is its transpose.

For a non negative integer $n$, $[n] \eqdef \{0, \cdots, n\}$ and
$[n]^\star \eqdef \{1, \cdots, n\}$. $a \wedge b$ is the minimum of
two real numbers $a,b$.

For two $p \times p$ matrices $A,B$, $\pscal{A}{B}$
is the trace of $B^T A$: $\pscal{A}{B} \eqdef \mathrm{Tr}(B^T
A)$. $\Id_p$ stands for the $p \times p$ identity matrix. $\otimes$
stands for the Kronecker product. $\mathcal{M}_p^+$ denotes the set of
the invertible $p \times p$ covariance matrices. $\mathrm{det}(A)$ is
the determinant of the matrix $A$.

$\mathcal{N}_p(\mu,\Gamma)$ denotes a $\rset^p$-valued Gaussian
distribution, with expectation $\mu$ and covariance matrix $\Gamma$.

\section{EM as a Majorize-Minimization algorithm}
\label{supp:sec:details2}
The following result shows that $\{\overline{F}(\cdot, \param'),
\param' \in \Param \}$ is a family of majorizing function of the
objective function $F$ from which a Majorize-Minimization approach for
solving \eqref{eq:problem} can be derived under
A\autoref{hyp:Tmap}. This MM algorithm is EM (see
\autoref{prop:MMproperty:item3}).
\begin{proposition}
  \label{prop:MMproperty}
  Assume A\autoref{hyp:model} and A\autoref{hyp:bars}.
  \begin{enumerate}
  \item  \label{prop:MMproperty:item1} For any $i \in [n]^\star$ and $\param' \in \Param$ we have 
   \[
\loss{i}(\cdot) \leq - \pscal{\bars_i(\param')}{\phi(\cdot)}+
    \mathcal{C}_i(\param') \eqsp.
\]
    \item \label{prop:MMproperty:item2} For any $\param' \in \Param$,
      we have $F \leq \overline{F}(\cdot, \param')$ and
      $\overline{F}(\param',\param') = F(\param')$.
      \item \label{prop:MMproperty:item3} Assume also
        A\autoref{hyp:Tmap}. Given $\param^0 \in \Param$, the sequence
        defined by $\param^{k+1} \eqdef \map \circ \bars (\param^k)$
        for any $k \geq 0$, satisfies $F(\param^{k+1}) \leq
        F(\param^k)$.
        \end{enumerate}
\end{proposition}
\begin{proof}
\textit{(proof of \autoref{prop:MMproperty:item1}).}  From the
Jensen's inequality, it holds
  \begin{multline*}
\loss{i}(\param) - \loss{i}(\param') \leq - \int_\Zset
\pscal{\s_i(z)}{\phi(\param) - \phi(\param')} \ p_i(z;\param') \,
\mu(\rmd z)  \\
= -\pscal{\bars_i(\param')}{\phi(\param) - \phi(\param')}  \eqsp;
\end{multline*}
  which concludes the proof.

  \textit{(proof of \autoref{prop:MMproperty:item2})} From
  \eqref{eq:problem} and \autoref{prop:MMproperty:item1}, it holds
\[
F(\param) \leq -\pscal{\bars(\param')}{\phi(\param)} +\frac{1}{n}
\sum_{i=1}^n \mathcal{C}_i(\param') + \R(\param) \eqsp.
\]

\textit{(proof of \autoref{prop:MMproperty:item3})} From
\autoref{prop:MMproperty:item2} and the definition of $\map$, it holds
\[
F(\map \circ \bars (\param^k)) \leq \overline{F}(\map \circ \bars (\param^k),\param^k) \leq
\overline{F}(\param^k,\param^k) = F(\param^k) \eqsp.
\]
\end{proof}

\section{Proof of the comments in \autoref{sec:FIEM:complexity}}
\subsection{Comments in Section~\autoref{sec:FIEM:errorrate:case1}}
\label{secApp:proof:coro:optimal:sampling}
\paragraph{$\bullet$  The choice $\lambda = C $.} Since $n \geq 2$, the second
condition in \eqref{eq:proof:coro:optsample:b} is satisfied with
$\lambda =C$. \eqref{eq:proof:coro:optsample:c} is a decreasing
function of $C$ so that by the first condition in
\eqref{eq:proof:coro:optsample:b}, $C$ solves
\[
\sqrt{C} \left( \frac{1}{n^{2/3}}+ \frac{1}{1 - n^{-1/3}}
\left(\frac{1}{n} + \frac{1}{1-C} \right) \right) = 2 \mu v_\min
\frac{L}{L_{\dot \lyap}} \eqsp.
\]
A solution exists in $\ooint{0,1}$ and is unique (see
\autoref{lem:trivial2}); it is denoted by $C^\star$. Since the LHS is
lower bounded by $C \mapsto \sqrt{C} (1-C)^{-1}$ on $\ooint{0,1}$, $C^\star$ is
upper bounded by $C^+ \in \ooint{0,1}$ solving
\[
\sqrt{C} = 2 \mu v_\min
\frac{L}{L_{\dot \lyap}} (1-C) \eqsp.
\]
This yields $C^+ = (\sqrt{1+4A^2}-1)/(2A)$ with $A \eqdef 2 \mu v_\min
L / L_{\dot \lyap}$. Note that $f_n(C^\star,C^\star) \leq
f_2(C^\star,C^\star) \leq f_2(C^+, C^+)$; for the second inequality,
\autoref{lem:trivial2} is used again.

\paragraph{$\bullet$ Another choice, for any $n$ large enough.}
When $n \to \infty$, we have 
\[
\mathcal{L}_n(C,\lambda) \uparrow \mathcal{L}_\infty(C, \lambda)
\eqdef \frac{L_{\dot \lyap} \sqrt{C}}{2L^2} \left( v_\min \frac{2L}{L_{\dot \lyap}}-
\frac{C^{3/2}}{\lambda} \frac{1}{1-\lambda} \right) \eqsp.
\]
By \autoref{lem:trivial3} applied with $A \leftarrow v_\min/L$ and $B \leftarrow 2
L_{\dot \lyap} / L^2$, we have $\mathcal{L}_\infty(C,\lambda) \leq
\mathcal{L}_\infty(C_\star,\lambda_\star)$ where
\begin{align*}
\lambda_\star & \eqdef \frac{1}{2} \eqsp, \qquad C_\star \eqdef
\frac{1}{4} \left( \frac{v_\min L}{L_{\dot \lyap}}\right)^{2/3}\eqsp, \\
\mathcal{L}_\infty(C_\star,\lambda_\star) & = \frac{3}{8}
\frac{v_\min}{L} \left( \frac{v_\min L}{L_{\dot \lyap}}\right)^{1/3}
\eqsp.
\end{align*}
In the proof of \autoref{coro:optimal:sampling}, we established that
for any $\lambda \in \ooint{0,1}$ and $C>0$ such that $\lambda - C
n^{-1/3} \in \ooint{0,1}$, we have
\[
n^{2/3} \alpha_k(n,C, \lambda) \geq n^{2/3}
\underline{\alpha}_n(C,\lambda) \geq \mathcal{L}_n(C,\lambda) \eqsp.
\]
Set $\tilde N_\star \eqdef (v_\min L/ L_{\dot \lyap})^2/8$; for any $n
\geq \tilde N_\star$, we have $\lambda_\star - C_\star n^{-1/3} \in
\ooint{0,1}$ so that
\[
n^{2/3} \alpha_k(n,C_\star, \lambda_\star) \geq n^{2/3}
\underline{\alpha}_n(C_\star,\lambda_\star) \geq
\mathcal{L}_n(C_\star,\lambda_\star) \eqsp.
\]
This implies that for any $k \in [\kmax-1]$,
\[
  \lim_n n^{2/3} \alpha_k(n, C_\star, \lambda_\star) \geq \lim_n
  n^{2/3} \underline{\alpha}_n(C_\star, \lambda_\star) \geq
  \mathcal{L}_\infty(C_\star,\lambda_\star) > 0 \eqsp,
\]
thus showing that for any $n$ large enough - let us say $n \geq
N_\star$ (with $N_\star$ which only depends upon $L, L_{\dot \lyap},
v_\min$), we have for any $k \in [\kmax-1]$,
\[
 n^{2/3} \alpha_k(n, C_\star, \lambda_\star) \geq
 \mathcal{L}_\infty(C_\star,\lambda_\star) > 0 \eqsp.
 \]
 Therefore, we first write
 \[
\mathsf{E}_1 \leq \frac{1}{\kmax}
\frac{n^{2/3}}{\mathcal{L}_\infty(C_\star,\lambda_\star)}
\sum_{k=0}^{\kmax-1} \alpha_k(n,C_\star, \lambda_\star) \PE\left[
  \|h(\hatS_k)\|^2 \right] \eqsp;
\]
we then write, by using $\delta_k \geq \pas_{k+1}^2 L_{\dot \lyap}/2$
and $\pas_{k+1}^2 = C_\star / (L^2 n^{4/3})$,
\begin{align*}
\frac{1}{3 n^{2/3}} \left(\frac{L_{\dot \lyap}}{L v_\min}
\right)^{2/3} \mathsf{E_2} &= \frac{n^{2/3} L_{\dot \lyap}}{2
  \mathcal{L}_\infty(C_\star,\lambda_\star)} \pas_{k+1}^2 \mathsf{E_2}
\\ & \leq \frac{1}{\kmax}
\frac{n^{2/3}}{\mathcal{L}_\infty(C_\star,\lambda_\star)}
\sum_{k=0}^{\kmax-1} \delta_k(n,C_\star, \lambda_\star) \PE\left[ \|
  \Sronde^{k+1} - \bars \circ \map(\hatS^k) \|^2 \right]
\end{align*}
from which we obtain
\[
\mathsf{E}_1 + \frac{1}{3 n^{2/3}} \left(\frac{L_{\dot \lyap}}{L
  v_\min} \right)^{2/3} \mathsf{E_2} \leq \frac{1}{\kmax}
\frac{n^{2/3}}{\mathcal{L}_\infty(C_\star,\lambda_\star)} \Delta \lyap
\eqsp.
\]
This concludes the proof.

\subsection{Comments in Section~\autoref{sec:FIEM:errorrate:case2}}
\label{secApp:errorrate:case2}
\paragraph{Complexity.}
For $\tau>0$, set $C = \lambda \tau$. Then for any $\lambda \in
\ooint{0,1}$,
\[
\sqrt{\lambda \tau } \tilde f_n(\lambda \tau,\lambda) =
\frac{\sqrt{\lambda} \sqrt{\tau}}{(n \kmax)^{1/3}} + \lambda^{3/2}
\tau^{3/2} \left(\frac{1}{n} + \frac{1}{1-\lambda} \right) \eqsp,
\]
which is a continuous increasing function of $\lambda$, which tends to
zero when $\lambda \to 0$ and to $ + \infty$ when $\lambda \to 1$.
Hence, there exists an unique $\lambda_\star \in \ooint{0,1}$,
depending upon $L, L_{\dot \lyap},v_\min, \tau, \mu$ and $n ,\kmax$
such that 
\[
\sqrt{\lambda_\star \tau } \tilde f_n(\lambda_\star
\tau,\lambda_\star) = 2 \mu v_\min L / L_{\dot
  \lyap} \eqsp.
\] Note however that since
$\sqrt{\lambda \tau } \tilde f_n(\lambda \tau,\lambda) \geq
\lambda^{3/2} \tau^{3/2} / (1-\lambda)$ for any
$\lambda \in \ooint{0,1}$, then
$\lambda_\star$ is upper bounded by the unique solution
$\lambda^+ \in \ooint{0,1}$ satisfying
$ L_{\dot \lyap} \lambda^{3/2} \tau^{3/2} / (2L(1-\lambda)) = \mu
v_\min$ (see \autoref{lem:trivial:4}). Such a solution
$\lambda^+$ only depends upon $L, L_{\dot \lyap},v_\min, \tau,
\mu$. Hence, for any $\tau>0$,
\[
\tilde f_n(\lambda \tau, \lambda) \leq \sup_{n,\kmax} \tilde
f_n(\lambda^+(\tau) \tau, \lambda^+(\tau))
\]
and the RHS does not depend on $n, \kmax$. This inequality implies that 
\[
\frac{n^{1/3}}{\kmax^{2/3}} \frac{L \, \tilde f_n(\lambda
  \tau,\lambda)}{ \mu (1-\mu) v_\min^2}\leq
\frac{n^{1/3}}{\kmax^{2/3}} M \qquad M \eqdef \frac{L \, }{ \mu
  (1-\mu) v_\min^2} \sup_{n,\kmax} \tilde f_n(\lambda^+(\tau) \tau,
\lambda^+(\tau))\eqsp.
\]
Hence, there exists $M>0$ depending upon $L, L_{\dot \lyap},v_\min,
\tau, \mu$ such that for any $\varepsilon > 0$,
\begin{multline*}
  \kmax \geq \left( \tau^{3/2} \sqrt{n} \right) \vee \left( M \sqrt{n}
    \varepsilon^{-3/2}\right) \Longrightarrow
  \frac{n^{1/3}}{\kmax^{2/3}} \frac{L \, \tilde f_n(\lambda
    \tau,\lambda)}{ \mu (1-\mu) v_\min^2} \leq \varepsilon \eqsp.
\end{multline*}

\paragraph{Another choice of $(\lambda,C)$, for any $n$ large enough.}
In this section, we consider that there exists $\tau>0$ such that
$\sup_{n, \kmax} n^{1/3} \kmax^{-2/3} \leq \tau$, that $n \to \infty$
and that $n \kmax \to \infty$. In this asymptotic, we have
$\mathcal{L}_n(C,\lambda) \uparrow \mathcal{L}_\infty(C,\lambda)$
where
\[
\mathcal{L}_\infty(C, \lambda) \eqdef \frac{\sqrt{C}}{L} \left(
v_\min - \frac{L_{\dot \lyap}}{2 L} \frac{C^{3/2}}{1-\lambda} \right)
\eqsp.
\]
For any $(C,\lambda) \in \rset^+ \times \ooint{0,1}$ s.t. $\tau \leq
\lambda / C$, we have $\mathcal{L}_\infty(C, \lambda) \leq
\mathcal{L}_\infty(C_\star(\lambda), \lambda)$ where
\begin{align*}
C_\star(\lambda) & \eqdef \left( \frac{v_\min L}{2 L_{\dot
    \lyap}}\right)^{2/3} (1-\lambda)^{2/3} \eqsp;
\end{align*}
see \autoref{lem:trivial3}.  The condition $C \tau \leq \lambda$
implies that this inequality holds for any $\lambda \in
\coint{\lambda_\star,1}$ where $\lambda_\star$ is the unique solution
of (see \autoref{lem:trivial:4})
\begin{align*}
  & \left(\frac{v_\min L}{2 L_{\dot \lyap}} \right)^2
  (1-\lambda_\star)^2 = \lambda^3_\star /\tau^3 \eqsp.
\end{align*}
Since $\mathcal{L}_\infty(C_\star(\lambda),\lambda) = \frac{3}{4}
\left( \frac{v_\min^4 }{ 2 L^2 L_{\dot \lyap}}\right)^{1/3}
(1-\lambda)^{1/3}$, this quantity is maximal by choosing $\lambda =
\lambda_\star$. Therefore, for any $(C,\lambda) \in \rset^+ \times
\ooint{0,1}$, s.t. $\tau \leq \lambda/C$, we have
\[
\lim_n n^{1/3} \kmax^{1/3}
\underline{\alpha}_n(C_\star(\lambda_\star),\lambda_\star) =
\mathcal{L}_\infty(C_\star(\lambda_\star),\lambda_\star) > 0 \eqsp.
\]
For any $n$ large enough (with a bound which only
depends upon $L, L_{\dot \lyap}, v_\min, \tau$), we have
\[
\frac{1}{\kmax \alpha_\star(C_\star,\lambda_\star)} = \frac{n^{1/3}}{\kmax^{2/3}} \frac{4}{3}
\left( \frac{ 2 L^2 L_{\dot \lyap}}{v_\min^4 }\right)^{1/3}
(1-\lambda_\star)^{-1/3} \eqsp.
\]
First, we write
\[
\mathsf{E}_1 \leq \frac{n^{1/3}}{\kmax^{2/3} \,
  \mathcal{L}_\infty(C_\star(\lambda_\star),\lambda_\star) }
\sum_{k=0}^{\kmax-1} \alpha_k(n,C_\star(\lambda_\star), \lambda_\star)
\, \PE\left[ \| h(\hatS^k)\|^2\right] \eqsp.
\]
Then we write that by using $\delta_k \geq \pas_{k+1}^2 L_{\dot
  \lyap}/2$ and $\pas_{k+1}^2 = C_\star(\lambda_\star) / (L^2 n^{2/3}
\kmax^{2/3})$ that
\begin{align*}
 & \frac{2^{10/3} (1-\lambda_\star)^{-1/3} \mu^2 }{\tilde
    f_n^2(\lambda_\star \tau,\lambda_\star)} \left( \frac{L
    v_\min}{L_{\dot \lyap}} \right)^{2/3} \frac{1}{(n\kmax)^{1/3}}
  \mathsf{E_2} \\ &= \frac{L_{\dot \lyap}}{2 \,
    \mathcal{L}_\infty(C_\star(\lambda_\star),\lambda_\star)}
  \pas_{k+1}^2  n^{1/3} \kmax^{1/3} \mathsf{E_2} \\ & \leq \frac{n^{1/3} \kmax^{1/3}}{\kmax \,
    \mathcal{L}_\infty(C_\star(\lambda_\star),\lambda_\star)
  }\sum_{k=0}^{\kmax-1} \delta_k(n,C_\star(\lambda_\star),
  \lambda_\star) \PE\left[ \| \Sronde^{k+1} - \bars \circ
    \map(\hatS^k) \|^2 \right]
\end{align*}
We then conclude that
\begin{multline*}
  \mathsf{E}_1 + \frac{2^{10/3} (1-\lambda_\star)^{-1/3} \mu^2
  }{\tilde f_n^2(\lambda_\star \tau,\lambda_\star)} \left( \frac{L
    v_\min}{L_{\dot \lyap}} \right)^{2/3} \frac{1}{(n\kmax)^{1/3}}
  \mathsf{E_2} \\ \leq \frac{n^{1/3}}{\kmax^{2/3} \,
    \mathcal{L}_\infty(C_\star(\lambda_\star),\lambda_\star) }
  \ \Delta \lyap \eqsp. 
  \end{multline*}

\subsection{Comments in Section~\autoref{sec:FIEM:errorrate:case3}}
\label{secApp:errorrate:case3}
\paragraph{Case $\lambda = C$.} A simple strategy is to choose
$n \geq 2$ and $C_\max = \lambda$ solution of $v_\min/2 = \sqrt{C}
f_n(C,C)$. This solution exists and is unique, and it is upper bounded
by a quantity $C^+$ which depends only on $L, L_{\dot \lyap}, v_\min$
- see \autoref{secApp:proof:coro:optimal:sampling} for a similar
discussion.

\paragraph{Case $\lambda=1/2$.}
$f_n(C,\lambda)$ controls the errors $\mathsf{E}_i$ and we can choose
$\lambda \in \ooint{0,1}$ and then $C>0$ such that this quantity is
minimal; to make the computations easier, we minimize w.r.t. $\lambda$
the function $\lim_n f_n(C,\lambda)$: it behaves like $\lambda^{-1}
(1-\lambda)^{-1}$ so that we set $\lambda=1/2$.  The equation $
\sqrt{C} f_n(C,1/2) = v_\min L/L_{\dot \lyap}$ possesses an unique
solution $C_n$ in $\ooint{0, n^{1/3}/2}$.

Upon noting that $x \mapsto \sqrt{x} f_n(x,1/2)$ is lower bounded by
$x \mapsto 4 x^{3/2}$, $C_n$ satisfies
\[
C_n \leq \left(\frac{v_\min L}{4L_{\dot \lyap}} \right)^{2/3} \eqsp,
\]
thus showing that the constraint $n^{-1/3} < \lambda/C_n = 1/(2C_n)$
is satisfied for any $n$ such that $8 n> \left(v_\min L/L_{\dot \lyap}
\right)^{2}$.

\section{Technical Lemmas}
\begin{lemma} \label{lem:polynomial:function}
  Let $A,B, v>0$ and define $F(x) \eqdef A x (v - Bx)$ on
  $\rset$. Then the roots of $F$ are $\{0, v/B \}$; $F$ is positive on
  $\ooint{0, v/B}$; the maximal value of $F$ is $A v^2/(4B)$ and it is
  reached at $x_\star \eqdef v/2B$.
\end{lemma}

\begin{lemma}
  \label{lem:trivial2}
Let $a,b>0$ and define $F$ on $\ooint{0,1}$ by $F(x) = \sqrt{x} (a +
b/(1-x))$. $F$ is increasing on $\ooint{0,1}$ and for any $v>0$, there
exists an unique $x \in \ooint{0,1}$ such that $F(x) = v$.
\end{lemma}
\begin{proof}  $x
\mapsto F(x)$ is continuous and increasing on $\ooint{0,1}$, tends to
zero when $x \to 0$ and to $+\infty$ when $x \to 1$; therefore for any
$v>0$, there exists an unique $x\in \ooint{0,1}$ such that $F(x) = v$.
\end{proof}

\begin{lemma}
  \label{lem:trivial3}
  Let $A,B >0$.  The function $F: x \mapsto Ax - B x^4$ defined on
  $\ooint{0,\infty}$ reaches its unique maximum at $x_\star \eqdef A^{1/3}
  B^{-1/3} 4^{-1/3}$ and $F(x_\star) =3 A^{4/3} / (B 4^4)^{1/3}$.
\end{lemma}
\begin{proof}
$F'(x) = A -4 B x^3$ and $F''(x) = - 12 B x^2 < 0$; hence, $F'$ is
  decreasing. $F'(x) =0$ iff $x^3 = A/(4B)$, showing $F'>0$ on
  $\ooint{0,x_\star}$ with $x_\star \eqdef A^{1/3}/(4B)^{1/3}$. Hence,
  $F$ is increasing on $\ccint{0,x_\star}$ and then decreasing.
  \end{proof}

\begin{lemma}
  \label{lem:trivial:4}
 For any $v >0$, the function $x \mapsto (1-x)^2/x^3$ is decreasing on
 $\ooint{0,1}$ and there exists an unique $x \in \ooint{0,1}$ solving
 $(1-x)^2/x^3 = v$.
  \end{lemma}
\begin{proof}
The derivative of $x \mapsto (1-x)^2/x^3$ is $-x^{-4} (x-3)(x-1)$ thus
showing that the function is decreasing on $\ooint{0,1}$; it tends to
$+\infty$ when $x \to 0$ and to $0$ when $x \to 1$. This concludes the
proof.\end{proof}

\section{Example:  Mixture of multivariate Gaussian distributions.}
Set
\begin{equation}
  \label{supp:eq:mixture:gaussian}
  f(y) \eqdef \sum_{\ell=1}^g \alpha_\ell \mathcal{N}_p(\mu_\ell,
  \Sigma)[y] \eqsp; \qquad \Gamma \eqdef \Sigma^{-1} \eqsp.
  \end{equation}
We write, up to the multiplicative constant $\sqrt{2 \pi}^{-p}$,
\begin{align*}
f(y) & = \sum_{z=1}^g \alpha_z \sqrt{\mathrm{det}(\Gamma)} \exp\left(-
\frac{1}{2}(y-\mu_z)^T \Gamma (y - \mu_z)\right) \\  & =
\sqrt{\mathrm{det}(\Gamma)} \ \exp\left( - \frac{1}{2} y^T \Gamma
y\right) \ \sum_{z=1}^g \exp\left( \sum_{\ell=1}^g 1_{z=\ell} \left\{
\ln \alpha_\ell - \frac{1}{2} \mu_\ell^T \Gamma \mu_\ell + \mu_\ell^T
\Gamma y \right\}\right) \eqsp; \qquad 
\end{align*}
    {\bf Parametric statistical model.} Set
    \[
\param \eqdef (\alpha_1, \ldots, \alpha_g, \mu_1, \ldots, \mu_g,
\Sigma) \eqsp,
    \]
and denote by $\mathcal{M}_p^+$ the set of the $p \times p$
positive definite matrices. Then we set
\[
\Param \eqdef \{\alpha_\ell \geq 0, \sum_{\ell=1}^g \alpha_\ell =1 \}
\times (\rset^{p})^g \times \mathcal{M}_p^+ \eqsp.
\]
{\bf Latent variable model in the exponential family.} The density \eqref{supp:eq:mixture:gaussian} is
of the form
\[
\exp(-\R_y(\param)) \ \sum_{z=1}^g \exp\left(\pscal{\s_y(z)}{\phi(\param)} \right)
\]
with
\[
\R_y(\param) \eqdef \frac{1}{2} y^T \Gamma y - \frac{1}{2}\ln \mathrm{det}(\Gamma)
\]
and $\s_y(z) \eqdef \A_y \rho(z) \in \rset^{g(1+p) } $ and $\phi = (\phi^{(1)}, \phi^{(2)})$ 
\begin{align}
  &  \A_y \eqdef \left[ \begin{matrix} \Id_{g}
      \\ \Id_g \otimes y \end{matrix}\right] \in \rset^{g(1+p) \times
    g} \\ & \rho(z) \eqdef \left[ \begin{matrix} \1_{z=1} \\ \ldots
      \\ \1_{z=g} \end{matrix}\right]  \in \rset^g \\ & \phi^{(1)}(\param) \eqdef
  \left[\ln \alpha_\ell - 0.5 \mu_\ell^T \Gamma \mu_\ell \right]_{1
    \leq \ell \leq g} \in \rset^g \label{supp:eq:def:phi1} \\ &
  \phi^{(2)}(\param) \eqdef \left[ (\Gamma \mu_\ell) \right]_{1
    \leq \ell \leq g} \in \rset^{pg} \label{supp:eq:def:phi2} \eqsp.
  \end{align}
Remember that $y^T A y = \mathrm{Tr}(A y y^T)$ is the scalar product
of $A$ and $y y^T$.

\subsection{The model}
Let $y_1, \ldots, y_n$ be $n$ $\rset^p$-valued observations; they are
modeled as the realization of a vector $(Y_1, \ldots, Y_n)$ with
distribution
\begin{itemize}
\item conditionally to a $[g]^\star$-valued vector of random
  variables $(Z_1, \ldots, Z_n)$, $(Y_1, \ldots, Y_n)$ are
  independent; and the conditional distribution of $Y_i$ is
  $\mathcal{N}_p(\mu_{Z_i}, \Sigma)$.
 \item the r.v.  $(Z_1, \ldots, Z_n)$ are i.i.d., $Z_1$ takes values
   on $[g]^\star$ with weights $\alpha_1, \ldots, \alpha_g$.
\end{itemize}
Equivalently, the random variables $(Y_1, \ldots, Y_n)$ are
independent with distribution $\sum_{\ell=1}^g \alpha_\ell \,
\mathcal{N}_p(\mu_\ell, \Sigma)$.

The goal is to estimate the parameter $\param \in \Param$ by a Maximum
Likelihood approach. 

\subsubsection{ The expression of $\loss{i}$ for $i \in [n]^\star$ and $\R$}
We want to minimize on $\Param$
\[
\param \mapsto \frac{1}{n} \sum_{i=1}^n \left(\R_{y_i}(\param) - \ln
\sum_{z=1}^g \exp\left(\pscal{\s_i(z)}{\phi(\param)} \right)\right)
\]
which is of the form $n^{-1} \sum_{i=1}^n \loss{i}(\param) +
\R(\param)$ with
\[
\loss{i}(\param) \eqdef - \ln \sum_{z=1}^g
\exp\left(\pscal{\s_i(z)}{\phi(\param)} \right) \eqsp, \qquad
\R(\param) \eqdef \frac{1}{2} \mathrm{Tr}\left(\Gamma \, \frac{1}{n}
\sum_{i=1}^n y_i y_i^T \right) - \frac{1}{2}\ln \mathrm{det}(\Gamma)
\eqsp;
\]
$\s_i$ is a shorthand notation
\begin{equation}
  \label{supp:eq:si}
  \s_i(z) \eqdef \s_{y_i}(z)  = \A_{y_i} \, \rho(z) \eqsp.
  \end{equation}
\subsubsection{The expression of $p_i(z,\param)$ and $\bars_i(\param)$ for $i \in [n]^\star$}
We have for any $u \in [g]^\star$,
\[
p_i(u,\param) \eqdef \frac{\alpha_u \ \mathcal{N}_p(\mu_u,
  \Sigma)[y_i]}{\sum_{\ell=1}^g \alpha_\ell \ 
  \mathcal{N}_p(\mu_\ell, \Sigma)[y_i]}
\]
so that 
\begin{align*}
  & \bars_i(\param) = \A_{y_i} \ \bar \rho_i(\param) \eqsp, 
  & \bar \rho_i(\param) = \left[ \begin{matrix} \bar \rho_{i,1}(\param) \\ \ldots  \\
 \bar \rho_{i,g}(\param)  \end{matrix} \right]\eqdef \left[ \begin{matrix} p_i(1,\param) \\ \ldots  \\
  p_i(g,\param)  \end{matrix} \right] \in \rset^g  \eqsp.
  \end{align*}

\subsubsection{The expression of $\map$.}
\label{secAPP:GMM:mapT} 
Let $s = (s^{(1)}, s^{(2)}) \in \rset^g \times \rset^{pg}$; we write
$\pscal{s}{\phi(\param)} = \sum_{j=1}^2
\pscal{s^{(j)}}{\phi^{(j)}(\param)}$ where $\phi^{(j)}$ are defined by
\eqref{supp:eq:def:phi1}-\eqref{supp:eq:def:phi2}.

Remember that $\map(s) = \mathrm{argmin}_{\param \in \Param}
-\pscal{s}{\phi(\param)}+\R(\param)$. We obtain $\map(s) = \{\alpha_1,
\cdots, \alpha_g, \mu_1, \cdots, \mu_g, \Sigma\}$ with
\begin{align*}
  \alpha_\ell &\eqdef \frac{s^{(1),\ell}}{ \sum_{u=1}^g s^{(1),u}}
  \eqsp, \\ \mu_\ell & \eqdef \frac{s^{(2),\ell}}{s^{(1),\ell}} \eqsp,
  \\ \Sigma & \eqdef \frac{1}{n} \sum_{i=1}^n y_i y_i^T -
  \sum_{\ell=1}^g s^{(1),\ell} \mu_\ell \, \mu_\ell^T \eqsp.
\end{align*}
The expressions of $\alpha_\ell, \mu_\ell$ are easily obtained; we
provide details for the covariance matrix. We have for any symmetric
matrix $H$
  \begin{align*}
  \ln \frac{\mathrm{det}(\Gamma+H)}{\mathrm{det}(\Gamma)} & = \ln
  \mathrm{det}(I + \Gamma^{-1}H) = \ln (1+ \mathrm{Tr}(\Gamma^{-1} H)
  +o(\|H\|)) \\ & = \mathrm{Tr}(\Gamma^{-1} H) +o(\|H\|) =
  \pscal{H}{\Gamma^{-1}} + o(\|H\|)
  \end{align*}
  $\map(s)$ depends on $\Gamma$ through the function
  \[
G(\Gamma) \eqdef -\frac{1}{2} \ln \mathrm{det}(\Gamma) + \frac{1}{2}
\pscal{\Gamma}{ \frac{1}{n}\sum_{i=1}^n y_i y_i^T + \sum_{\ell=1}^g
  s^{(1),\ell} \mu_\ell \mu_\ell^T} - \pscal{\Gamma}{\sum_{\ell=1}^g
  \mu_\ell (s^{(2),\ell})^T} \eqsp.
\]
Therefore
\[
G(\Gamma +H) - G(\Gamma) = \frac{1}{2} \pscal{H}{\Gamma^{-1}} + 
\frac{1}{2} \pscal{H}{ \frac{1}{n}\sum_{i=1}^n y_i y_i^T +
  \sum_{\ell=1}^g s^{(1),\ell} \mu_\ell \mu_\ell^T} -
\pscal{H}{\sum_{\ell=1}^g \mu_\ell (s^{(2),\ell})^T} \eqsp.
\]
This yields the update
\begin{align*}
\Sigma = \Gamma^{-1} & \eqdef \frac{1}{n}\sum_{i=1}^n y_i y_i^T +
\sum_{\ell=1}^g s^{(1),\ell} \mu_\ell \mu_\ell^T - 2 \sum_{\ell=1}^g
\mu_\ell (s^{(2),\ell})^T  \\
& = \frac{1}{n}\sum_{i=1}^n y_i y_i^T -
\sum_{\ell=1}^g s^{(1),\ell} \mu_\ell \mu_\ell^T
  \end{align*}
by using $\mu_\ell = s^{(2),\ell} / s^{(1),\ell}$.

  \subsubsection{The domain of $\map$.}
  \label{secApp:GMM:domainS}
 We will prove in the following sections that our algorithms {\bf all}
 require the computation of $\map(\hatS)$ for $\hatS$ of the form
 $n^{-1} \sum_{i=1}^n \A_{y_i} \widehat{\rho}_i$. Therefore, let us
 restrict our attention to the case \[ s = \frac{1}{n} \sum_{i=1}^n
 \A_{y_i} \, \rho_i \eqsp, \qquad \rho_i = \left[ \begin{matrix}
     \rho_{i,1} \\ \ldots \\ \rho_{i,g} \end{matrix} \right] \in
 \rset^g \eqsp,
 \]
 and let us formulate sufficient conditions on $\rho_{i,\ell}$ so that $\map(s) \in \Param$.
  \paragraph{The weights.} For all $\ell \in [g]^\star$,  we want $\alpha_\ell \in \ccint{0,1}$. Therefore, it is required
  \[
\forall \ell \in [g]^\star, \qquad \frac{\sum_{i=1}^n
  \rho_{i,\ell}}{\sum_{i=1}^n \sum_{u=1}^g \rho_{i,u}} \in \ccint{0,1}
\eqsp.
  \]
  \paragraph{The expectations.}
Upon noting that the expression of the log-likelihood of an
observation $y$ is unchanged if $y \leftarrow y-c$ and $\mu_\ell
\leftarrow \mu_\ell-c$ for any $c \in \rset^p$, we must have
  \[
\frac{1}{n} \sum_{i=1}^n y_i y_i^T - \sum_{\ell=1}^g s^{(1),\ell}
\mu_\ell \mu_\ell^T = \frac{1}{n} \sum_{i=1}^n (y_i-c) (y_i-c)^T -
\sum_{\ell=1}^g s^{(1),\ell} (\mu_\ell-c) (\mu_\ell-c)^T
\]
for any $c \in \rset^p$. This yields
\[
\sum_{\ell=1}^g s^{(1),\ell} =1 \eqsp, \qquad \sum_{\ell=1}^g
s^{(1),\ell} \mu_\ell = \frac{1}{n} \sum_{i=1}^n y_i \eqsp.
\]
Equivalently
\begin{equation}\label{eq:condition:poireau}
\frac{1}{n}\sum_{i=1}^n \sum_{\ell=1}^g \rho_{i,\ell} = 1, \qquad \frac{1}{n}
\sum_{i=1}^n y_i = \frac{1}{n}\sum_{i=1}^n \sum_{\ell=1}^g \rho_{i,\ell} y_i \eqsp.
\end{equation}
\paragraph{The covariance matrix.}
Finally, $\Sigma$ has to be definite positive: we have
\begin{align*}
\Sigma & = \frac{1}{n} \sum_{i=1}^n y_i y_i^T - \frac{1}{n}
\sum_{\ell=1}^g \left( \sum_{i=1}^n \rho_{i,\ell}\right)
\left(\sum_{i=1}^n \ \frac{\rho_{i,\ell}}{\sum_{j=1}^n \rho_{j,\ell}}
y_i \right) \left( \sum_{i=1}^n \frac{\rho_{i,\ell}}{\sum_{j=1}^n
  \rho_{j,\ell}} y_i \right)^T \\ & = \frac{1}{n} \sum_{i=1}^n \left(1
- \sum_{\ell=1}^g \rho_{i,\ell} \right) y_i y_i^T \\ & + \frac{1}{n} \sum_{i=1}^n
\sum_{\ell=1}^g \rho_{i,\ell} \left( y_i - \sum_{j=1}^n
\frac{\rho_{j,\ell}}{\sum_{r=1}^n \rho_{r,\ell}} y_j \right) \left(
y_i - \sum_{j=1}^n \frac{\rho_{j,\ell}}{\sum_{r=1}^n \rho_{r,\ell}}
y_j \right)^T \eqsp.
  \end{align*}

  \paragraph{As a conclusion,}  these conditions are satisfied if
  \[
\rho_{i,\ell} \geq 0, \qquad \sum_{\ell=1}^g \rho_{i, \ell} =1 \eqsp.
\]
Therefore, the domain of $\map$ contains
\[
\mathcal{S} \eqdef \left\{\frac{1}{n} \sum_{i=1}^n \A_{y_i} \rho_{i}:
\rho_i = (\rho_{i,1}, \ldots, \rho_{i,g}) \in (\rset_+)^g,
\sum_{\ell=1}^g \rho_{i,\ell}=1 \right\} \eqsp.
\]

\subsection{Algorithms}
\label{secAPP:GMM:algos}
\subsubsection{Notations}
Given $\param \in \Param$, define the a posteriori distribution for
all $i \in [n]^\star$ and $u \in [g]^\star$,
\[
p_i(u,\param) \eqdef \frac{\alpha_u \, \mathcal{N}_p(\mu_u,
  \Sigma)[y_i]}{\sum_{\ell=1}^g \alpha_\ell \, \mathcal{N}_p(\mu_\ell,
  \Sigma)[y_i]} \eqsp.
\]
For all $i \in [n]^\star$, set
\begin{align*}
  \bars_i(\param) & \eqdef \A_{y_i} \ \rho_i(\param) \ \qquad
  \rho_i(\param) = \left[ \begin{matrix}\rho_{i,1}(\param) \\ \cdots
      \\ \rho_{i,g}(\param) \end{matrix} \right] \eqdef
  \left[ \begin{matrix} p_i(1,\param) \\ \cdots
      \\ p_i(g,\param) \end{matrix} \right] \eqsp.
  \end{align*}
   For a subset $\batch \subseteq \{1, \ldots, n\}$ of size $\lbatch$,
   define
   \[
\bars_\batch(\param) \eqdef \frac{1}{\lbatch} \sum_{i \in \batch}
\bars_i(\param) = \frac{1}{\lbatch} \sum_{i=1}^n \A_{y_i} \,
\rho_i(\param) \, \1_{i \in \batch} \eqsp.
\]

\subsubsection{The EM algorithm}
\label{sec:supp:GMM:EM}
              {\bf Input.} the current value of the parameter
              $\param^k$.

              \noindent {\bf One iteration.} Compute the statistic
\begin{align*}
  & \bars(\param^k) \eqdef \frac{1}{n} \sum_{i=1}^n \A_{y_i} \,
  \rho_i(\param^k) \eqsp.
\end{align*}
Update the parameter $\param^{k+1} = \map(\bars(\param^k))$.

\noindent {\bf Is the statistic in the domain of $\map$ ?} 
We have $\param^{k+1} = \map(\bars(\param^k))$ with $\bars(\param^k) =
n^{-1} \sum_{i=1}^n \A_{y_i} \ \rho_i(\param^k)$. It is easily seen
that 
\[
\rho_{i,\ell}(\param^{k}) \geq 0, \qquad \sum_{\ell=1}^g
\rho_{i,\ell}(\param^{k}) = 1 \eqsp,
\]
which implies that $\hatS^{k+1} \in \mathcal{S}$ and therefore,
$\param^{k+1} \in \Param$. 

\noindent {\bf Update the parameters.} for $\ell \in [g]^\star$
\begin{align*}
  \alpha_\ell^{k+1} &\eqdef \frac{1}{n} \sum_{i=1}^n
  \rho_{i,\ell}(\param^k) \eqsp, \\ \mu_\ell^{k+1} & \eqdef
  \frac{\sum_{i=1}^n \rho_{i,\ell}(\param^k) \ y_i}{ \sum_{i=1}^n
    \rho_{i,\ell}(\param^k)}= \frac{1}{n \alpha_\ell^{k+1}}
  \sum_{i=1}^n \rho_{i,\ell}(\param^k) \ y_i \eqsp, \\ \Sigma^{k+1} &
  \eqdef \frac{1}{n} \sum_{i=1}^n y_i y_i^T - \sum_{\ell=1}^g
  \alpha_\ell^{k+1} \ \mu_\ell^{k+1} \ (\mu_\ell^{k+1})^T \eqsp.
  \end{align*}

\begin{algorithm}[htbp]
  \KwData{$\kmax \in \nset$, $\param^\init \in \Param$, $\Sigma_\star
    \eqdef n^{-1} \sum_{i=1}^n y_i y_i^T$} \KwResult{The EM sequence:
    $(\hatS^k, \param^k), k \in [\kmax]$} {\tcc*[h]{Initialization}}
  \; Compute $\hatS^{0} = n^{-1} \sum_{i=1}^n \A_{y_i} \,
  \rho_i(\param^\init)$ for $j=1,2$ \; Set $\param^0 = \param^\init$
  \; \For{$k=0, \ldots, \kmax-1$}{\tcc*[h]{Update the statistics} \;
    \qquad $\hatS^{k+1} = n^{-1} \sum_{i=1}^n \A_{y_i}
    \ \rho_{i}(\param^k) $ \; \tcc*[h]{Update the parameter
      $\param^{k+1}$} \; \qquad $\alpha_\ell^{k+1} = n^{-1}
    \sum_{i=1}^n \rho_{i,\ell}(\param^k)$ for $\ell \in [g]^\star$\;
    \qquad $\mu_\ell^{k+1} = n^{-1} \sum_{i=1}^n
    \rho_{i,\ell}(\param^k) \, y_i / \alpha_\ell^{k+1}$ for $\ell \in
        [g]^\star$ \; \qquad $\Sigma^{k+1} = \Sigma_\star -
        \sum_{\ell=1}^g \alpha_\ell^{k+1} \mu_\ell^{k+1}
        (\mu_\ell^{k+1})^T$}
    \caption{The EM algorithm. Total number of calls to the examples:
      $\kmax \times n$.  } \label{algo:EM}
\end{algorithm}
 
\subsubsection{The iEM algorithm}
\label{sec:supp:GMM:iEM}
{\bf Input:}
\begin{itemize}
\item the current value of the parameter $\param^k$
  \item a step size $\pas_{k+1} \in \ccint{0,1}$.
 \item the current value of the statistic $\hatS^k = n^{-1}
   \sum_{i=1}^n \A_{y_i} \hat \rho_i^k$ where $\hat \rho_{i,\ell}^k \geq
   0$ and $ \sum_{\ell=1}^g \hat \rho_{i,\ell}^k =1$.
\item the current memory vectors $\Smem^k_i = \A_{y_i} \, \tilde
  \rho^k_i$ for $i=1, \ldots, n$, where $\sum_{\ell=1}^g \tilde
  \rho^k_{i,\ell} =1$.
\item the current mean of this vector $\Sronde^k = n^{-1} \sum_{i=1}^n
  \A_{y_i} \tilde \rho^k_i$.
\end{itemize}
    {\bf One iteration.}  Sample at random a set $\batch_{k+1}$ of
    $\lbatch$ integers in $[n]^\star$, possibly with
    replacement. \\ Update the {\tt memory} quantities: for $i \notin
    \batch_{k+1}$, $\Smem_i^{k+1} = \Smem^k_i$ and otherwise for all
    $i \in \batch_{k+1}$, $\Smem_i^{k+1} \eqdef \A_{y_i}
    \ \rho_i(\param^k)$. \\ Update its mean
\begin{align*}
\Sronde^{k+1} & \eqdef \Sronde^k + \frac{1}{n} \left( \sum_{i \in
  \batch_{k+1}} \Smem_i^{k+1} - \sum_{i \in \batch_{k+1}} \Smem^k_i
\right) \\ & = \frac{1}{n} \sum_{i=1}^n \A_{y_i} \ \left( (1- \1_{i
  \in \batch_{k+1}}) \tilde \rho^k_i + \rho_i(\param^k) \1_{i \in
  \batch_{k+1}}\right) \eqsp.
\end{align*}
Update the statistics $\hatS $ by setting
\begin{align*}
  \hatS^{k+1} & \eqdef (1-\pas_{k+1}) \hatS^k + \pas_{k+1}
  \Sronde^{k+1} \\ & = \frac{1}{n} \sum_{i=1}^n \A_{y_i} \ \left(
  (1-\pas_{k+1}) \hat \rho_i^k + \pas_{k+1} (1- \1_{i \in \batch_{k+1}})
  \tilde \rho^k_i + \pas_{k+1} \rho_i(\param^k) \1_{i \in
    \batch_{k+1}} \right) \eqsp.
\end{align*}
{\bf Induction assumption on the expression of $\Sronde^{k+1}$ and
  $\hatS^{k+1}$.}  $\Sronde^{k+1}$ is of the form $n^{-1} \sum_{i=1}^n
\A_{y_i} \tilde \rho_i^{k+1}$ with, for any $i \in [n]^\star$,
\[
\tilde \rho_i^{k+1} \eqdef (1- \1_{i \in \batch_{k+1}}) \tilde
\rho^k_i + \rho_i(\param^k) \1_{i \in \batch_{k+1}}\eqsp;
\]
is is easily seen that $\sum_{\ell=1}^g \tilde \rho_{i,\ell}^{k+1} =
1$ and $\rho_{i,\ell}^{k+1} \geq 0$.

$\hatS^{k+1}$ is of the form $n^{-1} \sum_{i=1}^n \A_{y_i} \hat
\rho_i^{k+1}$ with, for any $i \in [n]^\star$,
\[
\hat \rho_i^{k+1} \eqdef (1-\pas_{k+1}) \hat \rho_i^k + \pas_{k+1} (1-
\1_{i \in \batch_{k+1}}) \tilde \rho^k_i + \pas_{k+1} \rho_i(\param^k)
\1_{i \in \batch_{k+1}} \eqsp.
\]
 Since $\sum_{\ell=1}^g \rho_{i,\ell}(\param^k) = \sum_{\ell=1}^g \hat
 \rho_{i,\ell} = \sum_{\ell=1}^g \tilde \rho^k_{i,\ell}=1$ then
 $\sum_{\ell=1}^g \hat \rho_{i,\ell}^{k+1} =1 $. In addition, since
 $\rho_i, \hat \rho_{i, \ell}^k$ and $\tilde \rho_{i,\ell}^k$ are non
 negative and $\pas_{k+1} \in \ccint{0,1}$, then $\hat
 \rho_{i,\ell}^{k+1}$ is non negative.

\noindent {\bf It the statistic $\hatS^{k+1}$ in the domain of $\map$
  ?}  We have just established that $\hat \rho_{i,\ell}^{k+1} \geq 0$
and $\sum_{\ell=1}^g \hat \rho_{i,\ell}^{k+1} =1 $.  Consequently
$\hatS^{k+1} \in \mathcal{S}$ which implies that $\map(\hatS^{k+1})
\in \Param$.

\noindent {\bf Update the parameters:}
\begin{align*}
  \alpha_\ell^{k+1} & \eqdef \hatS_\ell = \frac{1}{n } \sum_{i=1}^n
  \hat \rho_{i,\ell}^{k+1} \eqsp, \\ \mu_\ell^{k+1} & \eqdef
  \frac{\hatS_{g+(\ell-1)p+1:g+\ell p}^{k+1}}{\hatS_\ell^{k+1}} =
  \frac{1}{n \alpha_\ell^{k+1}} \sum_{i=1}^n y_i \, \hat
  \rho_{i,\ell}^{k+1}\eqsp, \\ \Sigma^{k+1} & \eqdef \frac{1}{n}
  \sum_{i=1}^n y_i y_i^T - \sum_{\ell=1}^g\alpha_\ell^{k+1}
  \mu_\ell^{k+1} (\mu_\ell^{k+1})^T \eqsp.
\end{align*}

\begin{algorithm}[htbp]
  \KwData{$\kmax \in \nset$, $\param^\init \in \Param$, $\pas \in
    \ocint{0,1}$, $\lbatch \in \nset$, $\Sigma_\star \eqdef
    \frac{1}{n} \sum_{i=1}^n y_i y_i^T$.}  \KwResult{The incremental
    EM sequence: $(\hatS^k, \param^k), k \in [\kmax]$}
         {\tcc*[h]{Initialization}} \; Compute $\Smem_{0,i} = \A_{y_i}
         \, \rho_i(\param^\init)$ for $i \in [n]^\star$ \; Compute
         $\Sronde^{0} = \hatS^{0} = n^{-1} \sum_{i=1}^n \A_{y_i} \,
         \rho_i(\param^\init)$ \; Set $\param^0 = \param^\init$ \;
         \For{$k=0, \ldots, \kmax -1$}{Sample a mini-batch
           $\batch_{k+1}$ of size $\lbatch$ \; \tcc*[h]{Update the
             memory quantities} \; $\Smem_{k+1,i} = \Smem_{k,i}$ for
           $i \notin \batch_{k+1}$\; $\Smem_{k+1,i} = \A_{y_i} \,
           \rho_i(\param^k)$ for $i \in \batch_{k+1}$ \;
           $\Sronde^{k+1} = \Sronde^{k} + n^{-1} \sum_{i \in
             \batch_{k+1}} \left( \Smem_{k+1,i} - \Smem_{k,i} \right)$
           \; \tcc*[h]{Update the statistics} \; $\hatS^{k+1} =
           (1-\pas) \hatS^{k} + \pas \Sronde_{k+1}$ \; \tcc*[h]{Update
             the parameter $\param^{k+1}$} \; $\alpha_\ell^{k+1} =
           \hatS^{k+1}_\ell$ for $\ell \in [g]^\star$\;
           $\mu_\ell^{k+1} = \hatS^{k+1}_{g+(\ell-1)p+1:g+\ell p} /
           \hatS^{k+1}_\ell$ for $\ell \in [g]^\star $ \;
           $\Sigma^{k+1} = \Sigma_\star - \sum_{\ell=1}^g
           \alpha_\ell^{k+1} \mu_\ell^{k+1} (\mu_\ell^{k+1})^T $}
    \caption{The incremental EM algorithm. Total number of calls to
      the examples: $n+ \kmax \times \lbatch$.} \label{algo:EM}
\end{algorithm}

\subsubsection{The Online-EM algorithm}
\label{sec:supp:GMM:onlineEM}
{\bf Input.}
\begin{itemize}
\item the current value of the parameter $\param^k$
  \item a step size $\pas_{k+1} \in \ccint{0,1}$
\item the current value of the statistics $\hatS^k = n^{-1}
  \sum_{i=1}^n \A_{y_i} \ \hat \rho_i^k$ such that $\hat
  \rho_{i,\ell}^k\geq 0$ and $ n^{-1} \sum_{i=1}^n \sum_{\ell=1}^g
  \hat{\rho}_{i,\ell}^{k} =1$;
\end{itemize}
\noindent {\bf One iteration.} Sample at random a set $\batch_{k+1}$
of $\lbatch$ integers in $\{1, \ldots, n\}$, with NO replacement; and
compute the statistics
\begin{align*}
   \hatS^{k+1} & \eqdef (1-\pas_{k+1}) \hatS^{k} + \pas_{k+1} \frac{1}{\lbatch}
   \sum_{i \in \batch_{k+1}} \A_{y_i} \ \rho_i(\param^k) \eqsp,  \\ & = \frac{1}{n}
   \sum_{i=1}^n \A_{y_i} \ \left( (1-\pas_{k+1}) \hat \rho_i^k +
  \pas_{k+1}  \frac{n}{\lbatch} \rho_i(\param^k) \, \1_{i \in \batch_{k+1}} \right)  \eqsp.
\end{align*}
It is of the form $\hatS^{k+1} = n^{-1} \sum_{i=1}^n \A_{y_i} \hat
\rho_i^{k+1}$ with
\[
\hat \rho_i^{k+1} \eqdef (1-\pas_{k+1}) \hat \rho_i^k + \pas_{k+1}
\frac{n}{\lbatch} \rho_i(\param^k) \, \1_{i \in \batch_{k+1}} \eqsp.
\]
\noindent {\bf Induction assumption on the expression of
  $\hatS^{k+1}$.} Since $\rho_{i,\ell}(\param^k) \geq 0$, $\hat
\rho_{i,\ell}^k \geq 0$ and $\pas_{k+1} \in \ccint{0,1}$,  then $\hat
\rho_{i,\ell}^{k+1} \geq 0$.

Since $\batch_{k+1}$ is sampled with NO replacement, we have
$\sum_{i=1}^n \1_{i \in \batch_{k+1}} = \lbatch$ thus implying, by
using the induction assumption $n^{-1} \sum_{i=1}^n \sum_{\ell=1}^g
\hat{\rho}_{i,\ell}^{k} =1$ that
\[
n^{-1} \sum_{i=1}^n \sum_{\ell=1}^g \hat{\rho}_{i,\ell}^{k+1} =1
\eqsp.
\]
\noindent{\bf Is the statistic $\hatS^{k+1}$ in the domain of $\map$ ?}
We have $\param^{k+1} = \map(\hatS^{k+1})$ with $\hatS^{k+1} = n^{-1}
\sum_{i=1}^n \A_{y_i} \ \hat{\rho}_i^{k+1}$. Unfortunately, we can not
prove (even by induction) that $\sum_{\ell=1}^g \hat
\rho_{i,\ell}^{k+1} = 1$ for all $i \in [n]^\star$.  Therefore, we do
not have necessarily $\hatS^{k+1} \in \mathcal{S}$.

\noindent {\bf Update the parameters.} For $\ell \in [g]^\star$,
\begin{align*}
  \alpha_\ell^{k+1} &\eqdef \hatS_\ell^{k+1} = \frac{1}{n} \sum_{i=1}^n \hat
  \rho_{i,\ell}^{k+1} \eqsp, \\ \mu_\ell^{k+1} & \eqdef
  \frac{\hatS_{g+(\ell-1)p+1:g+\ell p}^{k+1}}{\hatS_\ell^{k+1}} = \frac{1}{n
    \alpha_\ell^{k+1}} \sum_{i=1}^n \hat \rho_{i,\ell}^{k+1} \,  y_i\eqsp, \\ \Sigma^{k+1} & \eqdef
  \frac{1}{n} \sum_{i=1}^n y_i y_i^T - \sum_{\ell=1}^g
 \alpha_\ell^{k+1} \mu_\ell^{k+1} (\mu_\ell^{k+1})^T \eqsp.
  \end{align*}

\begin{algorithm}[htbp]
  \KwData{$\kmax \in \nset$, $\param^\init \in \Param$, $\pas \in
    \ocint{0,1}$, $\lbatch \in \nset$, $ \Sigma_\star \eqdef
    \frac{1}{n} \sum_{i=1}^n y_i y_i^T$.}  \KwResult{The Online EM
    sequence: $(\hatS^k, \param^k), k \in [\kmax]$}
         {\tcc*[h]{Initialization}} \; Compute $\hatS^{0} = n^{-1}
         \sum_{i=1}^n \A_{y_i} \, \rho_i(\param^\init)$ for $j=1,2$ \;
         Set $\param^0 = \param^\init$ \; \For{$k=0, \ldots, \kmax
           -1$}{Sample a mini-batch $\batch_{k+1}$ of size $\lbatch$
           with no replacement \; \tcc*[h]{Update the statistics} \; $
           \bars_{\batch_{k+1}}(\param^{k}) = \lbatch^{-1} \sum_{i \in
             \batch_{k+1}} \A_{y_i} \, \rho_i(\param^k)$ \;
           $\hatS^{k+1} = (1-\pas) \hatS^{k} + \pas
           \bars_{\batch_{k+1}}(\param^{k})$ \; \tcc*[h]{Update the
             parameter $\param^{k+1}$} \; $\alpha_\ell^{k+1} =
           \hatS^{k+1}_\ell$ for $\ell \in [g]^\star$\;
           $\mu_\ell^{k+1} = \hatS^{k+1}_{g+(\ell-1)p+1:g+\ell p} /
           \alpha_\ell^{k+1}$ for $\ell \in [g]^\star$ \;
           $\Sigma^{k+1} = \Sigma_\star - \sum_{\ell=1}^g
           \alpha_\ell^{k+1} \mu_\ell^{k+1} (\mu_\ell^{k+1})^T$}
    \caption{The Online EM algorithm. Total number of calls to the
      examples: $n+ \kmax \times \lbatch$.} \label{algo:EM}
\end{algorithm}

\subsubsection{The FIEM algorithm}
\label{sec:supp:GMM:FIEM}
{\bf Input.}
\begin{itemize}
\item the current value of the parameter $\param^k$
  \item a step size $\pas_{k+1} \in \ccint{0,1}$.
 \item the current value of the statistics $\hatS^k = n^{-1}
   \sum_{i=1}^n \A_{y_i} \, \hat \rho_i^k$
\item the current memory vectors $\Smem^k_i = \A_{y_i} \, \tilde
  \rho^k_i $, with $\sum_{\ell=1}^g \hat \rho^k_{i,\ell}=1$.
\item the current mean of this vector $\Sronde^k = n^{-1} \sum_{i=1}^n
  \A_{y_i} \, \tilde \rho^k_i$, with $\sum_{\ell=1}^g \tilde
  \rho^k_{i,\ell}=1$ and $\tilde \rho^k_{i,\ell} \geq 0$.
\end{itemize}
{\bf One iteration.}  Sample a mini-batch $\batch_{k+1}$ of size
$\lbatch$ and update the {\tt memory} quantities: for $i \notin
\batch_{k+1}$, $\Smem_i^{k+1} = \Smem^{k}_i$ and otherwise for $i \in
\batch_{k+1}$, $\Smem_i^{k+1} \eqdef \A_i \, \rho_i(\param^k)$. \\
Update the mean of this memory quantity
\begin{align*}
\Sronde^{k+1} & \eqdef \Sronde^k + \frac{1}{n} \sum_{i \in
  \batch_{k+1}} \left( \Smem_i^{k+1} - \Smem_i^k \right) \\ & =
\frac{1}{n} \sum_{i=1}^n \A_{y_i} \, \left( (1- \1_{i \in \batch_{k+1}})
\tilde \rho^k_i + \rho_i(\param^k) \1_{i \in \batch_{k+1}} \right) \eqsp.
\end{align*}
Sample a second mini-batch $\batch'_{k+1}$ of size $\lbatch$. Compute
\begin{align*}
  & \bars_{\batch'_{k+1}}(\param^k) \eqdef \frac{1}{\lbatch} \sum_{i
    \in \batch'_{k+1}} \A_{y_i} \ \rho_i(\param^k) \eqsp.
  \end{align*}
Set $\Smem_{\batch'_{k+1}} \eqdef \lbatch^{-1} \sum_{j\in\batch'_{k+1}} \Smem_j^{k+1}$,
and update the statistics $\hatS^{k+1}$
\begin{align*}
   \hatS^{k+1} & \eqdef \hatS^k + \pas_{k+1} \left(
   \bars_{\batch'_{k+1}}(\param^k) - \hatS^k + \Sronde^{k+1} -
   \Smem_{\batch'_{k+1}} \right) \eqsp, \\ & = \frac{1}{n}
   \sum_{i=1}^n \A_{y_i} \ \left( (1-\pas_{k+1}) \hat \rho_i^k +
   \pas_{k+1} \frac{n}{\lbatch} \rho_i(\param^k) \1_{i \in
     \batch'_{k+1}} + \pas_{k+1} (1- \1_{i \in \batch_{k+1}}) \tilde
   \rho^k_i \right. \\ & \left. + \pas_{k+1} \rho_i(\param^k) \1_{i
     \in \batch_{k+1}} - \pas_{k+1} \frac{n}{\lbatch} \left( \tilde
   \rho_i^k \1_{i \notin \batch_{k+1}} + \rho_i(\param^k)\1_{i \in
     \batch_{k+1}} \right) \1_{i \in \batch'_{k+1}} \right)\eqsp.
\end{align*}
$\hatS^{k+1}$ is of the form $n^{-1} \sum_{i=1}^n \A_{y_i} \hat \rho_i^{k+1}$ with
\begin{align*}
  \hat \rho_i^{k+1} & \eqdef (1-\pas_{k+1}) \hat \rho_i^k + \pas_{k+1}
  \frac{n}{\lbatch} \rho_i(\param^k) \1_{i \in \batch'_{k+1}} +
  \pas_{k+1} (1- \1_{i \in \batch_{k+1}}) \tilde \rho^k_i \\ & +
  \pas_{k+1} \rho_i(\param^k) \1_{i \in \batch_{k+1}} - \pas_{k+1}
  \frac{n}{\lbatch} \left( \tilde \rho_i^k \1_{i \notin \batch_{k+1}}
  + \rho_i(\param^k)\1_{i \in \batch_{k+1}} \right) \1_{i \in
    \batch'_{k+1}} \eqsp.
\end{align*}
{\bf Induction assumption on $\Sronde^{k}$ and $\hatS^k$.} Since
$\sum_{\ell=1}^g \rho_{i,\ell} = \sum_{\ell=1}^g \hat \rho^k_{i,\ell}
=1$, then it is easily seen that $\sum_{\ell=1}^g \hat
\rho^{k+1}_{i,\ell} = 1$. 

We have $\Sronde^{k+1} = n^{-1} \sum_{i=1}^n \A_{y_i} \tilde
\rho_i^{k+1}$ with $\tilde \rho_i^{k+1} = (1- \1_{i \in \batch_{k+1}})
\tilde \rho^k_i + \rho_i(\param^k) \1_{i \in \batch_{k+1}}$.  It is
easily seen that under the induction assumption $\sum_{\ell=1}^g
\tilde \rho_{i,\ell}^{k+1} = 1$ and $\tilde \rho_{i,\ell}^{k+1} \geq
0$.

\noindent{\bf Is the statistic $\hatS^{k+1}$ in the domain of $\map$
  ?} The property $\hat \rho_{i,\ell}^{k+1} \geq 0$ may fail even
assuming that $\hat \rho_{i,\ell}^{k} \geq 0$.

\noindent{\bf Update the parameters}
\begin{align*}
  \alpha_\ell^{k+1} &\eqdef\hatS_\ell^{k+1} \eqsp, \\ \mu_\ell^{k+1} &
  \eqdef \frac{\hatS_{g + (\ell -1)p+1 :g+\ell
      p}^{k+1}}{\hatS_\ell^{k+1}} \eqsp, \\ \Sigma^{k+1} & \eqdef
  \frac{1}{n} \sum_{i=1}^n y_i y_i^T - \sum_{\ell=1}^g
  \alpha_\ell^{k+1} \mu_\ell^{k+1} (\mu_\ell^{k+1})^T \eqsp.
\end{align*}

\begin{algorithm}[htbp]
  \KwData{$\kmax \in \nset$, $\param^\init \in \Param$, $\pas \in \ocint{0,1}$,
    $\lbatch \in \nset$, $\Sigma_\star \eqdef n^{-1} \sum_{i=1}^n y_i
    y_i^T$.} \KwResult{The FIEM sequence: $(\hatS^k, \param^k), k \in [\kmax]$} {\tcc*[h]{Initialization}} \; Compute $\Smem_{0,i}
  = \A_{y_i} \, \rho_i(\param^\init)$ for $i =1, \ldots, n$ \; Compute
  $\Sronde^{0} = \hatS^{0} = n^{-1} \sum_{i=1}^n \A_{y_i} \,
  \rho_i(\param^\init)$ \; Set $\param^0 = \param^\init$ \; \For{$k=0,
    \ldots, \kmax -1$}{Sample independently two mini-batches
    $\batch_{k+1}$ and $\batch'_{k+1}$, both of size $\lbatch$ \;
    \tcc*[h]{Update the memory quantities} \; $\Smem_{k+1,i} =
    \Smem_{k,i}$ for $i \notin \batch_{k+1}$ \; $\Smem_{k+1,i} =
    \A_{y_i} \, \rho_i(\param^k)$ for $i \in \batch_{k+1}$ \;
    $\Sronde^{k+1} = \Sronde^{k} + n^{-1} \sum_{i \in \batch_{k+1}}
    \left( \Smem_{k+1,i} - \Smem_{k,i} \right)$ \; \tcc*[h]{Update the
      statistics} \; $\bars_{\batch'_{k+1}}(\param^k) = \lbatch^{-1}
    \sum_{r \in \batch'_{k+1}} \A_{y_r} \, \rho_r(\param^k)$ \;
    $\Smem_{k+1,\batch'_{k+1}} = \lbatch^{-1} \sum_{r \in
      \batch'_{k+1}} \Smem_{k+1,r}$ \; $\hatS^{k+1} = (1-\pas)
    \hatS^{k} + \pas \left( \bars_{\batch'_{k+1}}(\param^k) +
    \Sronde_{k+1} - \Smem_{k+1,\batch'_{k+1}}\right)$ \;
    \tcc*[h]{Update the parameter $\param^{k+1}$} \;
    $\alpha_\ell^{k+1} = \hatS^{k+1}_\ell$ for $\ell \in [g]^\star$\;
    $\mu_\ell^{k+1} = \hatS^{k+1}_{g + (\ell-1)p+1:g+\ell p} /
    \hatS^{k+1}_\ell$ for $\ell \in [g]^\star $ \; $\Sigma^{k+1} =
    \Sigma_\star - \sum_{\ell=1}^g \alpha_\ell^{k+1} \mu_\ell^{k+1}
    (\mu_\ell^{k+1})^T$}
    \caption{The FIEM algorithm. Total number of calls to
      the examples: $n+ \kmax \times 2\lbatch$.} \label{algo:EM}
\end{algorithm}

\clearpage
\newpage

\subsection{Additional plots}

\begin{figure*}[htbp]
  \includegraphics[width=\textwidth]{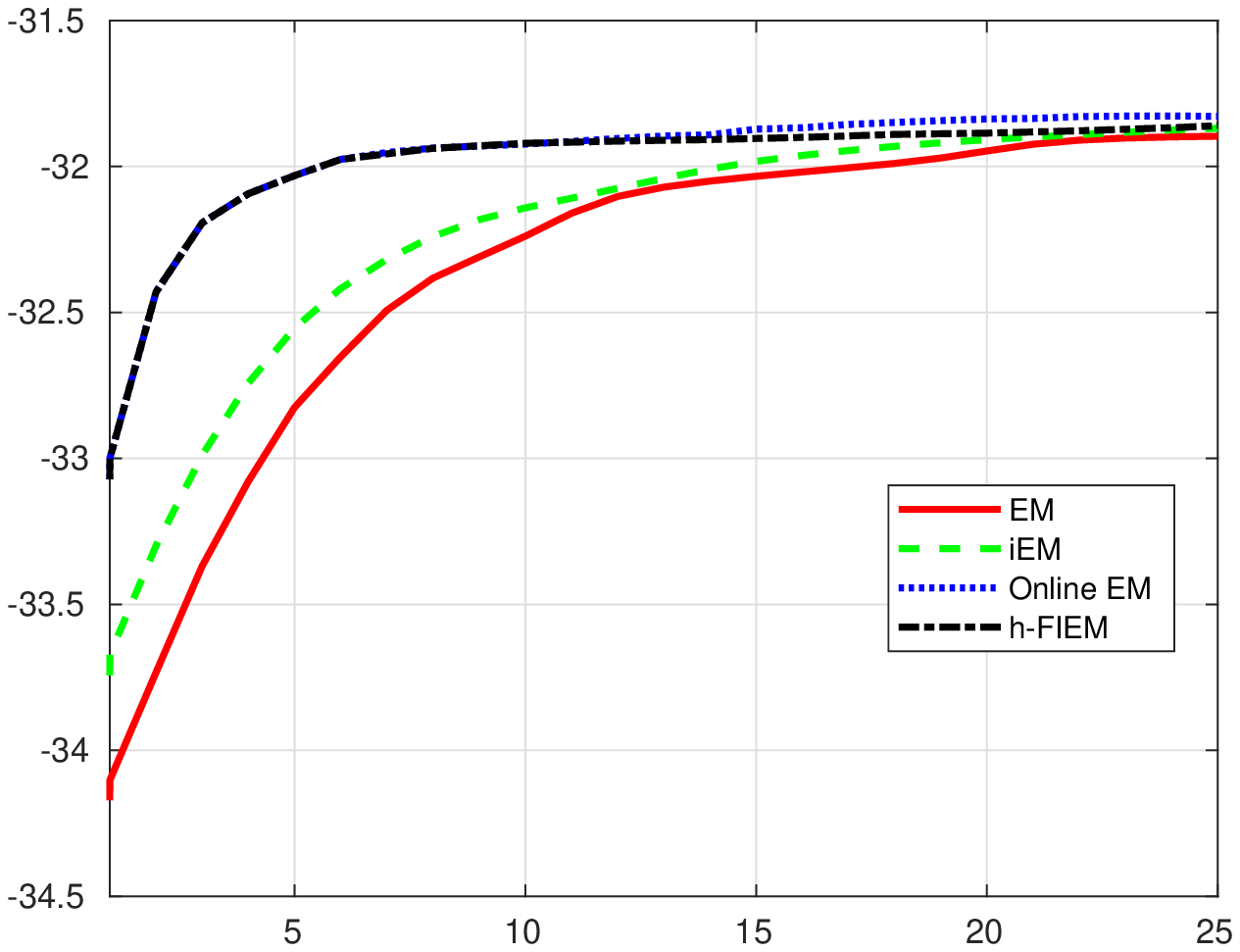}
\caption{Evolution of the normalized log-likelihood: average over $10$
  independent paths of length $100$ epochs. The first $25$ epochs are
  displayed. All the paths start from the same value at time $t=0$,
  with a normalized log-likelihood equal to $-58.31$.}
\end{figure*}

\begin{figure*}[htbp]
  \includegraphics[width=\textwidth]{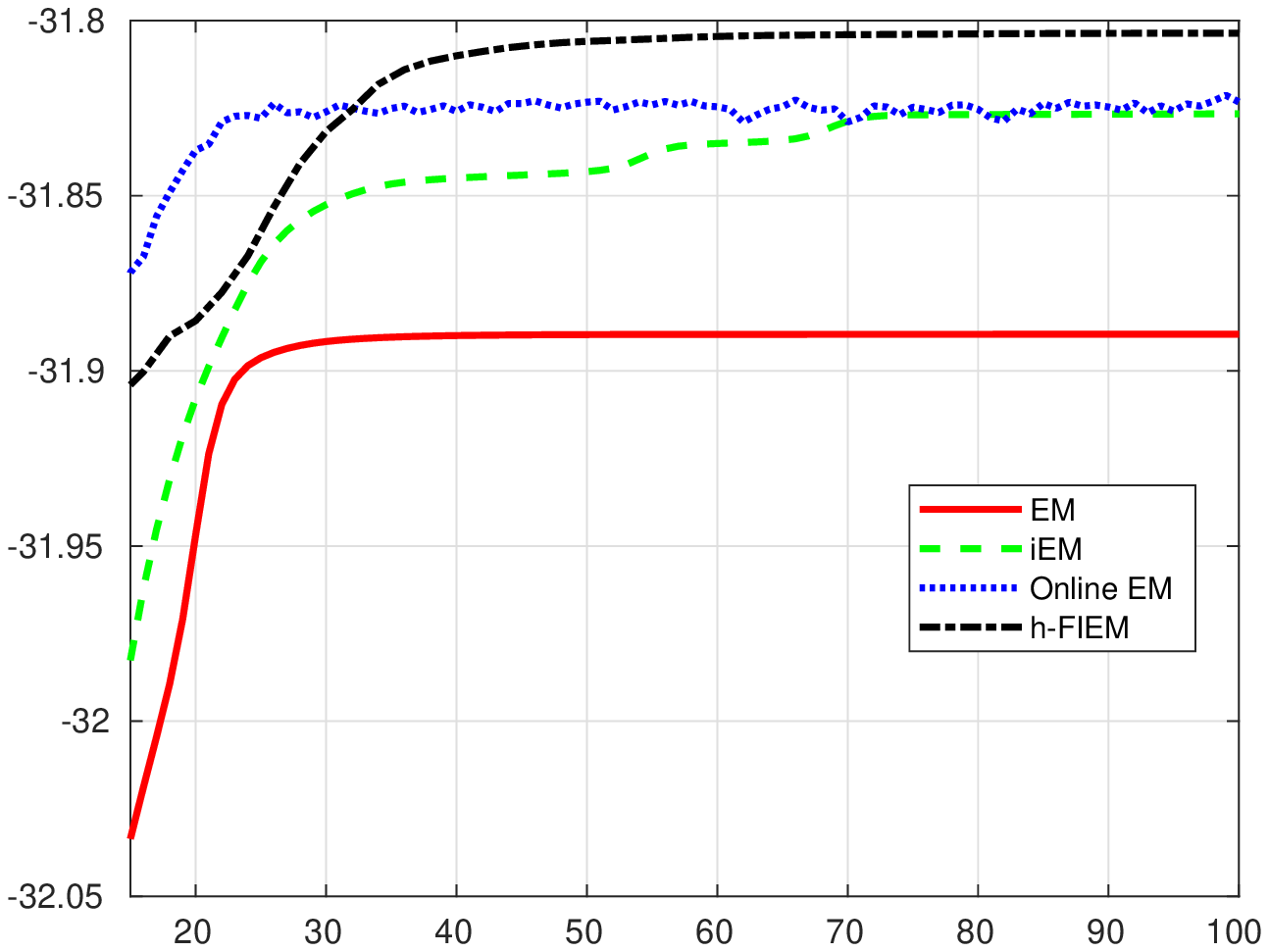}
\caption{Evolution of the normalized log-likelihood: average over $10$
  independent paths of length $100$ epochs.  The first $14$ epochs are
  discarded. All the paths start from the same value at time $t=0$,
  with a normalized log-likelihood equal to $-58.31$.}
\end{figure*}

\begin{figure*}[htbp]
  \includegraphics[width=\textwidth]{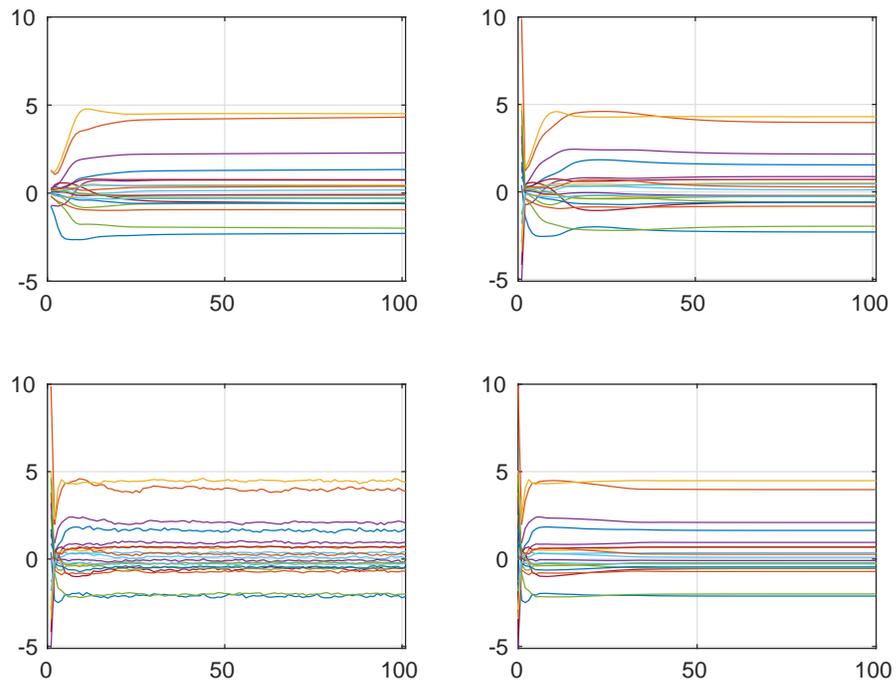}
  \caption{Evolution of the $p=20$ components of one of the
    expectation vector $\mu_\ell$ along one path of length $100$
    epochs. All the paths start from the same value at time $t=0$. EM
    (top left), iEM (top right), Online EM (bottom left) and h-FIEM
    (bottom right).}
\end{figure*}

\begin{figure*}[htbp]
  \includegraphics[width=\textwidth]{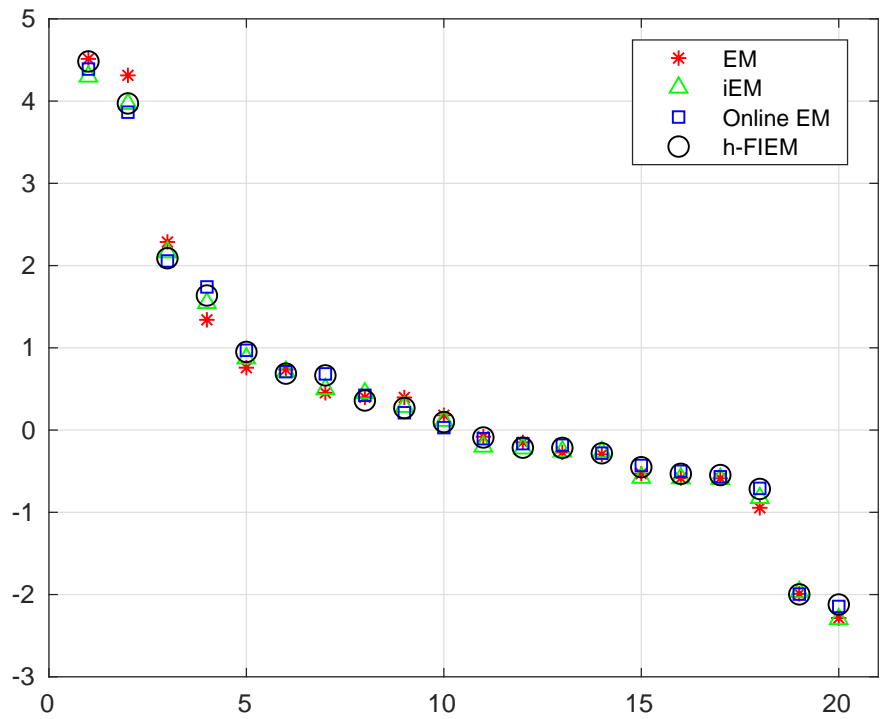}
\caption{Estimation of the $p=20$ components of one of the expectation
  vector $\mu_\ell$. The estimator is the value of the parameter
  obtained at the end of a single path of length $100$ epochs.}
\end{figure*}

\begin{figure*}[htbp]
  \includegraphics[width=\textwidth]{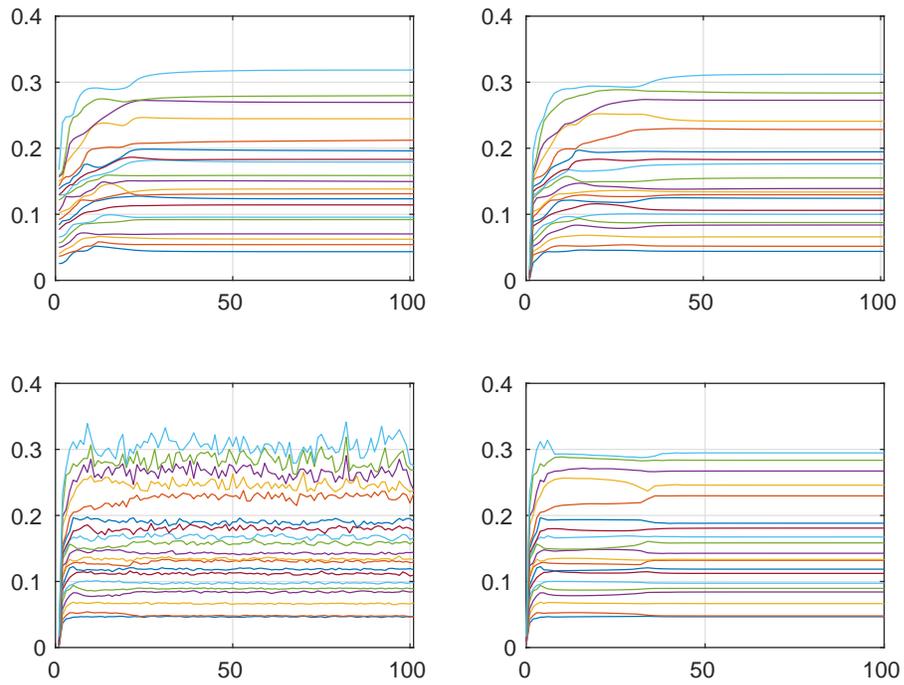}
  \caption{Evolution of the $p=20$ eigenvalues of the covariance
    matrix $\Sigma$ along one path of length $100$ epochs. All the
    paths start from the same value at time $t=0$. EM (top left), iEM
    (top right), Online EM (bottom left) and h-FIEM (bottom right).}
\end{figure*}

\begin{figure*}
  \includegraphics[width=\textwidth]{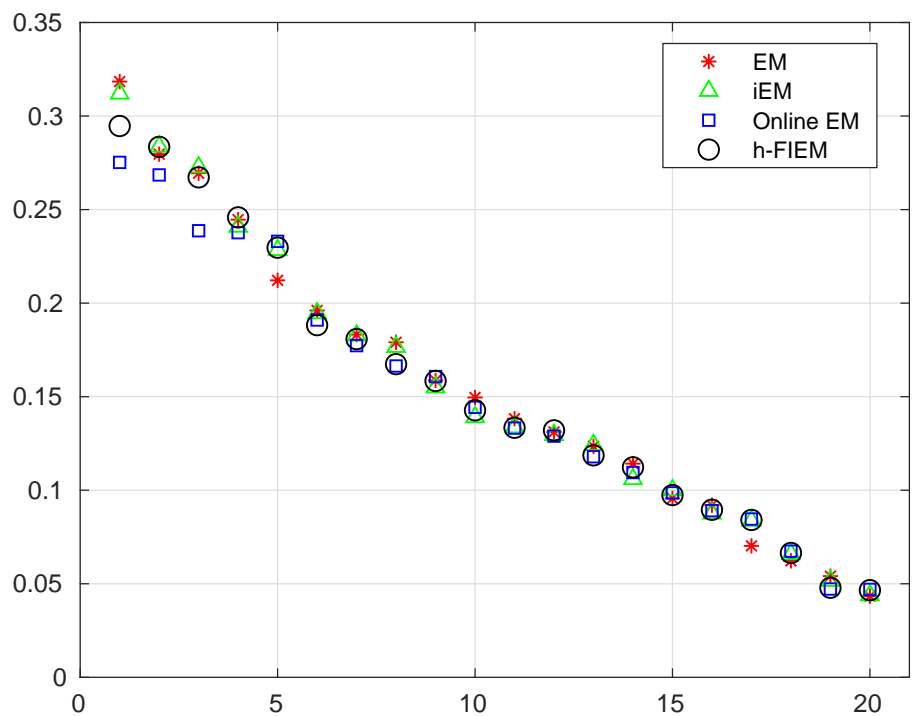}
\caption{Estimation of the $p=20$ eigenvalues of the covariance matrix
  $\Sigma$.  The estimator is the value of the parameter obtained at
  the end of a single path of length $100$ epochs.}
\end{figure*}

\end{document}